\documentclass[Afour,sageh,times]{sagej}

\usepackage{moreverb,url}
\usepackage{pifont} 

\usepackage[hidelinks,colorlinks,bookmarksopen,bookmarksnumbered,citecolor=blue,urlcolor=blue,allcolors=blue,pdfborder={0 0 0}]{hyperref}

\hypersetup{
    colorlinks=true, 
    linkcolor=blue, 
    citecolor=cyan, 
    filecolor=blue, 
    urlcolor=blue,   
    pdfborder={0 0 0} 
}
\usepackage{xcolor}
\colorlet{linkequation}{blue}
\usepackage[colorlinks]{hyperref}
\usepackage{cleveref}

\usepackage{todonotes}
\usepackage{enumerate}
\usepackage{enumitem}
\usepackage[super]{nth}
\usepackage[amssymb]{SIunits}
\usepackage{graphicx}
\usepackage{nicefrac}
\usepackage{amssymb}
\usepackage{amsthm}
\usepackage{bm} 
\usepackage{footnote}
\usepackage{threeparttable}
\usepackage{epsfig}
\usepackage{graphicx}
\usepackage{balance}
\usepackage{mathptmx}
\usepackage{booktabs}
\usepackage{subcaption}
\usepackage{wrapfig}
\usepackage{caption}
\usepackage{ulem}
\usepackage{ragged2e}
\usepackage{multirow}
\usepackage{array}
\usepackage{chemarr}

\usepackage{algorithm}
\usepackage{algorithmic}

\usepackage{todonotes}
\presetkeys{todonotes}{size=\tiny}{}

\newcommand{\fref}[1]{Figure~\ref{#1}}
\newcommand{\sref}[1]{Section~\ref{#1}}
\newcommand{\tref}[1]{Table~\ref{#1}}

\newcommand{\algref}[1]{Algorithm~\ref{#1}}
\newcommand{\lemmaref}[1]{Lemma~\ref{#1}}

\newtheorem{lemma}{Lemma}

\newcommand{\revise}[1]{{#1}}

\DeclareMathOperator*{\argminC}{\arg\min}   

\let\oldequation\equation
\let\oldendequation\endequation

\renewenvironment{equation}
  {\small\oldequation}
  {\oldendequation}

\let\oldalign\align
\let\oldendalign\endalign

\renewenvironment{align}
  {\small\oldalign}
  {\oldendalign}

\def\volumeyear{2024}

\begin{document}

\runninghead{Chen Wang et al.}

\title{Imperative Learning: A Self-supervised Neuro-Symbolic Learning Framework for Robot Autonomy}

\author{Chen Wang\affilnum{1}, Kaiyi Ji\affilnum{1}, Junyi Geng\affilnum{2}, Zhongqiang Ren\affilnum{3}, Taimeng Fu\affilnum{1}, Fan Yang\affilnum{4}, Yifan Guo\affilnum{5}, Haonan He\affilnum{3}, Xiangyu Chen\affilnum{1}, Zitong Zhan\affilnum{1}, Qiwei Du\affilnum{1}, Shaoshu Su\affilnum{1}, Bowen Li\affilnum{3}, Yuheng Qiu\affilnum{3}, Yi Du\affilnum{1}, Qihang Li\affilnum{1}, Yifan Yang\affilnum{1}, Xiao Lin\affilnum{1}, and Zhipeng Zhao\affilnum{1}}

\affiliation{%
\affilnum{1}University at Buffalo, USA\\
\affilnum{2}Pennsylvania State University, USA\\
\affilnum{3}Carnegie Mellon University, USA\\
\affilnum{4}ETH Zürich, Switzerland\\
\affilnum{5}Purdue University, USA\\
This work was partially funded by the DARPA award HR00112490426.
}

\corrauth{Chen Wang, Spatial AI \& Robotics (SAIR) Lab, Department of Computer Science and Engineering, University at Buffalo, Buffalo, NY 14260, USA.}

\email{chenw@sairlab.org}

\begin{abstract}
Data-driven methods such as reinforcement and imitation learning have achieved remarkable success in robot autonomy. However, their data-centric nature still hinders them from generalizing well to ever-changing environments. Moreover, labeling data for robotic tasks is often impractical and expensive. To overcome these challenges, we introduce a new self-supervised neuro-symbolic (NeSy) computational framework, imperative learning (IL), for robot autonomy, leveraging the generalization abilities of symbolic reasoning. The framework of IL consists of three primary components: a neural module, a reasoning engine, and a memory system. We formulate IL as a special bilevel optimization (BLO), which enables reciprocal learning over the three modules. This overcomes the label-intensive obstacles associated with data-driven approaches and takes advantage of symbolic reasoning concerning logical reasoning, physical principles, geometric analysis, etc. We discuss several optimization techniques for IL and verify their effectiveness in five distinct robot autonomy tasks including path planning, rule induction, optimal control, visual odometry, and multi-robot routing. Through various experiments, we show that IL can significantly enhance robot autonomy capabilities and we anticipate that it will catalyze further research across diverse domains.
\end{abstract}


\maketitle

\section{Introduction} \label{sec:introduction}

With the rapid development of deep learning \citep{lecun2015deep}, there has been growing interest in data-driven approaches such as reinforcement learning \citep{zhu2021deep} and imitation learning \citep{hussein2017imitation} for robot autonomy. However, despite these notable advancements, many data-driven autonomous systems are still predominantly constrained to their training environments, exhibiting limited generalization ability \citep{banino2018vector, albrecht2022avalon}. 
As a comparison, humans are capable of internalizing their experiences as abstract concepts or symbolic knowledge \citep{borghi2017challenge}. For instance, we interpret the terms ``road'' and ``path'' as symbols or concepts for navigable areas, whether it's a paved street in a city or a dirt trail through a forest \citep{hockley2011wayfaring}. Equipped with these concepts, humans can employ spatial reasoning to navigate new and complex environments \citep{strader2024indoor}.
This ability to apply spatial reasoning in unfamiliar contexts is a hallmark of human intelligence, one that remains largely missing from current data-driven autonomous robots \citep{garcez2022neural,kautz2022third}.

Though implicit reasoning abilities have garnered increased attention in the context of large language models (LLMs) \citep{lu2022neuro, shah2023lm}, robot autonomy systems still encounter significant challenges in achieving interpretable reasoning. This is particularly evident in tasks involving geometric, physical, and logical reasoning \citep{liu2023survey}. Overcoming these obstacles and integrating interpretable symbolic reasoning into data-driven models, a direction known as neuro-symbolic (NeSy) reasoning, could remarkably enhance robot autonomy \citep{garcez2022neural}.

While NeSy reasoning offers substantial potential, its specific application in the field of robotics is still in a nascent stage. A key reason for this is the emerging nature of the NeSy reasoning field itself, which has not yet reached a consensus on a rigorous definition \citep{kautz2022third}. On one side, a narrow definition views NeSy reasoning as an amalgamation of neural methods (data-driven) and symbolic methods, which utilize formal logic and symbols for knowledge representation and rule-based reasoning. Alternatively, the broader definition expands the scope of what constitutes a ``symbol.'' In this perspective, symbols are not only logical terms but also encompass any comprehensible, human-conceived concepts. This can include physical properties and semantic attributes, as seen in approaches involving concrete equations \citep{duruisseaux2023lie}, logical programming \citep{delfosse2023interpretable}, and programmable objectives \citep{yonetani2021path}. This broad definition hence supports reasoning in both continuous space (physical reasoning) and discrete space (logical reasoning).
In this context, many literatures exemplified NeSy systems, e.g., model-based reinforcement learning \citep{moerland2023model}, physics-informed networks \citep{karniadakis2021physics}, and learning-aided tasks such as control \citep{o2022neural}, task scheduling \citep{gondhi2017survey}, and geometry analytics \citep{heidari2024geometric}.

In this article, we explore the broader definition of NeSy reasoning for robot autonomy and introduce a self-supervised NeSy learning framework, which will be referred to as \textit{imperative learning} (IL).
It is designed to be a unified framework following three principles:
\textbf{(1) Data-driven and Scalable: Self-Supervised}.
Large data-driven models have demonstrated promising performance in robotics. However, labeling data for robotic tasks is often significantly more costly than in computer vision, as it typically requires specialized equipment rather than simple human annotations \citep{ebadi2022present}. For instance, generating accurate ground truth for robot planning is particularly challenging due to the complexity of robot dynamics.
As a result, to ensure the scalability of data-driven approaches in robot autonomy, the development of efficient self-supervised learning frameworks is essential.
\textbf{(2) Generalizable and Interpretable: Neuro-Symbolic Learning with Physical and Logical Structure}.
Robotic tasks often involve underlying structured knowledge governed by both physical laws and logical relationships \citep{dulac2021challenges}, such as kinematic and logical constraints, task preconditions, and safety rules. Rather than relying purely on opaque neural models, IL integrates neural learning with symbolic reasoning components that encode these domain-specific structures. This hybrid design enables the system to make decisions that are easier to interpret, while also improving generalization to novel tasks and environments by leveraging prior knowledge about how the world works.
\textbf{(3) Modular and Global Optimal: End-to-end Trainable}. 
A modular design reduces complexity by decomposing a system into independent yet interconnected components, making development, debugging, and interpretation more manageable. However, training neural and symbolic modules separately can lead to suboptimal integration, as errors may propagate and accumulate across components \citep{besold2021neural}. To overcome this limitation, IL adopts a modular yet end-to-end trainable architecture, retaining the interpretability and flexibility of modular design, while enabling joint optimization across the entire pipeline.

IL is designed to embody the three core principles within a unified framework, inspired by a key observation: both neural models and symbolic methods can be formulated as optimization problems. The key distinction lies in their supervision requirements: neural models typically rely on labeled data for training, whereas symbolic reasoning models can often operate without explicit supervision.
For instance, symbolic methods such as bundle adjustment (BA) for geometric reasoning, model predictive control (MPC) for physical reasoning, and A$^*$ search for logical reasoning can all be framed as optimization problems and solved without labeled data. IL leverages this property by enabling neural and symbolic components to mutually refine one another.
To achieve this, IL is formulated as a special bilevel optimization (BLO), where gradients are back-propagated from self-supervised symbolic modules into the neural networks. This design creates a novel self-supervised learning paradigm, in which symbolic reasoning imperatively guides the learning process. The term ``imperative'' is thus used to emphasize this passive form of self-supervised, neuro-symbolic, and end-to-end learning.

In summary, the contribution of this article include
\begin{itemize}[noitemsep,topsep=0pt]
    \item We explore a self-supervised NeSy learning framework, which is referred to as imperative learning (IL) for robot autonomy. IL is formulated as a special BLO problem to enforce the network to learn symbolic concepts and enhance symbolic reasoning via data-driven methods. This leads to a reciprocal learning paradigm, which can avoid the sub-optimal solutions caused by composed errors of decoupled systems.
    \item We discuss several optimization strategies to tackle the technical challenges of IL. We present how to incorporate different optimization techniques into our IL framework including closed-form solution, first-order optimization, second-order optimization, constrained optimization, and discrete optimization.
    \item To benefit the entire community, we demonstrated the effectiveness of IL for several tasks in robot autonomy including path planning, rule induction, optimal control, visual odometry, and multi-robot routing. We released the source code at \href{https://sairlab.org/iseries/}{https://sairlab.org/iseries/} to inspire more robotics research using IL.
\end{itemize}

This article is inspired by and built on our previous works in different fields of robot autonomy including local planning \citep{yang2023iplanner}, global planning \citep{chen2024iastar}, simultaneous localization and mapping (SLAM) \citep{fu2024islam}, feature matching \citep{zhan2024imatching}, and multi-agent routing \citep{guo2024imtsp}.
These works introduced a prototype of IL in several different domains while failing to introduce a systematic framework for NeSy learning in robot autonomy.
This article fills this gap by formally defining IL, exploring various optimization challenges of IL in different robot autonomy tasks, and introducing new applications such as rule induction and optimal control. 
Additionally, we introduce the theoretical background for solving IL based on BLO, propose several practical solutions to IL by experimenting with the five distinct robot autonomy applications, and demonstrate the superiority of IL over state-of-the-art (SOTA) methods in their respective fields.

\section{Related Works}

\subsection{Bilevel Optimization}

Bilevel optimization (BLO), first introduced by~\cite{bracken1973mathematical}, has been studied for decades. 
Classific approaches replaced the lower-level problem with its optimality conditions as constraints and reformulated the bilevel programming into a single-level constrained problem \citep{hansen1992new, gould2016differentiating, shi2005extended,sinha2017review}. 
More recently, gradient-based BLO has attracted significant attention due to its efficiency and effectiveness in modern machine learning and deep learning problems. Since this paper mainly focuses on the learning side, we will concentrate on gradient-based BLO methods, and briefly discuss their limitations in robot learning problems.

Methodologically, gradient-based BLO can be generally divided into approximate implicit differentiation (AID), iterative differentiation (ITD), and value-function-based approaches.
Based on the explicit form of the gradient (or hypergradient) of the upper-level objective function via implicit differentiation, AID-based methods adopt a generalized iterative solver for the lower-level problem as well as an efficient estimate for Hessian-inverse-vector product of the hypergradient~\citep{domke2012generic, pedregosa2016hyperparameter, liao2018reviving,arbel2022amortized}.
ITD-based methods approximate the hypergradient by directly taking backpropagation over a flexible inner-loop optimization trajectory using forward or backward mode of autograd~\citep{maclaurin2015gradient, franceschi2017forward,finn2017model, shaban2019truncated, grazzi2020iteration}.
Value-function-based approaches reformulated the lower-level problem as a value-function-based constraint and solved this constrained problem via various constrained optimization techniques such as mixed gradient aggregation, log-barrier regularization, primal-dual method, and dynamic barrier~\citep{sabach2017first,liu2020generic,li2020improved,sow2022constrained,liu2021value,ye2022bome}.
Recently, large-scale stochastic BLO has been extensively studied both in theory and in practice.
For example, \cite{chen2021single} and \cite{ji2021bilevel} proposed a Neumann series-based hypergradient estimator; \cite{yang2021provably}, \cite{huang2021biadam}, \cite{guo2021randomized}, \cite{yang2021provably}, and \cite{dagreou2022framework} incorporated the strategies of variance reduction and recursive momentum; and \cite{sow2021convergence} developed an evolutionary strategies (ES)-based method without computing Hessian or Jacobian.

Theoretically, the convergence of BLO has been analyzed extensively based on a key assumption that the lower-level problem is strongly convex \citep{franceschi2018bilevel,shaban2019truncated,liu2021value,ghadimi2018approximation,ji2021bilevel,hong2020two,arbel2022amortized,dagreou2022framework,ji2022will,huang2022efficiently}.  
Among them, \cite{ji2021lower} further provided lower complexity bounds for deterministic BLO with (strongly) convex upper-level functions. \cite{guo2021randomized}, \cite{chen2021single}, \cite{yang2021provably}, and \cite{khanduri2021near} achieved a near-optimal sample complexity with second-order derivatives. 
\cite{kwon2023fully,yang2024achieving} analyzed the convergence of first-order stochastic BLO algorithms. 
Recent works studied a more challenging setting where the lower-level problem is convex or satisfies Polyak-Lojasiewicz (PL) or Morse-Bott conditions~\citep{liu2020generic,li2020improved,sow2022constrained,liu2021value,ye2022bome,arbel2022non,chen2023bilevel,liu2021towards}.
More results on BLO and its analysis can be found in the survey \citep{liu2021investigating,chen2022gradient}.

BLO has been integrated into machine learning applications.
For example, researchers have introduced differentiable optimization layers \citep{amos2017optnet}, convex layers \citep{agrawal2019differentiable}, and declarative layers \citep{gould2021deep} into deep neural networks. They have been applied to several applications such as optical flow \citep{jiang2020joint}, pivoting manipulation \citep{shirai2022robust}, control \citep{landry2021differentiable}, and trajectory generation \citep{han2024learning}.
However, systematic approaches and methodologies to NeSy learning for robot autonomy remain under-explored.
Moreover, robotics problems are often highly non-convex, leading to many local minima and saddle points \citep{jadbabaie2019foundations}, adding optimization difficulties.
We will explore the methods with assured convergence as well as those empirically validated by various tasks in robot autonomy.

\subsection{Robot Learning Frameworks}

We summarize the major robot learning frameworks, including imitation learning, reinforcement learning, and meta-learning. The others will be briefly mentioned.

\paragraph{Imitation Learning} is a technique where robots learn tasks by observing and mimicking an expert's actions. Without explicitly modeling complex behaviors, robots can perform a variety of tasks such as dexterous manipulation~\citep{mcaleer2018solving}, navigation~\citep{triest2023learning}, and environmental interaction~\citep{chi2023diffusion}. 
Current research includes leveraging historical data, modeling multi-modal behaviors, employing privileged teachers~\citep{kaufmann2020deep, chen2020learning, lee2020learning}, and utilizing generative models to generate data like generative adversarial networks~\citep{ho2016generative}, variational autoencoders~\citep{zhao2023learning}, and diffusion models~\citep{chi2023diffusion}. These advancements highlight the vibrant and ongoing exploration of imitation learning.

Imitation learning differs from regular supervised learning as it does not assume that the collected data is independent and identically distributed (iid), and relies solely on expert data representing ``good'' behaviors. Therefore, any small mistake during testing can lead to cascading failures. While techniques such as introducing intentional errors for data augmentation~\citep{pomerleau1988alvinn, tagliabue2022efficient, codevilla2018end} and expert querying for data aggregation~\citep{ross2011reduction} exist, they still face notable challenges. These include low data efficiency, where limited or suboptimal demonstrations impair performance, and poor generalization, where robots have difficulty adapting learned behaviors to new contexts or unseen variations due to the labor-intensive nature of collecting high-quality data.

\paragraph{Reinforcement learning} (RL) is a learning paradigm where robots learn to perform tasks by interacting with their environment and receiving feedback in the form of rewards or penalties \citep{li2017deep}. 
Due to its adaptability and effectiveness, RL has been widely used in numerous fields such as navigation \citep{zhu2021deep}, manipulation \citep{gu2016deep}, locomotion \citep{margolis2024rapid}, and human-robot interaction \citep{modares2015optimized}.

However, RL also faces significant challenges, including sample inefficiency, which requires extensive interaction data~\citep{dulac2019challenges}, and the difficulty of ensuring safe exploration in physical environments~\citep{thananjeyan2021recovery}. These issues are severe in complex tasks or environments where data collection is forbidden or dangerous~\citep{pecka2014safe}. Additionally, RL often struggles with generalizing learned behaviors to new environments and tasks and faces significant sim-to-real challenges. It can also be computationally expensive and sensitive to hyperparameter choices~\citep{dulac2021challenges}. Moreover, reward shaping, while potentially accelerating learning, can inadvertently introduce biases or suboptimal policies by misguiding the learning process.

We notice that BLO has also been integrated into RL.
For instance, \cite{stadie2020learning} formulated the intrinsic
rewards as a BLO problem, leading to hyperparameter optimization;
\cite{hu2024bi} integrated reinforcement and imitation learning under BLO, addressing challenges like coupled behaviors and incomplete information in multi-robot coordination;
\cite{zhang2020bi} employed a bilevel actor-critic learning method based on BLO and achieved better convergence than Nash equilibrium in the cooperative environments.
However, they are still within the framework of RL and no systematic methods have been proposed.

\paragraph{Meta-learning} has garnered significant attention recently, particularly with its application to training deep neural networks~\citep{bengio1991learning,thrun2012learning}. Unlike conventional learning approaches, meta-learning leverages datasets and prior knowledge of a task ensemble to rapidly learn new tasks, often with minimal data, as seen in few-shot learning. Numerous meta-learning algorithms have been developed, encompassing metric-based~\citep{koch2015siamese,snell2017prototypical,chen2020variational,tang2020blockmix,gharoun2023meta}, model-based~\citep{munkhdalai2017meta,vinyals2016matching,liu2020efficient,co2021accelerating}, and optimization-based methods~\citep{finn2017model,nichol2018reptile,simon2020modulating,singh2021metamed,bohdal2021evograd,zhang2024metadiff,choe2024making}. Among them, optimization-based approaches are often simpler to implementation. They are achieving state-of-the-art results in a variety of domains. 

BLO has served as an algorithmic framework for optimization-based meta-learning.
As the most representative optimization-based approach, model-agnostic meta-learning (MAML)~\citep{finn2017model} learns an initialization such that a gradient descent procedure starting from this initial model can achieve rapid adaptation. In subsequent years, numerous works on various MAML variants have been proposed \citep{grant2018recasting,finn2019online,finn2018probabilistic,jerfel2018online,mi2019meta,liu2019taming,rothfuss2019promp,foerster2018dice,baik2020learning,raghu2019rapid,bohdal2021evograd,zhou2021task,baik2020meta,abbas2022sharp,kang2023meta,zhang2024metadiff,choe2024making}. Among them, \cite{raghu2019rapid} presents an efficient MAML variant named ANIL, which adapts only a subset of the neural network parameters. \cite{finn2019online} introduces a follow-the-meta-leader version of MAML for online learning applications. \cite{zhou2021task} improved the generalization performance of MAML by leveraging the similarity information among tasks. \cite{baik2020meta} proposed an improved version of MAML via adaptive learning rate and weight decay coefficients. \cite{kang2023meta} proposed geometry-adaptive pre-conditioned gradient descent for efficient meta-learning.  
Additionally, a group of meta-regularization approaches has been proposed to improve the bias in a regularized empirical risk minimization problem \citep{denevi2018learning,denevi2019learning,denevi2018incremental,rajeswaran2019meta,balcan2019provable,zhou2019efficient}. 
Furthermore, there is a prevalent embedding-based framework in few-shot learning \citep{bertinetto2018meta,lee2019meta,ravi2016optimization,snell2017prototypical,zhou2018deep,goldblum2020unraveling,denevi2022conditional,qin2023bi,jia2024meta}. The objective of this framework is to learn a shared embedding model applicable across all tasks, with task-specific parameters being learned for each task based on the embedded features. 

It is worth noting that IL is proposed to alleviate the drawbacks of the above learning frameworks in robotics. However, IL can also be integrated with any existing learning framework, e.g., formulating an RL method as the upper-level problem of IL, although out of the scope of this paper.

\subsection{Neuro-symbolic Learning}

As previously mentioned, the field of NeSy learning lacks a consensus on a rigorous definition. 
One consequence is that the literature on NeSy learning is scarce and lacks systematic methodologies.
Therefore, we will briefly discuss two major categories: logical reasoning and physics-infused networks.
This will encompass scenarios where symbols represent either discrete signals, such as logical constructs, or continuous signals, such as physical attributes.
We will address other related work in the context of the five robot autonomy examples within their respective sections.

\paragraph{Logical reasoning} aims to inject interpretable and deterministic logical rules into neural networks~\citep{serafini2016logic,riegel2020logical,xie2019embedding,ignatiev2018pysat}.
Some previous work directly obtained such knowledge from human expert~\citep{xu2018semantic,xie2019embedding,xie2021embedding,manhaeve2018deepproblog,riegel2020logical,ijcai2020p243} or an oracle for controllable deduction in neural networks~\citep{mao2018neuro,wang2022programmatically,hsu2023ns3d}, refered as \textit{deductive} methods.
Representative works include DeepProbLog~\citep{manhaeve2018deepproblog}, logical neural network~\citep{riegel2020logical}, and semantic Loss~\citep{xu2018semantic}.
Despite their success, deductive methods require structured formal symbolic knowledge from humans, which is not always available.
Besides, they still suffer from scaling to larger and more complex problems.
In contrast,  \textit{inductive} approaches induct the structured symbolic representation for semi-supervised efficient network learning.
One popular strategy is based on forward-searching algorithms~\citep{li2020closed,li2022softened,evans2018dILP,sen2022loa}, which is time-consuming and hard to scale up.
Others borrow gradient-based neural networks to conduct rule induction, such as SATNet~\citep{wang2019satnet}, NeuralLP~\citep{yang2017neurallp}, and neural logic machine (NLM)~\citep{dong2019nlm}.
In particular, NLM introduced a new network architecture inspired by first-order logic, which shows better compositional generalization capability than vanilla neural nets.
However, existing inductive algorithms either work on structured data like knowledge graphs~\citep{yang2017neurallp,yang2019nlil} or only experiment with toy image datasets~\citep{shindo2023aILP,wang2019satnet}.
We push the limit of this research into robotics with IL, providing real-world applications with high-dimensional image data.

\paragraph{Physical reasoning} in NeSy learning refers to approaches that integrate physical knowledge into the learning process to improve interpretability, generalization, and safety. 
Two major paradigms in this space are physics-informed and physics-infused learning.
Physics-informed learning typically incorporates physical laws into the training objective by adding physics-based loss terms or constraints, guiding neural networks to produce physically consistent outputs \citep{raissi2019physics, karniadakis2021physics}.
For example, in fluid dynamics, the Navier–Stokes equations can be enforced as soft penalties to regularize flow predictions \citep{sun2020physics}, and in structural mechanics, variational formulations are used to encode stress-strain behavior \citep{rao2021physics}.
In contrast, physics-infused learning embeds physical laws or physics-based models directly into the model architecture \citep{vargas2023physics, martin2023physics}. For instance, \cite{romero2023actor, han2024learning} incorporated differentiable trajectory optimization into learning-based planning and control pipelines, allowing end-to-end learning with explicit physical reasoning; \cite{zhao2024physord} programmed energy conservation laws into the model architecture for motion prediction in off-road driving.
These methods tightly couple data-driven components with physical structure, enabling more expressive and interpretable models that can propagate gradients through physical interactions.
We aim to incorporate physical reasoning via differentiable optimization within a unified neuro-symbolic learning framework for robot autonomy.

\section{Imperative Learning}

\subsection{Structure}

The proposed framework of imperative learning (IL) is illustrated in \fref{fig:framework}, which consists of three modules, i.e., a neural learning system, a symbolic reasoning engine, and a memory module. 
Specifically, the neural system extracts high-level semantic attributes from raw sensor data such as images, LiDAR points, IMU measurements, and their combinations.
These semantic attributes are subsequently sent to the reasoning engine, a symbolic process represented by physical principles, logical inference, analytical geometry, etc.
The memory module stores the robot's experiences and acquired knowledge, such as data, symbolic rules, and maps about the physical world for either a long-term or short-term period.
Additionally, the reasoning engine performs a consistency check with the memory, which will update the memory or make necessary self-corrections.
Intuitively, this design has the potential to combine the expressive feature extraction capabilities from the neural system, the interpretability and the generalization ability from the reasoning engine, and the memorability from the memory module into a single framework.
We next explain the mathematical rationale to achieve this.

\subsection{Formulation}\label{sec:formulation}

One of the most important properties of the framework in \fref{fig:framework} is that the neural, reasoning, and memory modules can perform reciprocal learning in a self-supervised manner.
This is achieved by formulating this framework as a BLO.
Denote the neural system as $\bm z = f({\bm{\theta}}, \bm{x})$, where $\bm{x}$ represents the sensor measurements, ${\bm{\theta}}$ represents the perception-related learnable parameters, and $\bm z$ represents the neural outputs such as semantic attributes; the reasoning engine as $g(f, M, {\bm{\mu}})$ with reasoning-related parameters ${\bm{\mu}}$ and the memory system as $M({\bm{\gamma}}, {\bm{\nu}})$, where ${\bm{\gamma}}$ is perception-related memory parameters \citep{wang2021unsupervised} and ${\bm{\nu}}$ is reasoning-related memory parameters.
In this context, our \textit{imperative learning} (IL) is formulated as a special BLO:
\begin{subequations}\label{eq:il}
\begin{align}
\min_{ \bm \psi  \doteq  [{\bm{\theta}}^\top,~{\bm{\gamma}}^\top]^\top} & U\left(f({\bm{\theta}}, \bm{x}), g({\bm{\mu}}^*), M({\bm{\gamma}}, {\bm{\nu}}^*)\right), \label{eq:high-il} \\
\textrm{s.t.} \quad & \bm \phi^* \in \argminC_{ \bm \phi \doteq  [{\bm{\mu}}^\top,~{\bm{\nu}}^\top]^\top} L(f({\bm{\theta}}, \bm{x}), g({\bm{\mu}}), M({\bm{\gamma}}, {\bm{\nu}})), \label{eq:low-il} \\
&\textrm{s.t.} \quad  \xi(M({\bm{\gamma}}, {\bm{\nu}}), {\bm{\mu}}, f({\bm{\theta}}, \bm{x})) = \text{ or }\leq 0, \label{eq:il-constraint}
\end{align}
\end{subequations}
where $\xi$ is a general constraint (either equality or inequality); $U$ and $L$ are the upper-level (UL) and lower-level (LL) cost functions; and $\bm \psi \doteq [{\bm{\theta}}^\top, {\bm{\gamma}}^\top]^\top$ are stacked UL variables and $\bm \phi \doteq [{\bm{\mu}}^\top, {\bm{\nu}}^\top]^\top$ are stacked LL variables, respectively.
Alternatively, $U$ and $L$ are also referred to as the \textit{neural cost} and \textit{symbolic cost}, respectively.
As aforementioned, the term ``\textit{imperative}'' is used to denote the passive nature of the learning process: once optimized, the neural system $f$ in the UL cost will be driven to align with the LL reasoning engine $g$ (e.g., logical, physical, or geometrical reasoning process) with constraint $\xi$, so that it can learn to generate logically, physically, or geometrically feasible semantic attributes or predicates.
In some applications, $\bm \psi$ and $\bm \phi$ are also referred to as \textit{neuron-like} and \textit{symbol-like} parameters, respectively.

\begin{figure}[t]
    \centering
    \includegraphics[width=1\linewidth]{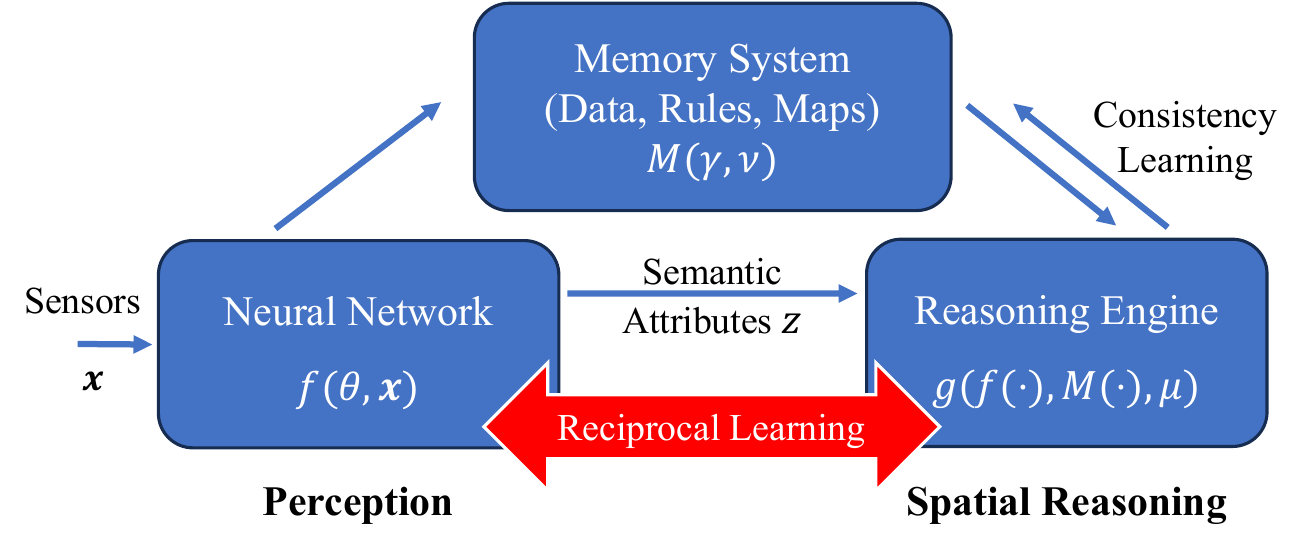}
    \caption{The framework of imperative learning (IL) consists of three primary modules including a neural perceptual network, a symbolic reasoning engine, and a general memory system. IL is formulated as a special BLO, enabling reciprocal learning and mutual correction among the three modules.}
    \label{fig:framework}
    \vspace{-10pt}
\end{figure}

\paragraph{Self-supervised Learning}

As presented in \sref{sec:introduction}, the formulation of IL is motivated by one important observation:  many symbolic reasoning engines including geometric, physical, and logical reasoning, can be optimized or solved without providing labels.
This is evident in methods like logic reasoning like equation discovery \citep{billard2002symbolic} and A$^*$ search \citep{hart1968formal}, geometrical reasoning such as bundle adjustment (BA) \citep{agarwal2010bundle}, and physical reasoning like model predictive control \citep{kouvaritakis2016model}.
The IL framework leverages this phenomenon and jointly optimizes the three modules by BLO, which enforces the three modules to mutually correct each other.
Consequently, all three modules can learn and evolve in a self-supervised manner by observing the world.
However, it is worth noting that, although IL is designed for self-supervised learning, it can easily adapt to supervised or weakly supervised learning by involving labels either in UL or LL cost functions or both.

\paragraph{Memory}
The memory system within the IL framework is a general component that can retain and retrieve information \uline{online}.
Specifically, it can be any structure that is associated with \texttt{write} and \texttt{read} operations to retain and retrieve data \citep{wang2021unsupervised}.
A memory can be a neural network, where information is ``written'' into the parameters and is ``read'' through a set of math operations or implicit mapping, e.g., a neural radiance fields (NeRF) model \citep{mildenhall2021nerf};
It can also be a structure with explicit physical meanings such as a map created online, a set of logical rules inducted online, or even a dataset collected online;
It can also be the memory system of LLMs in textual form, such as retrieval-augmented generation (RAG) \citep{lewis2020retrieval}, which writes, reads, and manages symbolic knowledge.
Hyperdimensional computing \citep{kleyko2022survey, kleyko2023survey} also presents a promising direction for memory modeling, though it falls outside the scope of this work. In this article, we will explore several customized memory structures tailored to specific applications, as detailed in \sref{sec:applications}.

\subsection{Optimization} \label{sec:solution}

BLO has been explored in frameworks such as meta-learning \citep{finn2017model}, hyperparameter optimization \citep{franceschi2018bilevel}, and reinforcement learning \citep{hong2020two}.
However, most of the theoretical analyses have primarily focused on their applicability to data-driven models, where \nth{1}-order gradient descent (GD) is frequently employed \citep{ji2021bilevel,gould2021deep}.
Nevertheless, many reasoning tasks present unique challenges that make GD less effective. For example, geometrical reasoning like BA requires \nth{2}-order optimizers \citep{fu2024islam} such as Levenberg-Marquardt (LM) \citep{marquardt1963algorithm}; multi-robot routing needs combinatorial optimization over discrete variables \citep{ren2023cbss}. Employing such LL optimizations within the BLO framework introduces extreme complexities and challenges, which are still underexplored \citep{ji2021bilevel}.
Therefore, we will first delve into general BLO and then provide practical examples covering distinct challenges of LL optimizations in our IL framework.

The solution to IL \eqref{eq:il} mainly involves solving the UL parameters ${\bm{\theta}}$ and ${\bm{\gamma}}$ and the LL parameters ${\bm{\mu}}$ and ${\bm{\nu}}$.
Intuitively, the UL parameters which are often neuron-like weights can be updated with the gradients of the UL cost $U$:
\begin{equation}
\begin{aligned}\label{eq:solution}
\nabla_{\bm{\theta}} U &= \frac{\partial U}{\partial f} \frac{\partial f}{\partial {\bm{\theta}}} +  \frac{\partial U}{\partial g} \frac{\partial g}{\partial {\bm{\mu}}^*}{\color{blue}\frac{\partial {\bm{\mu}}^*}{\partial {\bm{\theta}}}} + \frac{\partial U}{\partial M}\frac{\partial M}{\partial {\bm{\nu}}^*}{\color{blue}\frac{\partial {\bm{\nu}}^*}{\partial {\bm{\theta}}}},
\\\nabla_{\bm{\gamma}} U& = \frac{\partial U}{\partial M} \frac{\partial M}{\partial {\bm{\gamma}}} + \frac{\partial U}{\partial g} \frac{\partial g}{\partial {\bm{\mu}}^*} {\color{blue}\frac{\partial {\bm{\mu}}^*}{\partial {\bm{\gamma}}}} +\frac{\partial U}{\partial M} \frac{\partial M}{\partial {\bm{\nu}}^*} {\color{blue}\frac{\partial {\bm{\nu}}^*}{\partial {\bm{\gamma}}}}.
\end{aligned}
\end{equation}
The key challenge of computing \eqref{eq:solution} is the implicit differentiation parts in blue fonts, which take the forms of
\begin{equation} \label{implicit_part}
{\color{blue} \frac{\partial \bm \phi^*}{\partial \bm \psi}} = \left[
\begin{aligned}
\frac{\partial \bm \mu^*}{\partial \bm \theta} &\quad \frac{\partial \bm \mu^*}{\partial \bm \gamma} \\
\frac{\partial \bm \nu^*}{\partial \bm \theta} &\quad \frac{\partial \bm \nu^*}{\partial \bm \gamma} \\
\end{aligned}\right], \\
\end{equation}
where ${\bm{\mu}}^*$ and ${\bm{\nu}}^*$ are the solutions to the LL problem.
For simplicity, we can write \eqref{eq:solution} into the matrix form.
\begin{equation}\label{eq:solution-simplified}
    \nabla_{\bm{\psi}} U = \frac{\partial U}{\partial \bm \psi} + \frac{\partial U}{\partial \bm \phi^*}{\color{blue}\frac{\partial \bm\phi^*}{\partial {\bm{\psi}}}},
\end{equation}
where $\frac{\partial U}{\partial \bm \psi} = \left[(\frac{\partial U}{\partial f} \frac{\partial f}{\partial {\bm{\theta}}})^\top,(\frac{\partial U}{\partial M} \frac{\partial M}{\partial {\bm{\gamma}}})^\top\right]^\top$ and $\frac{\partial U}{\partial \bm \phi^*} = \left[(\frac{\partial U}{\partial g} \frac{\partial g}{\partial {\bm{\mu}}^*})^\top, (\frac{\partial U}{\partial M}\frac{\partial M}{\partial {\bm{\nu}}^*})^\top \right]^\top$.
Computing these gradients is challenging. We therefore first introduce a simple yet effective method, \textit{alternating optimization}, which avoids explicit gradient computation, and then present \textit{unrolled differentiation} and \textit{implicit differentiation} to compute them.

\subsubsection{Alternating Optimization} \label{sec:alternating}

One simple strategy for BLO is \textit{alternating optimization}, which avoids explicit computation of the implicit gradients $\frac{\partial \bm \phi^*}{\partial \bm \psi}$.
Specifically, the UL and LL problems are solved in an alternating and iterative manner.
Given UL variables $\bm \psi^{(k)}$ at the $k$-th iteration, the LL problem is first (approximately) solved to obtain $\bm \phi^{(k)}$, possibly under constraints~\eqref{eq:il-constraint}.
Then, fixing $\bm \phi^{(k)}$, the UL variables are updated by minimizing the UL cost $U$ using any standard gradient-based methods.
This process is repeated iteratively until convergence, as illustrated in \eqref{eq:alternating}.
\begin{equation}\footnotesize\label{eq:alternating}
    \underbrace{\bm \psi^{(k+1)} \leftarrow \bm \psi^{(k)} - \eta \cdot \frac{\partial {U} (\bm \psi; \bm {\bm{\phi}}^{(k)})}{\partial \bm \psi}}_{\text{Upper-level Optimization}}
    \xrightleftharpoons[\bm \phi^{(k)}]{\bm \psi^{(k+1)}}
    \underbrace{\bm \phi^{(k)} \leftarrow \arg \min_{\bm \phi} L(\bm \phi; \bm \psi^{(k)})}_{\text{Lower-level Optimization}}
\end{equation}
Alternating optimization can be viewed as a block coordinate descent on the bilevel structure.
The LL reasoning engine acts as a constraint-enforcing module that corrects neural and memory outputs, while the UL learning process gradually aligns these modules with the LL solutions.
A key advantage of alternating optimization is that it places minimal requirements on the LL solver.
The LL problem need not be differentiable and may involve discrete variables or non-smooth constraints.
This makes alternating optimization particularly suitable for symbolic reasoning tasks that rely on combinatorial search, constraint solver, or other settings where implicit differentiation is difficult to compute.
Alternating optimization has been verified to be effective in practice, such as in geometric perception \citep{yang2021self} and rule induction in \sref{sec:exp-iLogic}.

Alternating optimization ignores the dependency of $\bm{\phi}^*$ on $\bm{\psi}$ and thus yields biased gradients, it can sometimes be unstable and sensitive to initialization. Nevertheless, such an approximation is often acceptable in training deep neural networks~\citep{finn2017model,nichol2018reptile}, where the Jacobian and Hessian terms induced by the implicit dependence of $\bm{\phi}^*$ on $\bm{\psi}$ are typically orders of magnitude smaller than the direct upper-level gradient $\frac{\partial {L} (\bm \psi; \bm {\bm{\phi}}^{(k)})}{\partial \bm \psi}$. However, in certain specialized tasks where the upper-level objective $U$ depends solely on $\bm{\phi}$, so that only the implicit dependence of $\bm{\phi}^*$ on $\bm{\psi}$ remains ($\bm{\psi}$ can only affect $U$ via affecting $\bm{\phi}^*$, i.e., the UL cost becomes $U(\bm{\phi}^*(\bm{\psi}))$), this alternative optimization strategy is no longer applicable.

\subsubsection{Unrolled Differentiation} \label{sec:unrolled}
The method of unrolled differentiation is an easy-to-implement solution for BLO problems. It needs automatic differentiation (AutoDiff) through LL optimization. Specifically, given an initialization for LL variable $\bm \phi_0 = \bm \Phi_0(\bm {\bm{\psi}})$ at step $t=0$, the iterative process of unrolled optimization is
\begin{equation}\label{eq:unroll-system}
    \bm \phi_t = \bm \Phi_t( \bm \phi_{t-1}; \bm \psi), \quad t = 1, \cdots, T,
\end{equation}
where $\bm \Phi_t$ denotes an updating scheme based on a specific LL problem and $T$ is the number of iterations. One popular updating scheme is based on the gradient descent:
\begin{equation}\label{eq:unroll}
    \bm \Phi_t (\bm \phi_{t-1}; \bm \psi) = \bm \phi_{t-1} - \eta \cdot \frac{\partial {L} (\bm \phi_{t-1}; \bm {\bm{\psi}})}{\partial \bm \phi_{t-1}},
\end{equation}
where $\eta$ is a learning rate and the term $\frac{\partial {L} (\bm \phi_{t-1}; \bm \psi)}{\partial \bm \phi_{t-1}}$ can be computed from AutoDiff.
Therefore, we can obtain $\nabla_{\bm{{\bm{\theta}}}} {U} (\bm{\bm{\theta}})$, $\nabla_{\bm{{\bm{\gamma}}}} {U} (\bm{\bm{\gamma}})$ by substituting $\bm \phi_T = [\bm {\bm{\mu}}_T, \bm {\bm{\nu}}_T]$ approximately for $\bm \phi^* = [\bm {\bm{\mu}}^*, \bm {\bm{\nu}}^*]$ and the full unrolled system is defined as 
\begin{equation}\label{eq:unroll-composition}
    \bm \Phi(\bm\psi) = (\bm \Phi_T \circ \cdots \circ \bm \Phi_1 \circ \bm \Phi_0)(\bm \psi),
\end{equation}
where the symbol $\circ$ denotes the function composition.
Therefore, we instead only need to consider an alternative:
\begin{equation}\label{eq:unroll-opt}
    \min_{\bm{\psi} } \;\; {U}(\bm\psi, \bm \Phi(\bm\psi)),
\end{equation}
where $\frac{\partial \bm \Phi(\bm \psi)}{\partial \bm \psi}$ can be computed via AutoDiff instead of directly calculating the four terms in $\frac{\partial \bm \phi^*}{\partial \bm \psi}$ of \eqref{implicit_part}.
A generic framework incorporating both unrolled differentiation and implicit differentiation is listed in \algref{alg:main} for clarity.

\begin{algorithm}[t]
    \small 
    \caption{The Algorithm of Solving Imperative Learning by \uline{Unrolled Differentiation} or \uline{Implicit Differentiation}.}
    \label{alg:main}
    \begin{algorithmic}[1]
    \WHILE{Not Convergent}
	\STATE{Obtain ${\bm{\mu}}_T, {\bm{\nu}}_T$ by solving LL problem~\eqref{eq:low-il} (possibly with constraint~\eqref{eq:il-constraint}) by a generic optimizer $\mathcal{O}$ with $T$ steps.}
        \STATE{Efficient estimation of UL gradients in~\eqref{eq:solution} via \\
        \textbf{Unrolled Differentiation}: Compute $\hat \nabla_{\bm{\theta}} U$ and $\hat \nabla_{\bm{\gamma}} U$ by having $\bm\Phi$ in \eqref{eq:unroll-composition} and ${U}(\bm\psi, \bm \Phi(\bm\psi))$ in \eqref{eq:unroll-opt} via AutoDiff.}\\
        \textbf{Implicit Differentiation (\algref{alg:implicit-diff})}: Compute
        {
        \begin{equation*} 
          \hat \nabla_{\bm{\psi}} U = \frac{\partial U}{\partial \bm \psi} \Big|_{\bm\phi_T} +  \frac{\partial U}{\partial \bm \phi^*}{\color{blue}\frac{\partial \bm\phi^*}{\partial {\bm{\psi}}}}\Big|_{\bm \phi_T},
        \end{equation*}}where the implicit derivatives $\color{blue}\frac{\partial \bm\phi^*}{\partial {\bm{\psi}}}$ can be obtained by solving a linear system via the LL optimality conditions.
    \ENDWHILE
    \end{algorithmic}
\end{algorithm}

\subsubsection{Implicit Differentiation} \label{sec:implicit}

The method of implicit differentiation directly computes the derivatives $\color{blue} \frac{\partial \bm \phi^*}{\partial \bm \psi}$.
We next introduce a generic framework for the implicit differentiation algorithm by solving a linear system from the first-order optimality condition of the LL problem, while the exact solutions depend on specific tasks and will be illustrated using several independent examples in \sref{sec:applications}.

Assume $\bm \phi_T = [\bm \mu_T, \bm \nu_T]$ is the outputs of $T$ steps of a generic optimizer of the LL problem~\eqref{eq:low-il} possibly under constraints~\eqref{eq:il-constraint}, then the approximated UL gradient \eqref{eq:solution} is
\begin{equation}
    \hat \nabla_{\bm{\psi}} U \approx \frac{\partial U}{\partial \bm \psi} \Big|_{\bm\phi_T} +  \frac{\partial U}{\partial \bm \phi^*}{\color{blue}\frac{\partial \bm\phi^*}{\partial {\bm{\psi}}}}\Big|_{\bm \phi_T}.
\end{equation}
Then the derivatives $\frac{\partial \bm\phi^*}{\partial {\bm{\psi}}}$ are obtained via solving implicit equations from optimality conditions of the LL problem, i.e., $\frac{\partial \tilde L}{\partial \bm \phi^*} = \bm 0$, where $\tilde L$ is a generic LL optimality condition. Specifically, taking the derivative of equation $\frac{\partial \tilde L}{\partial \bm \phi^*} = \bm 0$ with respect to $\bm \psi$ on both sides leads to
\begin{equation}
    \frac{\partial^2 \tilde{L} (\bm \phi^*(\bm\psi), \bm \psi)}{\partial \bm \phi^*(\bm\psi)\partial \bm \psi} + \frac{\partial^2 \tilde{L} (\bm \phi^*(\bm\psi), \bm \psi)}{\partial \bm \phi^*(\bm\psi)\partial \bm \phi^*(\bm \psi)} \cdot {\color{blue} \frac{\partial \bm \phi^* (\bm \psi)}{\partial \bm \psi}} = \bm 0.
\end{equation}
Solving the equation gives us the implicit gradients as
\begin{equation}\label{eq:implicit-gradiant}
    {\color{blue} \frac{\partial \bm \phi^* (\bm \psi)}{\partial \bm \psi}} = - \underbrace{ \left( \frac{\partial^2 \tilde{L} (\bm \phi^*(\bm\psi), \bm \psi)}{\partial \bm \phi^*(\bm\psi)\partial \bm \phi^*(\bm \psi)}  \right)^{-1} }_{\bm H_{\bm \phi^* \bm \phi^*}^{-1}} \underbrace{ \frac{\partial^2 \tilde{L} (\bm \phi^*(\bm\psi), \bm \psi)}{\partial \bm \phi^*(\bm\psi)\partial \bm \psi}}_{\bm H_{\bm \phi^* \bm \psi}}.
\end{equation}
This means we obtain the implicit gradients at the cost of an inversion of the Hessian matrix $\bm H_{\bm \phi^* \bm \phi^*}^{-1}$.

In practice, the Hessian matrix can be too big to calculate and store\footnote[2]{For instance, assume both UL and LU costs have a network with merely 1 million ($10^6$) parameters ($32$-bit float numbers), thus each network only needs a space of $10^6\times 4 \text{Byte} = 4 \mega B$ to store, while their Hessian matrix needs a space of $(10^6)^2 \times 4 \text{Byte} = 4 \tera B$ to store. This indicates that a Hessian matrix cannot even be explicitly stored in the memory of a low-power computer, thus directly calculating its inversion is more impractical.}, but we could bypass it by solving a linear system. Substitute \eqref{eq:implicit-gradiant} into \eqref{eq:solution-simplified}, we have the UL gradient
\begin{subequations}
    \begin{align}
    \nabla_{\bm\psi} {U} & = \frac{\partial {U}}{\partial \bm\psi} - \underbrace{\frac{\partial {U}(\bm{\psi}, \bm{\phi}^*)}{\partial \bm \phi^*} }_{\bm v^\top} \bm H_{\bm \phi^* \bm \phi^*}^{-1} \cdot \bm H_{\bm \phi^* \bm \psi} \\
    &= \frac{\partial {U}}{\partial \bm\psi} - \bm q^\top \cdot \bm H_{\bm \phi^* \bm \psi}.
    \end{align}
\end{subequations}
Therefore, instead of calculating the Hessian inversion, we can solve a linear system $\bm H \bm q = \bm v$ for $\bm q^\top$ by optimizing 
\begin{equation}\label{eq:linear-system}
    \bm q^* = \argminC_{\bm q} Q (\bm q) = \argminC_{\bm q} \left(\frac{1}{2} \bm q^\top \bm H \bm q - \bm q^\top \bm v \right),
\end{equation}
where $\bm H_{\bm \phi^* \bm \phi^*}$ is denoted as $\bm H$ for simplicity.
The linear system \eqref{eq:linear-system} can be solved without explicitly calculating and storing the Hessian matrix by solvers such as conjugate gradient \citep{hestenes1952methods} or gradient descent.
For example, due to the gradient $\frac{\partial Q(\bm q)}{\partial \bm q} = \bm H \bm q - \bm v$, the updating scheme based on the gradient descent algorithm is:
\begin{equation}
    \bm q_{k} = \bm q_{k-1} - \eta \left(\bm H \bm q_{k-1} - \bm v\right).
\end{equation}
This updating scheme is efficient since $\bm H \bm q$ can be computed using the fast Hessian-vector product, i.e., a Hessian-vector product is the gradient of a gradient-vector product:
\begin{equation}
    \bm H \bm q = \frac{\partial^2 \tilde{L}}{\partial \bm \phi \partial \bm \phi } \cdot \bm q =  \frac{\partial \left(\frac{\partial \tilde{L}}{\partial \bm \phi} \cdot \bm q \right)}{\partial \bm \phi},
\end{equation}
where $\frac{\partial \tilde{L}}{\partial \bm \phi}\cdot \bm q$ is a scalar. This means the Hessian $\bm H$ is not explicitly computed or stored.
We summarize this implicit differentiation in \algref{alg:implicit-diff}.
Note that the optimality condition depends on the LL problems.
In \sref{sec:applications}, we will show that it can either be the derivative of the LL cost function for an unconstrained problem such as \eqref{eq:stationary_condition} or the Lagrangian function for a constrained problem such as \eqref{eq:stationary}.

\begin{algorithm}[t]
    \small
    \caption{\textit{Implicit Differentiation} via Linear System.}
    \label{alg:implicit-diff}
    \begin{algorithmic}[1]
        \STATE {\bf Input:} The current variable $\bm {\bm{\psi}}$ and the optimal variable $\bm {\bm{\phi}}^*$.
        \STATE {\bf Initialization:} $k=1$, learning rate $\eta$.
        \WHILE{Not Convergent}
            \STATE{Perform gradient descent (or use conjugate gradient):}
            \begin{equation}
                \bm q_{k} = \bm q_{k-1} - \eta \left(\bm H \bm q  - \bm v\right),
            \end{equation}
            where $\bm H \bm q$ is computed from fast Hessian-vector product.
        \ENDWHILE
        \STATE{Assign $\bm q = \bm q_k $}
        \STATE{Compute $\nabla_{\bm{\bm{\theta}}} {U}$ and $\nabla_{\bm{\bm{\gamma}}} {U}$ in \eqref{eq:solution} as:}
            \begin{equation}
                \nabla_{\bm{\bm{\psi}}} {U} = \frac{\partial {U}}{\partial \bm\psi} - \bm q^\top \cdot \bm H_{\bm \phi^* \bm \psi},
            \end{equation}
            where $\bm q^\top \cdot \bm H_{\bm \phi^* \bm \psi}$ is also from the fast Hessian-vector product.
    \end{algorithmic}
\end{algorithm}

\paragraph{Approximation}
Implicit differentiation is complicated to implement but there is one approximation, which is to ignore the implicit components and only use the direct part $\hat \nabla_{\bm{\psi}} U \approx \frac{\partial U}{\partial \bm \psi}\Big|_{\bm{\phi}_T}$. This is equivalent to taking the solution ${\bm{\phi}}_T$ from the LL optimization as \uline{constants} in the UL problem.
Such an approximation is more efficient but introduces an error term 
\begin{equation}\label{eq:error-term}
    \epsilon \sim \left| \frac{\partial U}{\partial \bm \phi^*}\frac{\partial \bm\phi^*}{\partial {\bm{\psi}}} \right|.
\end{equation}
Nevertheless, it is useful when the implicit derivatives contain products of \uline{small} second-order derivatives, which again depends on the specific LL problems. 

It is worth noting that in the framework of IL \eqref{eq:il}, we assign perception-related parameters $\bm \psi$ to the UL neural cost \eqref{eq:high-il}, while reasoning-related parameters $\bm \phi$ to the LL symbolic cost \eqref{eq:low-il}. 
This design stems from two key considerations:
First, it can avoid involving large Jacobian and Hessian matrices for neuron-like variables such as $\bm H_{\bm \phi^* \bm \phi^*}$.
Given that real-world robot applications often involve an immense number of neuron-like parameters (e.g., a simple neural network might possess millions), placing them in the UL cost reduces the complexity involved in computing implicit gradients necessitated by the LL cost \eqref{eq:low-il}.

Second, perception-related (neuron-like) parameters are usually updated using gradient descent algorithms such as SGD \citep{sutskever2013importance}. However, such simple first-order optimization methods are often inadequate for LL symbolic reasoning, e.g., geometric problems \citep{wang2023pypose} usually need second-order optimizers.
Therefore, separating neuron-like parameters from reasoning-related parameters makes the BLO in IL easier to solve and analyze. However, this again depends on the LL tasks.

\paragraph{Discussion on theoretical optimality and sufficient conditions.} 
The guarantee of feasible solutions depends on the assumptions and conditions of the objective functions. For example, in unrolled differentiation (\sref{sec:unrolled}) and implicit differentiation (\sref{sec:implicit}), if the LL function $L$ in \eqref{eq:low-il} is twice differentiable and strongly convex with respect to $\phi$ with a unique solution, we expect to achieve first-order stationarity and global optimality for non-convex and convex UL objective functions, respectively.

The optimality for generic non-convex LL problems remain an open challenge and are sometimes NP-hard~\citep{dempe2015bilevel} in BLO theory. This difficulty arises because non-uniqueness in the LL problem prevents the implicit differentiation theorem from guaranteeing the existence of implicit differentiability of $\phi^*$ with respect to the UL variables $\psi$. A weaker condition for establishing optimality is the Polyak-Lojasiewicz (PL) condition~\citep{shen2023penalty}, which allows for an overall non-convex landscape while maintaining certain convexity-like structures. Such structures have been observed in overparameterized neural networks and can guarantee an $\epsilon$-approximate solution to the BLO problem~\citep{liu2020toward}. 
We will show more detailed theoretical guarantees and their assumptions in the next section with respect to the types of LL optimization. To ensure the language remaining accessible and not overly theoretical, we will only provide summarized results and provide references to their sources. In this way, we hope it will be helpful for readers to easily find more relevant formal results when needed.

\begin{table}[t]
    \caption{Five examples in logic, planning, control, SLAM, and MTSP are selected to cover distinct scenarios of the LL problems for solving the BLO in imperative learning.}
    \label{tab:applications}
    \centering
    \resizebox{\linewidth}{!}{
    \begin{tabular}{ccccl}
    \toprule
    Section & Application & LL problem & LL solution & Type\\
    \midrule
    Sec. \ref{sec:close-form} & Planning & A$^*$ & Closed-form & (C) \\
    Sec. \ref{sec:1st-order} & Logic & NLM & \nth{1}-order & (B)$^*$\\
    Sec. \ref{sec:constrained} & Control & MPC & Constrained &  (A) \\
    Sec. \ref{sec:2nd-order} & SLAM & PGO & \nth{2}-order & (C) \\
    Sec. \ref{sec:discrete} & MTSP & TSP & Discrete & (A) \\
    \bottomrule
    \end{tabular}
    }
    \footnotesize{$^*$All modules are trained, although a pre-trained feature extractor is given.}\\
\end{table}

\section{Applications and Examples} \label{sec:applications}

To showcase the effectiveness of IL, we will introduce five distinct examples in different fields of robot autonomy.
These examples, along with their respective LL problem and optimization methods, are outlined in \tref{tab:applications}.
Specifically, they are selected to cover distinct tasks, including path planning, rule induction, optimal control, visual odometry, and multi-agent routing, to showcase different optimization techniques required by the LL problems, including closed-form solution, first-order optimization, constrained optimization, second-order optimization, and discrete optimization, respectively.
We will explore several memory structures mentioned in  \sref{sec:formulation}.

Additionally, since IL is a self-supervised learning framework consisting of three primary components, we have three kinds of learning types.
This includes (\textbf{A}) given known (pre-trained or human-defined) reasoning engines such as logical reasoning, physical principles, and geometrical analytics, robots can learn a logic-, physics-, or geometry-aware neural perception system, respectively, in a self-supervised manner; (\textbf{B}) given neural perception systems such as a vision foundation model \citep{kirillov2023segment}, robots can discover the world rules, e.g., traffic rules, and then apply the rules to future events; and (\textbf{C}) given a memory system, (e.g., experience, world rules, or maps), robots can simultaneously update the neural system and reasoning engine so that they can adapt to novel environments with a new set of rules.
Note that if one of the modules has a good initialization, all the three modules can be further finetuned simultaneously.
The five examples will cover all three learning types.

\subsection{Closed-form Solution} \label{sec:close-form}

We first illustrate the scenarios in which the LL cost $L$ in \eqref{eq:low-il} has closed-form solutions.
In this case, one can directly optimize the UL cost by solving
\begin{equation}\label{eq:ggs}
\min_{{\bm{\theta}},{\bm{\gamma}}} \quad U\left(f({\bm{\theta}}, \bm{x}), g({\bm{\mu}}^*({\bm{\theta}},{\bm{\gamma}})), M({\bm{\gamma}}, {\bm{\nu}}^*({\bm{\theta}},{\bm{\gamma}}))\right),
\end{equation}
where the LL solutions ${\bm{\mu}}^*({\bm{\theta}},{\bm{\gamma}})$ and ${\bm{\nu}}^*({\bm{\theta}},{\bm{\gamma}})$, that contain the implicit components $\frac{\partial {\bm{\mu}}^*}{\partial {\bm{\theta}}}$, $\frac{\partial {\bm{\mu}}^*}{\partial {\bm{\gamma}}}$, $\frac{\partial {\bm{\nu}}^*}{\partial {\bm{\theta}}}$, and $\frac{\partial {\bm{\nu}}^*}{\partial {\bm{\gamma}}}$, can be calculated directly due to the closed-form of ${\bm{\mu}}^*$ and ${\bm{\nu}}^*$. As a result, standard gradient-based algorithms can be applied in updating ${\bm{\theta}}$ and ${\bm{\gamma}}$ with gradients given by  \eqref{eq:solution}.
In this case, there is {\bf no} approximation error induced by the LL minimization, in contrast to existing widely-used implicit and unrolled differentiation methods that require ensuring a sufficiently small or decreasing LL optimization error for guaranteeing the convergence~\citep{ji2021bilevel}. 

One possible problem in computing \eqref{implicit_part} is that the implicit components can contain expensive matrix inversions. Let us consider a simplified quadratic case in \eqref{eq:low-il} with ${\bm{\mu}}^*=\arg\min_{{\bm{\mu}}}\frac{1}{2}\|f({\bm{\theta}},x){\bm{\mu}}-y\|^2$. The closed-form solution takes the form of ${\bm{\mu}}^*=[f({\bm{\theta}},x)^\top f({\bm{\theta}},x)]^{-1}y$, which can induce a computationally expensive inversion of a possibly large matrix $f({\bm{\theta}},x)^\top f({\bm{\theta}},x)$. To address this problem, one can again formulate the matrix-inverse-vector computation $H^{-1}y$ as solving a linear system $\min_v \frac{1}{2}v^\top Hv-v^\top y$ using any optimization methods with efficient matrix-vector products. 

Many symbolic costs can be effectively addressed through closed-form solutions. 
For example, both linear quadratic regulator (LQR) \citep{shaiju2008formulas} and Dijkstra's algorithm \citep{dijkstra1959note} can be solved with a determined optimal solution.
To demonstrate the effectiveness of IL for closed-form solutions, we next present two examples in path planning to utilize the neural model for reducing the search and sampling space of symbolic optimization.

\subsection*{Example 1: Path Planning}

Path planning is a computational process to determine a path from a starting point to a destination within an environment. It typically involves navigating around obstacles and may also optimize certain objectives such as the shortest distance, minimal energy use, or maximum safety. 
Path planning algorithms are generally categorized into global planning, which utilizes global maps of the environment, and local planning, which relies on real-time sensory data. 
We will enhance two widely-used algorithms through IL: A$^*$ search for global planning and cubic spline for local planning, both of which offer closed-form solutions.

\subsection*{Example 1.A: Global Path Planning}

\paragraph{Background}

The A$^*$ algorithm is a graph search technique that aims to find the global shortest feasible path between two nodes \citep{hart1968formal}.
Specifically, A$^*$ selects the path passing through the next node $n\in G$ that minimizes
\begin{equation}
    C(n) = s(n) + h(n),
\end{equation}
where $s(n)$ is the cost from the start node to $n$ and $h(n)$ is a heuristic function that predicts the cheapest path cost from $n$ to the goal.
The heuristic cost can take the form of various metrics such as the Euclidean and Chebyshev distances.
It can be proved that if the heuristic function $h(n)$ is admissible and monotone, i.e., $\forall m\in G, h(n) \le h(m) + d(m,n)$, where $d(m,n)$ is the cost from $m$ to $n$, A$^*$ is guaranteed to find the optimal path without searching any node more than once \citep{dechter1985generalized}.
Due to this optimality, A$^*$ became one of the most widely used methods in path planning \citep{paden2016surveymotionplanning,smith2012dual,algfoor2015comprehensive}.

However, A$^*$ encounters significant limitations, particularly in its requirement to explore a large number of potential nodes. 
This exhaustive search process can be excessively time-consuming, especially for low-power robot systems.
To address this, recent advancements showed significant potential for enhancing efficiency by predicting a more accurate heuristic cost map using data-driven methods \citep{choudhury2018data, yonetani2021path, kirilenko2023transpath}.
Nevertheless, these algorithms utilize optimal paths as training labels, which face challenges in generalization, leading to a bias towards the patterns observed in the training datasets.

\begin{figure}
    \centering
    \includegraphics[width=\linewidth]{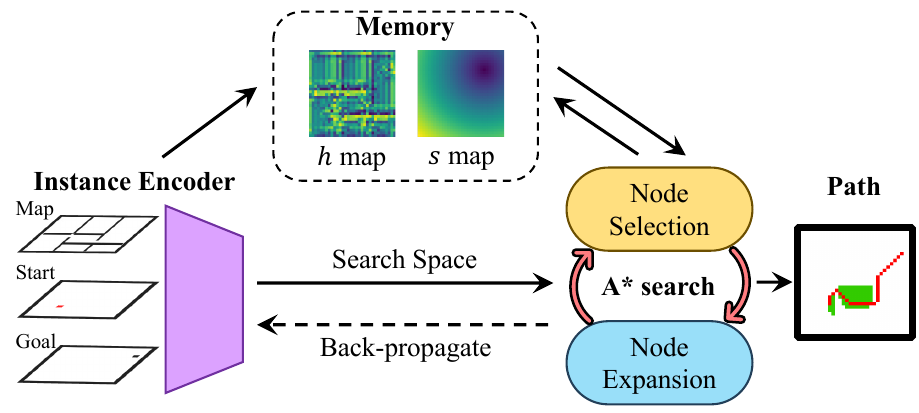}
    \caption{The framework of iA* search. The network predicts a confined search space, leading to overall improved efficiency. The A$^*$ search algorithm eliminates the label dependence, resulting in a self-supervised path planning framework.}
    \label{fig:iastar}
\end{figure}

\paragraph{Approach}

To address this limitation, we leverage the IL framework to remove dependence on path labeling.
Specifically, we utilize a neural network $f$ to estimate the $h$ value of each node, which can be used to reduce the search space.
Subsequently, we integrate a differentiable A$^*$ module, serving as the symbolic reasoning engine $g$ in \eqref{eq:low-il}, to determine the most efficient path.
This results in an effective framework depicted in \fref{fig:iastar}, which we refer to as imperative A$^*$ (iA*) algorithm.
Notably, the iA* framework can operate on a self-supervised basis inherent in the IL framework, eliminating the need for annotated labels.

Specifically, the iA* algorithm is to minimize
\begin{subequations}\label{eq:iastar}
\begin{align}
\min_{{\bm{\theta}}} \quad & U\left(f(\bm{x}, {\bm{\theta}}), g({\bm{\mu}}^*), M({\bm{\nu}}^*)\right), \label{eq:high-iastar} \\
\textrm{s.t.} \quad & {\bm{\mu}}^*,{\bm{\nu}}^* = \arg\min_{{\bm{\mu}},{\bm{\nu}}} L(f, g({\bm{\mu}}), M({\bm{\nu}})), \label{eq:low-iastar}
\end{align}
\end{subequations}
where $\bm{x}$ denotes the inputs including the map, start node, and goal node, ${\bm{\mu}}$ is the set of paths in the solution space, ${\bm{\nu}}$ is the accumulated $h$ values associated with a path in ${\bm{\mu}}$, ${\bm{\mu}}^*$ is the optimal path, ${\bm{\nu}}^*$ is the accumulated $h$ values associated with the optimal path ${\bm{\mu}}^*$, and $M$ is the intermediate $s$ and $h$ maps.

The lower-level cost $L$ is defined as the path length (cost)
\begin{equation}
    L({\bm{\mu}}, M, g) \doteq  C_l({\bm{\mu}}^*, M),
\end{equation}
where the optimal path ${\bm{\mu}}^*$ is derived from the A$^*$ reasoning $g$ given the node cost maps $M$.
Thanks to its closed-form solution, this LL optimization is directly solvable. The UL optimization focuses on updating the network parameter ${\bm{\theta}}$ to generate the $h$ map. Given the impact of the $h$ map on the search area of A$^*$, the UL cost $U$ is formulated as a combination of the search area $C_a({\bm{\mu}}^*, {\bm{\nu}}^*)$ and the path length $C_l({\bm{\mu}}^*)$. This is mathematically represented as:
\begin{equation}
    U({\bm{\mu}}^*, M)  \doteq w_a C_a({\bm{\mu}}^*, {\bm{\nu}}^*) + w_l C_l({\bm{\mu}}^*, M),
\end{equation}
where $w_a$ and $w_l$ are the weights to adjust the two terms, and $C_a$ is the search area computed from the accumulated $h$ map with the optimal path ${\bm{\mu}}^*$.
In the experiments, we define the $s$ cost as the Euclidean distance. The input $\bm{x}$ is represented as a three-channel 2-D tensor, with each channel dedicated to a specific component: the map, the start node, and the goal node. Specifically, the start and goal node channels are represented as a one-hot 2-D tensor, where the indices of the ``ones'' indicate the locations of the start and goal node, respectively.
This facilitates a more nuanced and effective representation of path planning problems.

\paragraph{Optimization}

As shown in \eqref{eq:solution}, once the gradient calculation is completed, we backpropagate the cost $U$ directly to the network $f$.
This process is facilitated by the closed-form solution of the A$^*$ algorithm. For the sake of simplicity, we employ the differentiable A$^*$ algorithm as introduced by \cite{yonetani2021path}. 
It transforms the node selection into an \texttt{argsoftmax} operation and reinterprets node expansion as a series of convolutions, leveraging efficient implementations and AutoDiff in PyTorch \citep{paszke2019pytorch}.

Intuitively, this optimization process involves iterative adjustments.
On one hand, the $h$ map enables the A$^*$ algorithm to efficiently identify the optimal path, but within a confined search area. 
On the other hand, the A$^*$ algorithm's independence from labels allows further refinement of the network.
This is achieved by the back-propagating of the search area and length cost through the differentiable A$^*$ reasoning.
This mutual connection encourages the network to generate increasingly smaller search areas over time, enhancing overall efficiency.
As a result, the network inclines to focus on more relevant areas, marked by reduced low-level reasoning costs, improving the overall search quality.

\subsection*{Example 1.B: Local Path Planning}

\paragraph{Background}
End-to-end local planning, which integrates perception and planning within a single model, has recently attracted considerable interest, particularly for its potential to enable efficient inference through data-driven methods such as reinforcement learning \citep{hoeller2021learning, wijmans2019dd, lee2024learning, ye2021auxiliary} and imitation learning \citep{sadat2020perceive, shah2023gnm, loquercio2021learning}. Despite these advancements, significant challenges persist. Reinforcement learning-based methods often suffer from sample inefficiency and difficulties in directly processing raw, dense sensor inputs, such as depth images. Without sufficient guidance during training, reinforcement learning struggles to converge on an optimal policy that generalizes well across various scenarios or environments.
Conversely, imitation learning relies heavily on the availability and quality of labeled trajectories. Obtaining these labeled trajectories is particularly challenging for robotic systems that operate under diverse dynamics models, thereby limiting their broad applicability in flexible robotic systems.

\begin{figure}
    \centering
    \includegraphics[width=\linewidth]{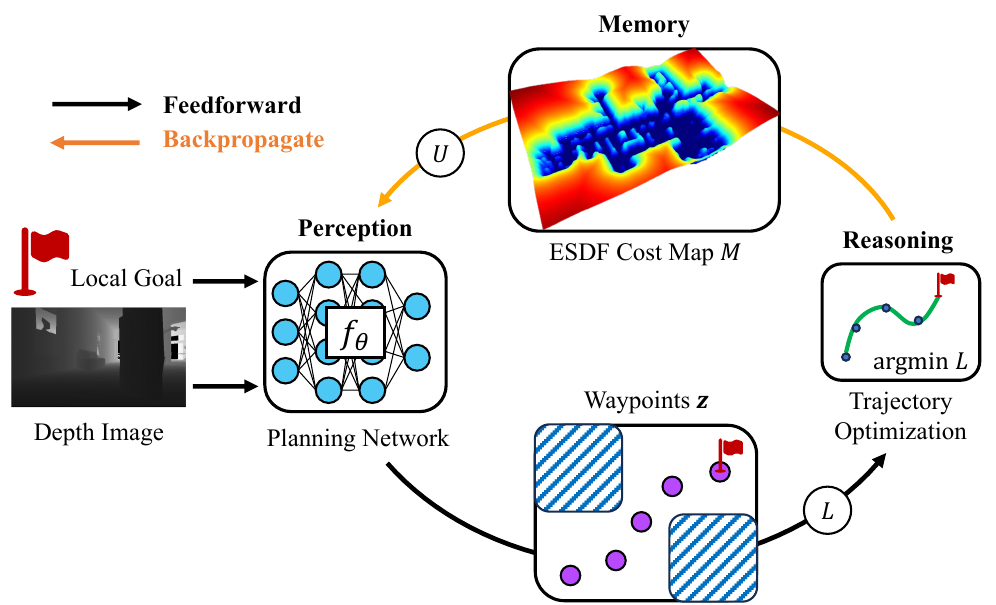}
    \caption{The framework of iPlanner. The higher-level network predicts waypoints, which are interpolated by the lower-level optimization to ensure path continuity and smoothness.}
    \label{fig:iplanner}
\end{figure}

\paragraph{Approach}
To address these challenges, we introduce IL to local planning and refer to it as imperative local planning (iPlanner), as depicted in \fref{fig:iplanner}. Instead of predicting a continuous trajectory directly, iPlanner uses the network to generate sparse waypoints, which are then interpolated using a trajectory optimization engine based on a cubic spline. This approach leverages the strengths of both neural and symbolic modules: neural networks excel at dynamic obstacle detection, while symbolic modules optimize multi-step navigation strategies under dynamics. By enforcing the network output sparse waypoints rather than continuous trajectories, iPlanner effectively combines the advantages of both modules.
Specifically, iPlanner can be formulated as:
\begin{subequations}\label{eq:iplanner}
\begin{align}
\min_{{\bm{\theta}}} \quad & U\left(f(\bm{x}, {\bm{\theta}}), g({\bm{\mu}}^*), M({\bm{\mu}}^*)\right), \\
\textrm{s.t.} \quad & {\bm{\mu}}^* = \arg\min_{{\bm{\mu}}\in \mathbb{T}} L(\bm z, {\bm{\mu}} ),
\end{align}
\end{subequations}
where $\bm x$ denotes the system inputs including a local goal position and the sensor measurements such as depth images, $\theta$ is the parameters of a network $f$, $\bm z = f(\bm{x}, \bm{\theta})$ denotes the generated waypoints, which are subsequently optimized by the path optimizer $g$, $\bm{\mu}$ represents the set of valid paths within the constrained space $\mathbb{T}$.
The optimized path, ${\bm{\mu}}^*$, acts as the optimal solution to the LL cost $L$, which is defined by tracking the intermediate waypoints and the overall continuity and smoothness of the path:
\begin{equation}
    L \doteq D(\bm{\mu}, \bm z) + A(\bm{\mu}), \label{eq:low-iplanner}
\end{equation}
where $D(\bm{\mu}, \bm z)$ measures the path continuity based on the \nth{1}-order derivative of the path and $A(\bm{u})$ calculates the path smoothness based on \nth{2}-order derivative. Specifically, we employ the cubic spline interpolation which has a closed-form solution to ensure this continuity and smoothness. This is also essential for generating feasible and efficient paths.
On the other hand, the UL cost $U$ is defined as:
\begin{equation}
    U \doteq w_G C^G({\bm{\mu}}^*, x_{\text{goal}}) + w_L C^L({\bm{\mu}}^*) + w_M M({\bm{\mu}}^*),
\end{equation}
where $C^G({\bm{\mu}}^*, x_{\text{goal}})$ measures the distance from the endpoint of the generated path to the goal $x_{\text{goal}}$, assessing the alignment of the planned path with the desired destination; $C^L({\bm{\mu}}^*)$ represents the motion loss, encouraging the planning of shorter paths to improve overall movement efficiency; $M({\bm{\mu}}^*)$ quantifies the obstacle cost, utilizing information from a pre-built Euclidean signed distance fields (ESDF) map to evaluate the path safety; and $w_G$, $w_L$, and $w_M$ are hyperparameters, allowing for adjustments in the planning strategy based on specific performance objectives.

\paragraph{Optimization}
During training, we leverage the AutoDiff capabilities provided by PyTorch \citep{paszke2019pytorch} to solve the BLO in IL. For the LL trajectory optimization, we adopt the cubic spline interpolation implementation provided by PyPose \citep{wang2023pypose}. It supports differentiable batched closed-form solutions, enabling a faster training process. Additionally, the ESDF cost map is convoluted with a Gaussian kernel, enabling a smoother optimization space. 
This setup allows the loss to be directly backpropagated to update the network parameters in a single step, rather than requiring iterative adjustments. 
As a result, the UL network and the trajectory optimization can mutually improve each other, enabling a self-supervised learning process.

\subsection{First-order Optimization}\label{sec:1st-order}

We next illustrate the scenario that the LL cost $L$ in \eqref{eq:low-il} uses first-order optimizers such as GD. Because GD is a simple differentiable iterative method, one can leverage \textit{unrolled optimization} listed in \algref{alg:main} to solve BLO via AutoDiff.
It has been theoretically proved by \cite{ji2021bilevel,ji2022will} that when the LL problem is strongly convex and smooth, unrolled optimization with LL gradient descent can approximate the hypergradients $\nabla_{\bm{\theta}}$ and $\nabla_{\bm{\gamma}}$ up to an error that decreases linearly with the number of GD steps.
For the \textit{implicit differentiation}, we could leverage the optimality conditions of the LL optimization to compute the implicit gradient.
To be specific, using Chain rule and optimality of ${\bm{\mu}}^*$ and ${\bm{\nu}}^*$, we have the stationary points
\begin{equation}\label{eq:stationary_condition}
    \frac{\partial L}{\partial {\bm{\mu}}} ({\bm{\mu}}^*,{\bm{\nu}}^*,{\bm{\theta}},{\bm{\gamma}})=\frac{\partial L}{\partial {\bm{\nu}}} ({\bm{\mu}}^*,{\bm{\nu}}^*,{\bm{\theta}},{\bm{\gamma}})=0,
\end{equation}
which, by taking differentiation over ${\bm{\theta}}$ and ${\bm{\gamma}}$ and noting the implicit dependence of ${\bm{\mu}}^*$ and ${\bm{\nu}}^*$ on ${\bm{\theta}}$ and ${\bm{\gamma}}$, yields
\begin{subequations}
    \begin{align}\label{eq:uv}
   \underbrace{\frac{\partial^2L}{\partial f\partial {\bm{\mu}}}\frac{\partial f}{\partial {\bm{\theta}}}}_{A} +  \underbrace{\frac{\partial^2 L}{\partial {\bm{\mu}}^2}}_{B} \frac{\partial{\bm{\mu}}^*}{\partial {\bm{\theta}}} +  \underbrace{\frac{\partial^2L}{\partial{\bm{\nu}}\partial{\bm{\mu}}}}_{C}\frac{\partial{\bm{\nu}}^*}{\partial {\bm{\theta}}} &= 0,
    \\    \underbrace{ \frac{\partial^2L}{\partial f\partial {\bm{\nu}}}\frac{\partial f}{\partial {\bm{\theta}}}}_{\widetilde A} +  \underbrace{\frac{\partial^2 L}{\partial {\bm{\nu}}^2}}_{\widetilde B}\frac{\partial{\bm{\nu}}^*}{\partial {\bm{\theta}}} +  \underbrace{\frac{\partial^2L}{\partial{\bm{\mu}}\partial{\bm{\nu}}}}_{\widetilde C}\frac{\partial{\bm{\mu}}^*}{\partial {\bm{\theta}}} &= 0.
    \end{align}
\end{subequations}
Assume that the concatenated matrix $ \Big[\begin{array}{c} B,\; C \\
    \widetilde C, \; \widetilde B
\end{array}\Big]$ invertible at ${\bm{\mu}}^*,{\bm{\nu}}^*,{\bm{\theta}},{\bm{\gamma}}$. Then, it can be derived from the linear equations in \eqref{eq:uv} that the implicit gradients $\frac{\partial {\bm{\mu}}^*}{\partial {\bm{\theta}}}$ and $\frac{\partial{\bm{\nu}}^*}{\partial {\bm{\theta}}}$
\begin{equation}\label{eq:2nd-solution}
   \bigg[\begin{array}{c} \frac{\partial {\bm{\mu}}^*}{\partial {\bm{\theta}}} \\
\frac{\partial {\bm{\nu}}^*}{\partial {\bm{\theta}}}
\end{array}\bigg] = -  \bigg[\begin{array}{c} B,\; C \\
    \widetilde C, \; \widetilde B
\end{array}\bigg]^{-1} \bigg[\begin{array}{c}
A \\ \widetilde A
\end{array}\bigg],
\end{equation}
which, combined with \eqref{eq:solution}, yields the UL gradient $\nabla_{\bm{\theta}} U$ as 
\begin{subequations}
\begin{align}\label{eq:U_grad}
    \nabla_{\bm{\theta}} U =&\frac{\partial U}{\partial f}\frac{\partial f}{\partial {\bm{\theta}}} +   \bigg[\begin{array}{c} \frac{\partial U}{\partial {\bm{\mu}}^*} \\
\frac{\partial U}{\partial {\bm{\nu}}^*}
\end{array}\bigg]^\top \bigg[\begin{array}{c} \frac{\partial {\bm{\mu}}^*}{\partial {\bm{\theta}}} \\
\frac{\partial {\bm{\nu}}^*}{\partial {\bm{\theta}}}
\end{array}\bigg]
\\ = &\frac{\partial U}{\partial f}\frac{\partial f}{\partial {\bm{\theta}}} - \underbrace{\bigg[\begin{array}{c} \frac{\partial U}{\partial {\bm{\mu}}^*} \\
\frac{\partial U}{\partial {\bm{\nu}}^*}
\end{array}\bigg]^\top  \bigg[\begin{array}{c} B,\; C \\
    \widetilde C, \; \widetilde B
\end{array}\bigg]^{-1} }_{D}\bigg[\begin{array}{c}
A \\ \widetilde A
\end{array}\bigg]
\end{align}
\end{subequations}
where the vector-matrix-inverse-product $D$ can be efficiently approximated by solving a linear system w.r.t.~a vector $v$
\begin{equation}
\begin{aligned}\label{eq:ls}
\min_{v}\quad \frac{1}{2}v^\top\bigg[\begin{array}{c} B,\; C \\
    \widetilde C, \; \widetilde B
\end{array}\bigg] v + \bigg[\begin{array}{c} \frac{\partial U}{\partial {\bm{\mu}}^*} \\
\frac{\partial U}{\partial {\bm{\nu}}^*}
\end{array}\bigg]^\top v.
\end{aligned}
\end{equation}
A similar linear system can be derived for computing $\nabla_{\bm{\gamma}} U$. 
Then, the practical methods should first use a first-order optimizer to solve the lower-level problem in~\eqref{eq:low-il} to obtain approximates $\hat{\bm{\mu}}$ and $\hat{\bm{\nu}}$ of the solutions ${\bm{\mu}}^*$ and ${\bm{\nu}}^*$, which are then incorporated into \eqref{eq:U_grad} to obtain an approximate $\widehat\nabla_{\bm{\theta}} U$ of the upper-level gradient $\nabla_{\bm{\theta}} U$ (similarly for $\widehat\nabla_{\bm{\gamma}} U$).
Then, the upper-level gradient approximates $\widehat\nabla_{\bm{\theta}} U$ and $\widehat\nabla_{\bm{\gamma}} U$ are used to optimize the target variables ${\bm{\theta}}$ and ${\bm{\gamma}}$.

\paragraph{Approximation} As mentioned in \sref{sec:solution}, one could approximate the solutions in \eqref{eq:U_grad} by assuming the second-order derivatives are small and thus can be ignored. In this case, we can directly use the fully first-order estimates without taking the implicit differentiation into account as 
\begin{equation}\label{eq:first_approx}
 \widehat\nabla_{\bm{\theta}} U = \frac{\partial U}{\partial f} \frac{\partial f}{\partial {\bm{\theta}}},\quad \widehat\nabla_{\bm{\gamma}} U = \frac{\partial U}{\partial M} \frac{\partial M}{\partial {\bm{\gamma}}},
\end{equation}
which are evaluated at the approximates $\hat {\bm{\mu}}$ and $\hat{\bm{\nu}}$ of the LL solution ${\bm{\mu}}^*,{\bm{\nu}}^*$. 
We next demonstrate it with an example of inductive logical reasoning using first-order optimization.

\paragraph{Discussion on theoretical guarantee and optimality}
To achieve theoretical guarantees, one sufficient condition is that the LL function $L$ in \eqref{eq:low-il} is twice differentiable and strongly convex respect to $\phi$. The UL function $U$ can be possibly non-convex. Based on \citep[Theorem 1]{ji2021bilevel}, the proposed first-order method, which solves the linear system in \eqref{eq:ls} and uses $\hat{\bm{\mu}}$ and $\hat{\bm{\nu}}$ of the solutions ${\bm{\mu}}^*$ and ${\bm{\nu}}^*$, can achieve the first-order $\epsilon$-accurate stationary optimality, i.e., $\|\nabla_\psi\|\leq \epsilon$, if we use $\log\frac{1}{\epsilon}$-level steps for the  approximation in \eqref{eq:first_approx}. Moreover, when the UL function $U$ in \eqref{eq:ggs} is strongly convex with respect to $\theta$ and $\gamma$, we can further show that the obtained solution achieves the global optimality, based on \citep[Theorem 9]{ji2021lower}.

\subsection*{Example 2: Inductive Logical Reasoning}

Given few known facts, logical reasoning aims at deducting the truth value of unknown facts with formal logical rules~\citep{iwanska1993logical,overton2013competence,halpern2013thought}.
Different types of logic have been invented to address distinct problems from various domains, including propositional logic~\citep{wang2019satnet}, linear temporal logic~\citep{xie2021embedding}, and first-order-logic (FOL)~\citep{cropper2021popper}.
Among them, FOL decomposes a logic \textit{term} into \textit{lifted} \textit{predicates} and \textit{variables} with quantifiers, which has strong generalization power and can be applied to arbitrary entities with any compositions.
Due to the strong capability of generalization, FOL has been widely leveraged in knowledge graph reasoning~\citep{yang2019nlil} and robotic task planning~\citep{chitnis2022nsrt}. 
However, traditional FOL requires a human expert to carefully design the predicates and rules, which is tedious.
Automatically summarizing the FOL predicates and rules is a long-standing problem, which is known as inductive logic programming (ILP).
However, existing works study ILP only in simple structured data like knowledge graphs.
To extend this research into robotics, we will explore how IL can make ILP work with high-dimensional data like RGB images.

\paragraph{Background}
One stream of the solutions to ILP is based on forward search algorithms~\citep{cropper2021popper,shindo2023aILP,cropper2024maxsynth}.
For example, Popper constructs answer set programs based on failures, which can significantly reduce the hypothesis space~\citep{cropper2021popper}.
However, as FOL is a combinatorial problem, search-based methods can be extremely time-consuming as the samples scale up.
Recently, some works introduced neural networks to assist the search process~\citep{yang2017neurallp,yang2019nlil,yang2022logicdef} or directly implicitly represent the rule~\citep{dong2019nlm}.
To name a few, NeurlLP re-formulates the FOL rule inference into the multi-hop reasoning process, which can be represented as a series of matrix multiplications~\citep{yang2017neurallp}.
Thus, learning the weight matrix becomes equivalent to inducting the rules.
Neural logic machines, on the other hand, designed a new FOL-inspired network structure, where the rules are implicitly stored in the network weights~\citep{dong2019nlm}.

Despite their promising results in structured symbolic datasets, such as binary vector representations in BlocksWorld \citep{dong2019nlm} and knowledge graphs \citep{yang2017neurallp}, their capability of handling high-dimensional data like RGB images is rarely explored.

\begin{figure}
    \centering
    \includegraphics[width=\linewidth]{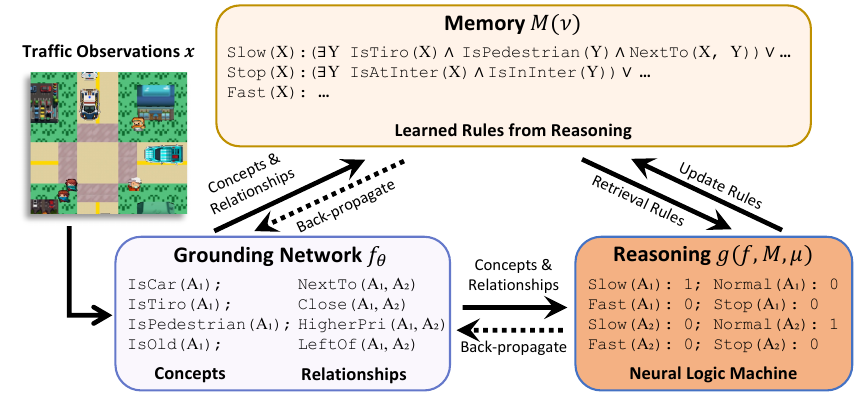}
    \caption{The iLogic pipeline, which simultaneously conducts rule induction and high-dimensional data grounding.}
    \label{fig:iLogic}
\end{figure}

\paragraph{Approach}
To address this gap, we verify the IL framework with the visual action prediction (VAP) task in the LogiCity \citep{li2024logicity}, a recent logical reasoning benchmark.
In VAP, a model must simultaneously discover traffic rules, identify the concept of agents and their spatial relationships, and predict the agents' actions.
As shown in \fref{fig:iLogic}, we utilize a grounding network $f$ to predict the agent concepts and spatial relationships, which takes the observations such as images as the model input.
The predicted agent concepts and their relationships are then sent to the reasoning module for rule induction and action prediction. The learned rules from the reasoning process are stored in the memory and retrieved when necessary.
For example, the grounding network may output concepts ``\texttt{IsTiro($A_1$)}'', ``\texttt{IsPedestrian($A_2$)},'' and their relationship ``\texttt{NextTo($A_1$, $A_2$)}''. Then the reasoning engine applies the learned rules, ``\texttt{Slow($X$)}$\leftarrow$($\exists$ $Y$ \texttt{IsTiro($X$)} $\land$ \texttt{IsPedestrian($Y$)} $\land$ \texttt{NextTo($X$, $Y$)}) $\lor$ ...'', stored in the memory module to infer the next actions ``\texttt{Slow($A_1$)}'' or ``\texttt{Normal($A_2$)}'', as is displayed in \fref{fig:vap:example}. Finally, the predicted actions and the observed actions from the grounding networks are used for loss calculation.
In summary, we formulate this pipeline in \eqref{eq:iLogic} and refer to it as imperative logical reasoning (iLogic):
\begin{subequations}\label{eq:iLogic}
\begin{align}
    \min_{{\bm{\theta}}}& \;\; U(f({\bm{\theta}},\bm{x}), g({\bm{\mu}}^*), M({\bm{\nu}}^*)), \\
    \operatorname{s.t.}& \;\; {\bm{\mu}}^*,~{\bm{\nu}}^*=\argminC_{{\bm{ {\bm{\mu}},~{\bm{\nu}}}}} \; L(f, g({M, \bm{\mu}}), M({\bm{\nu}})).
\end{align}
\end{subequations}
Specifically, in the experiments we use the same cross-entropy function for the UL and LL costs, i.e., $U \doteq L \doteq -\sum a_{i}^o\log a_{i}^p$, where $a_{i}^o$ and $a_{i}^p$ are the observed and predicted actions of the $i$-th agent, respectively. This results in a self-supervised simultaneous grounding and rule induction pipeline for logical reasoning.
Additionally, $f$ can be any grounding networks and we use a feature pyramid network (FPN)~\citep{lin2017fpn} as the visual encoder and two MLPs for relationships; $g$ can be any logical reasoning engines and we use a neural logical machine (NLM) \citep{dong2019nlm} with parameter ${\bm{\mu}}$; and $M$ can be any memory modules storing the learned rules and we use the MLPs in the NLM parameterized by ${\bm{\nu}}$. 
More details about the task, model, and list of concepts are presented in the \sref{sec:exp-iLogic}.

\begin{figure}
    \centering
    \includegraphics[width=\linewidth]{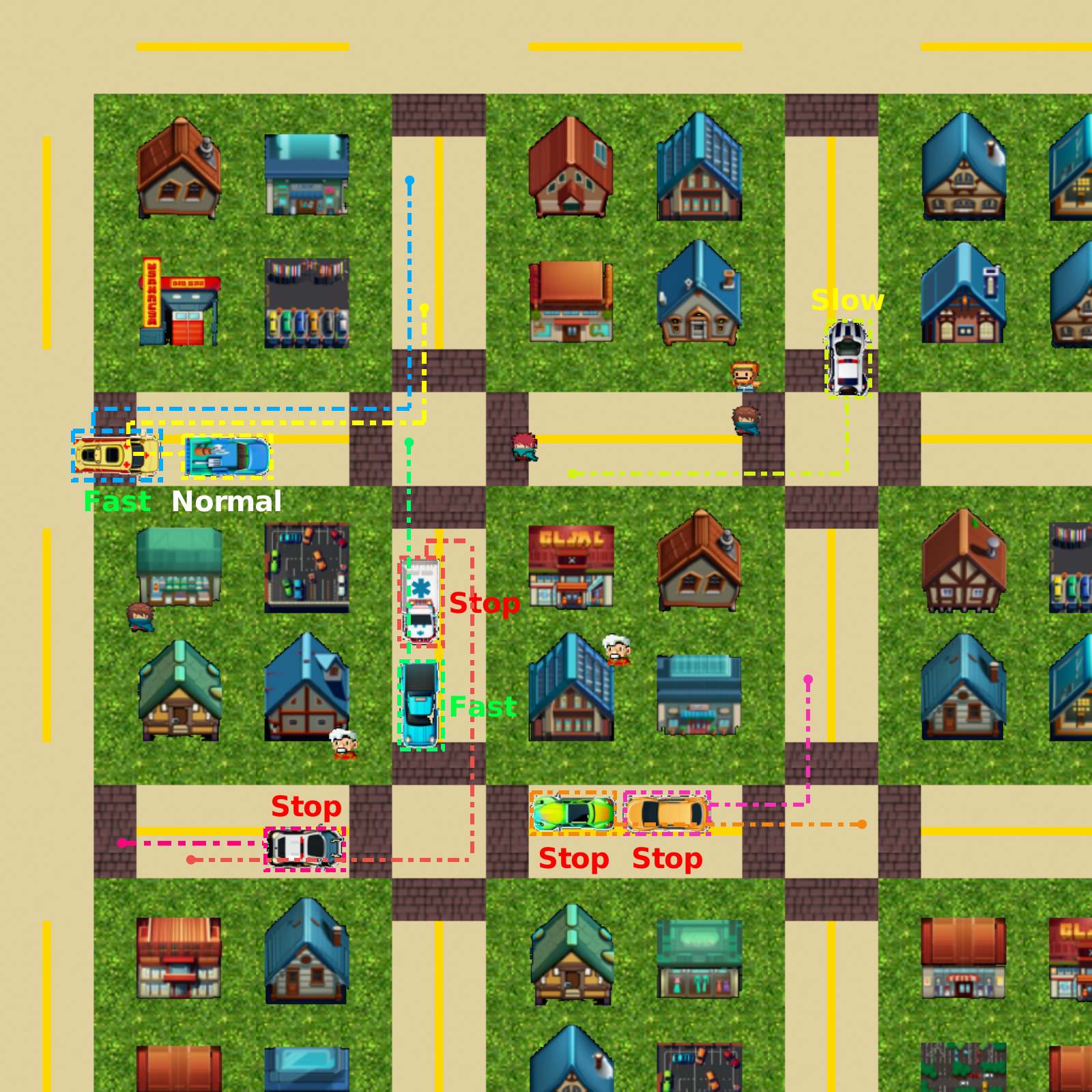}
    \caption{One snapshot of the visual action prediction task of LogiCity benchmark. The next actions of each agent are reasoned by our iLogic based on groundings and learned rules.}
    \label{fig:vap:example}
\end{figure}

\paragraph{Optimization}
Given that gradient descent algorithms can update both the grounding networks and the NLM, we can apply the first-order optimization to solve iLogic.
The utilization of task-level action loss enables efficient self-supervised training, eliminating the necessity for explicit concept labels. Furthermore, BLO, compared to single-level optimization, helps the model concentrate more effectively on learning concept grounding and rule induction, respectively. This enhances the stability of optimization and decreases the occurrence of sub-optimal outcomes. Consequently, the model can learn rules more accurately and predict actions with greater precision.

\subsection{Constrained Optimization}\label{sec:constrained}

We next illustrate the scenarios in which the LL cost $L$ in \eqref{eq:low-il} is subject to the general constraints in \eqref{eq:il-constraint}. We discuss two cases with equality and inequality constraints, respectively.
Constrained optimization is a thoroughly explored field, and related findings were presented in \citep{dontchev2009implicit} and summarized in \citep{gould2021deep}.
This study will focus on the integration of constrained optimization into our special form of BLO \eqref{eq:il} under the framework of IL.

\paragraph{Equality Constraint}

In this case, the constraint in \eqref{eq:il-constraint} is
\begin{equation}\label{equal_constraint}
    \xi(M({\bm{\gamma}}, {\bm{\nu}}), {\bm{\mu}}, f) = 0.
\end{equation}
Recall that $\bm \phi = [{\bm{\mu}}^\top,{\bm{\nu}}^\top]^\top$ is a vector concatenating the symbol-like variables ${\bm{\mu}}$ and ${\bm{\nu}}$, and $\bm \psi =[{\bm{\theta}}^\top,{\bm{\gamma}}^\top]^\top$ is a vector concatenating the neuron-like variables ${\bm{\theta}}$ and ${\bm{\gamma}}$.
Therefore, the implicit gradient components can be expressed as differentiation of $\bm \phi^*$ to $\bm \psi$, and we have
\begin{equation}
  \nabla \bm \phi^*(\bm \psi) =
  \begin{pmatrix}
 \frac{\partial {\bm{\mu}}^*}{\partial {\bm{\theta}}} &  \frac{\partial {\bm{\mu}}^*}{\partial {\bm{\gamma}}} \\
\frac{\partial {\bm{\nu}}^*}{\partial {\bm{\theta}}} & \frac{\partial {\bm{\nu}}^*}{\partial {\bm{\gamma}}} \\
\end{pmatrix}.
\end{equation}
\begin{lemma}
Assume the LL cost $L(\cdot)$ and the constraint $\xi(\cdot)$ are \nth{2}-order differentiable near $({\bm{\theta}},{\bm{\gamma}},\bm \phi^*({\bm{\theta}},{\bm{\gamma}}))$ and the Hessian matrix $H$ below is invertible, we then have 
\begin{equation}
    \nabla \bm \phi^*(\bm \psi) = H^{-1} L_{\bm{\phi}}^\top[L_{\bm{\phi}} H^{-1}L_{\bm{\phi}}^\top]^{-1}(L_{\bm{\phi}} H^{-1}L_{\bm{\phi \psi}}-L_{\bm{\psi}}) - H^{-1}L_{\bm{\phi \psi}},
\end{equation}
where the derivative matrices are given by
\begin{equation}
\begin{aligned}\label{eq:cons_gra}
L_{\bm{\phi}} = & \frac{\partial  \xi(M({\bm{\gamma}}, {\bm{\nu}}^*), {\bm{\mu}}^*, f)}{\partial {\bm{\phi}}},\quad L_{\bm{\psi}} =  \frac{\partial  \xi(M({\bm{\gamma}}, {\bm{\nu}}^*), {\bm{\mu}}^*, f)}{\partial {\bm{\psi}}},
\\ L_{\bm{\phi \psi}} =& \frac{\partial^2 L(f, g({\bm{\mu}}^*), M({\bm{\gamma}}, {\bm{\nu}}^*))}{\partial \bm\psi\partial \bm\phi} - \lambda \frac{\partial^2  \xi(M({\bm{\gamma}}, {\bm{\nu}}^*), {\bm{\mu}}^*, f)}{\partial \bm\psi\partial \bm\phi},
\\ H =& \frac{\partial^2 L(f, g({\bm{\mu}}^*), M({\bm{\gamma}}, {\bm{\nu}}^*))}{\partial^2 \bm\phi} - \lambda \frac{\partial^2  \xi(M({\bm{\gamma}}, {\bm{\nu}}^*), {\bm{\mu}}^*, f)}{\partial^2 \bm\phi}, 
\end{aligned}
\end{equation}
and duel variable $\lambda\in\mathbb{R}$ satisfies $\lambda L_{\bm{\phi}} = \frac{\partial L\xi(f, g({\bm{\mu}}), M({\bm{\gamma}}, {\bm{\nu}}))}{\partial \bm\phi}$.
\end{lemma}

\begin{proof}
First, since this is a constrained problem, one can define the Lagrangian function
\begin{equation}
\begin{aligned}
\mathfrak{L}(\bm\psi,\bm\phi,\lambda) = L(f, g({\bm{\mu}}), M({\bm{\gamma}}, {\bm{\nu}}))-\lambda \xi(M({\bm{\gamma}}, {\bm{\nu}}), {\bm{\mu}}, f).
\end{aligned}
\end{equation}
Assume $\bm\phi^*$ to be the solution given input $\bm\psi$. Then, a $\lambda$ exists such that the Lagrangian has zero gradients at $\bm\phi^*$ and $\lambda$. Thus, taking derivatives w.r.t.~$\bm\phi$ and $\lambda$ yields
\begin{equation}
\begin{aligned}\label{eq:stationary}
    \Bigg[\begin{array}{c} \frac{\partial L(f, g({\bm{\mu}}^*), M({\bm{\gamma}}, {\bm{\nu}}^*))}{\partial \bm\phi}-\lambda\frac{\partial \xi(M({\bm{\gamma}}, {\bm{\nu}}^*), {\bm{\mu}}^*, f)}{\partial \bm\phi} \\
    \xi(M({\bm{\gamma}}, {\bm{\nu}}^*), {\bm{\mu}}^*, f)
    \end{array}\Bigg] = \mathbf{0}.
\end{aligned}
\end{equation}
Note that if $\lambda=0$, then this means that the gradient $\frac{\partial L(f, g({\bm{\mu}}^*), M({\bm{\gamma}}, {\bm{\nu}}^*))}{\partial \bm\phi}=\mathbf{0}$ (i.e., the solution of the unconstrained problems already satisfies the constraint). If $\lambda>0$, then we can obtain from \eqref{eq:stationary} that 
\begin{equation}    
    \frac{\partial L(f, g({\bm{\mu}}^*), M({\bm{\gamma}}, {\bm{\nu}}^*))}{\partial \bm\phi}=\lambda\frac{\partial \xi(M({\bm{\gamma}}, {\bm{\nu}}^*), {\bm{\mu}}^*, f)}{\partial \bm\phi}, 
\end{equation}
which means that the gradient of $L$ w.r.t.~$\bm\phi$ has the same direction as the gradient of constraint $\xi$ w.r.t.~$\bm\phi$ at the optimal solution ${\bm{\mu}}^*$ and ${\bm{\nu}}^*$.
Taking the derivative of \eqref{eq:stationary} w.r.t.~$\bm\psi$ and noting that $\bm\phi^*$ and $\lambda$ are functions of $\bm\psi$, thus
\begin{subequations}
\begin{align}
&\frac{\partial^2 L}{\partial \bm\psi\partial \bm\phi}  + \frac{\partial^2 L}{\partial^2 \bm\phi}  \nabla \bm\phi^*- \frac{\partial \xi}{\partial \bm\phi} \nabla \lambda - \lambda (\frac{\partial^2 \xi}{\partial \bm\psi\partial \bm\phi}+\frac{\partial^2 \xi}{\partial^2 \bm\phi}\nabla \bm\phi^*) = \mathbf{0}, \label{eq:cons_imp_1}
\\&\frac{\partial \xi}{\partial \bm\psi} + \frac{\partial \xi}{\partial \bm\phi}\nabla \bm\phi^*=\mathbf{0},\label{eq:cons_imp_2}
\end{align}
\end{subequations}
which together indicates that
\begin{equation}
    \nabla \bm\phi^* = H^{-1} L_{\bm\phi}^\top[L_{\bm\phi} H^{-1}L_{\bm\phi}^\top]^{-1}(L_{\bm\phi} H^{-1}L_{\bm\phi \bm\psi}-L_{\bm\psi}) - H^{-1}L_{\bm\phi \bm\psi},
\end{equation}
which completes the proof. 
\end{proof}

\paragraph{Inequality Constraint}

In this case, the constraint in \eqref{eq:il-constraint} is
\begin{equation}\label{inequal_constraint}
\xi(M({\bm{\gamma}}, {\bm{\nu}}), {\bm{\mu}}, f) \leq 0.
\end{equation}
The analysis is similar to the equality case, with minor differences. 
Since the solution $\bm\phi^*$ with the inequality constraint \eqref{inequal_constraint} is a local minimum of the original inequality-constrained problem, it is reasonable to assume that the inequality constraints are active at $\bm\phi^*$.
Assume $\bm\phi^*(\bm\psi)$ exists, functions $L(\cdot)$ and $\xi(\cdot)$ are \nth{2}-order differentiable in the neighborhood of $(\bm\psi,\bm\phi^*(\bm\psi))$, and the inequality constraint is active at $\bm\phi^*(\bm\psi)$. Then, we can derive the same implicit derivative $\nabla \bm\phi^*(\bm\psi)$ as in \eqref{eq:cons_gra} of the equality case, except that the dual variable $\lambda\leq 0$ and $\nabla \bm\phi^*(\bm\psi)$ needs to be computed using one-sided partial $\lambda=0$ due to non-differentiability, meaning we approach the point from only one direction (either left or right) to define the gradient, as the function may have abrupt changes or `kinks' at that point. 

\begin{proof}
    Because the inequality constraints are active at $\bm\phi^*$ (i.e., $\xi(M({\bm{\gamma}}, {\bm{\nu}}^*), {\bm{\mu}}^*, f)=0$), the remaining proof exactly follows the same steps as in the equality case.
\end{proof}

While we have derived closed-form solutions for the equality constraint in \eqref{equal_constraint} and the inequality constraint in \eqref{inequal_constraint}, calculating the matrix inversions $H^{-1}$ and $[L_{\bm\phi} H^{-1}L_{\bm\phi}^\top]^{-1}$ can be computationally expensive. 
To mitigate this, one may employ the approach discussed in \sref{sec:close-form}, which involves approximating the matrix-inverse-vector product $H^{-1}y$ by solving the linear system $\min_v \frac{1}{2}v^\top Hv - v^\top y$.
This complicated computing process can sometimes be simplified by specific robot settings.
We next present an example of robot control illustrating this process.

\paragraph{Discussion on theoretical guarantee and optimality} When the Hessian matrices $\frac{\partial^2 L} {\partial^2 \phi}$ and $\frac{\partial^2 \xi} {\partial^2 \phi}$  of the LL function $L$ and the constraint function $\xi$ are invertible, we can guarantee the matrix $H$ in \eqref{eq:cons_gra} is invertible. One sufficient condition for such invertibility is that $L$ and $\xi$ are strongly convex. 
Motivated by the approximate Karush-Kuhn-Tucker (KKT) condition \citep{andreani2010new}, which is characterized as an
optimality condition for nonlinear program, we can use the residual function in \citep[Eq. (15)]{yao2024constrained} as an optimality measurement. As a result, the optimality can be established and guaranteed under these conditions.

\subsection*{Example 3: Model Predictive Control}

Modeling and control of nonlinear dynamics precisely is critical for robotic applications such as aerial manipulation~\citep{he2023image, geng2020cooperative} and off-road navigation~\citep{triest2023learning}. 
Although classic methods based on physics have excellent generalization abilities, they rely heavily on problem-specific parameter tuning. 
Methods like state feedback~\citep{mellinger2011minimum} and optimal control~\citep{garcia1989model, ji2020robust} require precise system modeling to predict and control dynamics effectively.
However, it is challenging to accurately model unpredictable factors such as wind gusts, boundary layer effects, uneven terrain, or hidden dynamics of chaotic nonlinear effects. 

\paragraph{Background}
In recent years, researchers have spent efforts to integrate physical modeling into data-driven methods and have gained remarkable success in the field of system control~\citep{o2022neural,hu2023off}.
For example, \cite{amos2018differentiable} developed the differentiable model predictive control (MPC) to combine the physics-based modeling with data-driven methods, enabling learning dynamic models and control policies in an end-to-end manner~\citep{jin2020pontryagin}.
However, many prior studies depend on expert demonstrations or labeled data for supervised learning~\citep{jin2022learning, song2022policy}. While supervised learning yields effective results in the training datasets, it suffers from challenging conditions such as unseen environments and external disturbances.

\begin{figure}[t]
    \centering
    \includegraphics[width=0.98\linewidth]{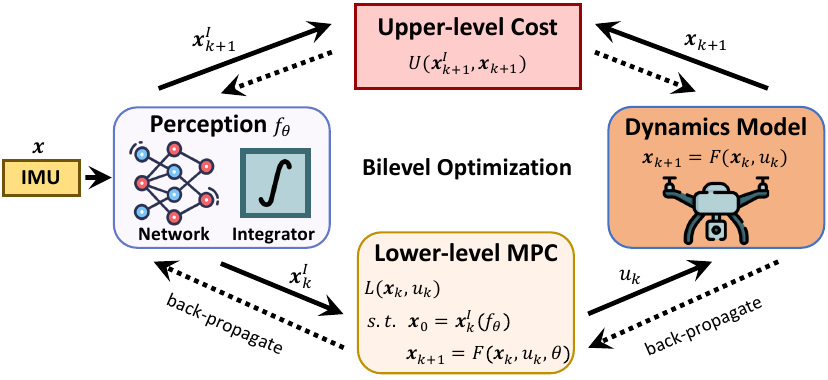}
    \caption{The framework of iMPC. The IMU model predicts the current state $\bm{x}_k^I$. The differentiable MPC minimizes the lower-level cost $L$ to get the optimal action $\bm{u}_k$, which controls the dynamics model to the next predicted state ($\bm{x}_{k+1}$) and actuates the real UAV to the next state measured by the IMU model ($\bm{x}^I_{k+1}$). The upper-level optimization minimizes the discrepancy between the two states $\bm{x}_{k+1}$ and $\bm{x}^I_{k+1}$.}
    \label{fig:impc}
\end{figure}

\paragraph{Approach}

To eliminate dependence on human demonstrations, we incorporate differentiable MPC into the IL framework, hereafter termed imperative MPC (iMPC). We present a specific application, the IMU-based attitude control for unmanned aerial vehicles (UAVs) as an example, where a network denoises the raw IMU measurements and predicts the UAV attitude, which is then used by MPC for attitude control. It is worth noting that, although designed for attitude control, iMPC can be adjusted to other control problems.

As illustrated in \fref{fig:impc}, iMPC consists of a perception network $f({\bm{\theta}}, \bm{x})$ and a physics-based reasoning engine based on MPC.
Specifically, the encoder $f$ takes raw sensor measurements, and here are IMU measurements $\bm{x}$, to predict the HL parameters for MPC, and here is UAV's current state (attitude) $\bm{x}_{k}^I$.
The MPC takes the current attitude as initial state and solves for the optimal states and control sequence ${\bm{\mu}}^* = \{{\bm{\mu}}_k^*\}_{1:T} = \{\bm{x}_k, \bm{u}_k\}_{1:T}$ over a time horizon $T$ under the constraints of system dynamics $F(\cdot)$, actuator limit, etc.
Following the IL framework, the iMPC can be formulated as:
\begin{subequations}\label{eq:control}
\begin{align}
\min_{{\bm{\theta}}} \quad & U\left(f(\bm{x}; {\bm{\theta}}), g({\bm{\mu}}^*)\right), \\
\textrm{s.t.} \quad & {\bm{\mu}}^* = \arg\min_{{\bm{\mu}}} L(f, g({\bm{\mu}})), \\
&\textrm{s.t.} \quad  \bm{x}_{k+1} = F(\bm{x}_k, \bm{u}_k), \quad k = 1,\cdots,T \\
&\quad \quad~\bm{x}_{0} = \bm{x}(t_0), \\
&\quad \quad~\bm{x}_k \in \mathcal{X}; \; \bm{u}_k \in \mathcal{U},
\end{align}
\end{subequations}
where the LL optimization is the standard MPC optimization to output an optimal control sequence ${\bm{\mu}}^*=\{{\bm{\mu}}_k\}_{1:T}$. Here, $L$ is designed to minimize the tracking error and control effort: 
\begin{equation}\label{eq:mpc_cost1}
    L({\bm{\mu}}_k) \doteq \sum_{k=0}^{T-1}\left(\Delta \bm \mu_k^\top\bm{Q}_k\Delta \bm \mu_k + \bm{p}_k \Delta \bm \mu_k \right),
\end{equation}
where $ \Delta \bm \mu_k = {\bm{\mu}}_k - {\bm{\mu}}_k^{\text{ref}}$, ${\bm{\mu}}_k^{\text{ref}}$ is the reference state trajectory, and $\bm{Q}_k$ and $\bm{q}_k$ are the weight matrices to balance tracking performance and control effort, respectively. 
The optimal control action $\bm{u}_k$ will be commanded to both the modeled UAV (dynamics) and the real UAV, leading to the next state $\bm{x}_{k+1}$ and $\bm{x}^I_{k+1}$ (measured by the IMU network).
The UL cost is defined as the Euclidean distance between $\bm{x}_{k+1}$ and $\bm{x}^I_{k+1}$
\begin{equation}
    U({\bm{\theta}}) \doteq \| \bm{x}_{k+1}^I-\bm{x}_{k+1} \|_2.
\end{equation}
This discrepancy between the predicted states and the IMU network measurement reflects the models's imperfectness, indicating the learnable parameters ${\bm{\theta}}$ in both the network and the MPC module can be adjusted.
Specifically, we adopt an IMU denoising model, AirIMU \citep{qiu2023airimu} as the perception network, which is a neural model that can propagate system uncertainty through a time horizon. The differentiable MPC (Diff-MPC) from PyPose \citep{wang2023pypose} is selected as the physics-based reasoning module.

\paragraph{Optimization}
Similar to other examples, the key step for the optimization is to compute the implicit gradient $\frac{\partial {\bm{\mu}}*}{\partial {\bm{\theta}}}$ from~\eqref{eq:cons_gra}. Thanks to the Diff-MPC which uses a problem-agnostic AutoDiff implementation that applies one extra optimization iteration at the optimal point in the forward pass to enable the automatic gradient computation for the backward pass, our iMPC does not require the computation of gradients for the entire unrolled chain or analytical derivatives via implicit differentiation.
This allows us to backpropagate the UL cost $U$ through the MPC to update the network parameters ${\bm{\theta}}$ in one step.
Additionally, iMPC not only avoids the heavy reliance on annotated labels leveraging the self-supervision from the physical model but also obtains dynamically feasible predictions through the MPC module. This results in a ``physics-infused'' network, enabling more accurate attitude control, as demonstrated in \sref{exp:impc}.

\subsection{Second-order Optimization}\label{sec:2nd-order}

We next illustrate the scenario that LL cost $L$ in \eqref{eq:low-il} requires \nth{2}-order optimizers such as Gauss-Newton (GN) and Levenberg–Marquardt (LM). Because both GN and LM are iterative methods, one can leverage \textit{unrolled differentiation} listed in \algref{alg:main} to solve BLO.
For the second strategy \textit{implicit differentiation}, we could combine \eqref{eq:U_grad} with a linear system solver in \eqref{eq:ls} in \sref{sec:1st-order} to compute the implicit gradient $\nabla_{\bm{\theta}} U$ and  $\nabla_{\bm{\gamma}} U$. Similar to the practical approach with LL first-order optimization, we can first use a second-order optimizer to solve the LL problem in~\eqref{eq:low-il} to obtain approximates $\hat{\bm{\mu}}$ and $\hat{\bm{\nu}}$ of the solutions ${\bm{\mu}}^*$ and ${\bm{\nu}}^*$, which are then incorporated into \eqref{eq:U_grad} to obtain an approximate $\widehat\nabla_{\bm{\theta}} U$ of the UL gradient $\nabla_{\bm{\theta}} U$ 
(similarly for $\widehat\nabla_{\bm{\gamma}} U$).
Then, the UL gradient approximates $\widehat\nabla_{\bm{\theta}} U$ and $\widehat\nabla_{\bm{\gamma}} U$ are used to optimize the target variables ${\bm{\theta}}$ and ${\bm{\gamma}}$.
The first-order approximation without any second-order derivative computations as in \eqref{eq:first_approx} can be also developed for a practical training speedup. 

\paragraph{Discussion on theoretical guarantee and optimality} The theoretical guarantee is similar to that for the first-order methods in \sref{sec:1st-order}, as long as the second-order LL optimizers such as Gauss-Newton and LM approximate $\mu^*$ and $\nu^*$ to the target accuracy $\epsilon$. For example, if the LL problem is strongly convex, the second-order optimization can find an $\epsilon$-accurate solution with an accelerated linear convergence rate \citep[Theorem 9 and 10]{ji2021lower}.

\subsection*{Example 4: SLAM}

Simultaneous localization and mapping (SLAM) is a critical technology in computer vision and robotics \citep{aulinas2008slam,cadena2016past}. 
It aims to simultaneously track the trajectory of a robot and build a map of the environment.
SLAM is essential in diverse robotic applications, such as indoor navigation \citep{wang2017non}, underwater exploration \citep{suresh2020active}, and space missions \citep{dor2021visual}.
A SLAM system generally adheres to a front-end and back-end architecture. Specifically, the front-end is typically responsible for interpreting sensor data and providing an initial estimate of the robot's trajectory and the environment map; the back-end refines the initial estimate by performing global optimization to improve overall accuracy.

\paragraph{Background}

Recent advancements in the field of SLAM have shown that supervised learning-based methods can deliver remarkable performance in front-end motion estimation \citep{wang2021tartanvo, teed2021droid}. These methods apply machine learning algorithms that depend on external supervision, usually in the form of a labeled dataset. Once trained, the model can make estimations without being explicitly programmed for the task. Meanwhile, geometry-based approaches remain crucial for the system's back-end, primarily tasked with reducing front-end drift \citep{qin2018vins,xu2025airslam}. These techniques employ geometrical optimization, such as pose graph optimization (PGO) \citep{pgo_labbe2014online}, to maintain global consistency and minimize trajectory drift. However, their front-end and back-end components operate independently and are connected only in one direction. This means that the data-driven front-end model cannot receive feedback from the back-end for joint error correction. As a result, this decoupled approach might lead to suboptimal performance, adversely affecting the overall performance of the current SLAM systems.

\begin{figure}[t]
    \centering
    \includegraphics[width=\linewidth]{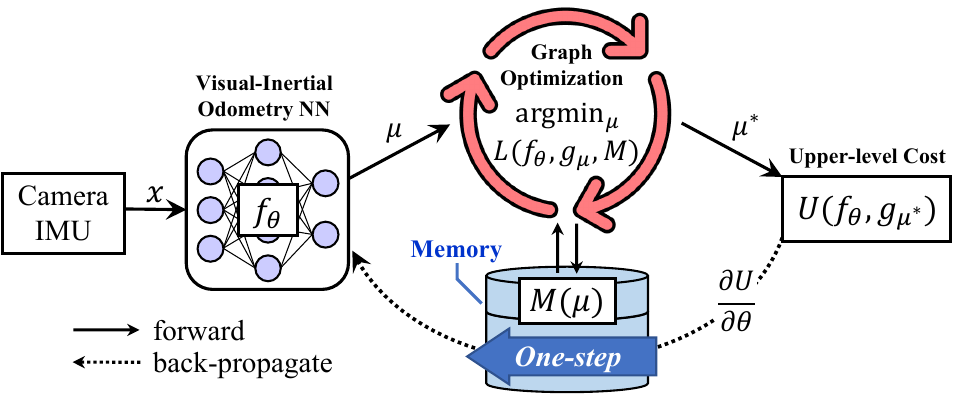}
    \caption{The framework of iSLAM. On the forward path, the odometry module $f$ (front-end) predicts the robot trajectory ${\bm{\mu}}$. Then the pose graph optimization (back-end) minimizes the cost $L$ in several iterations to get optimal poses ${\bm{\mu}}^*$. On the backward path, the cost $U$ is back-propagated through the map $M$ with a ``one-step'' strategy to update the network parameters ${\bm{\theta}}$.}
    \label{fig:iSLAM}
\end{figure}

\paragraph{Approach}

In response to the problem of suboptimal performance of the decoupled approach, we introduce IL into the SLAM system to jointly optimize the front-end and back-end components.
Specifically, we reformulate SLAM as a BLO between the front-end and back-end as:
\begin{subequations}\label{eq:iSLAM}
\begin{align}
    \min_{{\bm{\theta}}}& \;\; U(f({\bm{\theta}},\bm{x}), L({\bm{\mu}}^*)), \\
    \operatorname{s.t.}& \;\; {\bm{\mu}}^*=\arg\min_{{\bm{\mu}}} \; L(f, {\bm{\mu}}, M),
\end{align}
\end{subequations}
where $f({\bm{\theta}},\bm{x})$ is the front-end odometry network with parameter ${\bm{\theta}}$ and input $\bm{x}$; $L(f,{\bm{\mu}},M)$ is the back-end geometry-based reasoning (optimization) which enforces global consistency. Specifically, ${\bm{\mu}}$ is the poses and velocities to be optimized; $M$ is a map (memory) to store the historical estimations; and $U$ is the UL loss for model training. We omit unused symbols in \eqref{eq:il}, and let $U$ depend directly on $L$ instead of $g$ as in the general form. This framework is illustrated in \fref{fig:iSLAM}, which we refer to as imperative SLAM (iSLAM).

The choice of front-end odometry network and the back-end optimization are flexible in iSLAM. 
Here we illustrate one example of using a stereo visual-inertial odometry network as the front-end and a pose-velocity graph optimization as the back-end \citep{fu2024islam}.
Specifically, the front-end odometry consists of two independent efficient modules, i.e., a stereo odometry network and an IMU network \citep{qiu2023airimu}.
The stereo odometry network is responsible for pose estimation, while the IMU network estimates both pose and velocity. These estimations are then incorporated as constraints (edges) in the back-end pose-velocity graph, linking the pose and velocity nodes. This setup allows for the optimization of pose and velocity nodes to reduce inconsistency between the visual and inertial measurements.
As a result, the LL cost $L$ can be any estimation inconsistency between graph nodes:
\begin{equation}
    L = \operatorname{PVG} (\hat p_{1}^{v} \cdots \hat p_{k}^{v}, \hat p_{1}^{i} \cdots \hat p_{k}^{i}, \hat v_{1}^{i} \cdots \hat v_{k}^{i}),
\end{equation}
where $\rm{PVG}$ denotes the summation of the inconsistencies calculated based on edge constraints in the pose-velocity graph; $\hat p_k^v$ and $\hat p_k^i$ are the estimates of transformations between the $(k-1)^{\text{th}}$ and $k^{\text{th}}$ image frame from the stereo odometry network and the IMU network, respectively; $\hat v_k^i$ is the estimated delta velocity from the IMU network between the $(k-1)^{\text{th}}$ and $k^{\text{th}}$ image frame. The pose and velocity nodes are combined to denote the LL variables ${\bm{\mu}}=\{p_{0}, p_{1},\cdots, p_{k}, v_{1}, v_{2}, \cdots, v_{k}\}$.
Intuitively, this LL solves the best poses and velocities based on both the visual and inertial observations, thus achieving better trajectory accuracy by combining the respective strengths of both camera and IMU.

In the process of UL optimization, the front-end model acquires knowledge from the feedback provided by the back-end.
In this example, UL cost $U$ is picked the same as $L$, although they are not necessarily equal in the general IL framework.
Specifically, $U$ is calculated on the optimal estimations ${\bm{\mu}}^*$ after PVGO and then back-propagated to networks $f$ to tune its parameter ${\bm{\theta}}$ using gradient descent. More details about the iSLAM can be found at \sref{exp:islam}.

The structure of iSLAM results in a self-supervised learning method.
At the back-end (lower-level), the robot's path is adjusted through PVGO to ensure geometric consistency between the visual and inertial motion estimations, whereas, at the front-end (upper-level), the model parameters are updated to incorporate the knowledge derived from the back-end, improving its performance and generalizability in unseen environments.
This formulation under the IL framework seamlessly integrates the front-end and back-end into a unified optimization problem, facilitating reciprocal enhancement between the two components.

\paragraph{Optimization} One can leverage Equation \eqref{eq:2nd-solution} to calculate the implicit gradients. However, due to the uniqueness of \eqref{eq:iSLAM}, we can apply an efficient approximation method, which we refer to as a ``one-step'' back-propagation strategy.
It utilizes the nature of stationary points to avoid unrolling the inner optimization iterations. 
Specifically, according to the chain rule, the gradient of $U$ with respect to ${\bm{\theta}}$ is
\begin{equation}
    \frac{\partial U}{\partial {\bm{\theta}}} = \frac{\partial U}{\partial f}\frac{\partial f}{\partial {\bm{\theta}}} + \frac{\partial U}{\partial L^*}\left(\frac{\partial L^*}{\partial f}\frac{\partial f}{\partial {\bm{\theta}}} + \frac{\partial L^*}{\partial {\bm{\mu}}^*}\frac{\partial {\bm{\mu}}^*}{\partial {\bm{\theta}}}\right),
\end{equation}
where $L^*$ is the short note of lower-level cost at its solution $L(f,g({\bm{\mu}}^*),M)$.
Intuitively, $\frac{\partial {\bm{\mu}}^*}{\partial {\bm{\theta}}}$ embeds the gradients from the inner optimization iterations, which is computationally heavy. However, if we assume the optimization converges (either to the global or local optimal), we have a stationary point where $\frac{\partial L^*}{\partial {\bm{\mu}}^*}\approx0$. This eliminates the complex gradient term $\frac{\partial {\bm{\mu}}^*}{\partial {\bm{\theta}}}$. In this way, we bypass the LL optimization iterations and thereby can back-propagate the gradients in one step. Note that the precondition for using this method is that $U$ incorporates $L$ as the sole term involving ${\bm{\mu}}^*$, otherwise it will introduce a constant error term on convergence, as analyzed in Equation \eqref{eq:error-term}.

\subsection{Discrete Optimization} \label{sec:discrete}

In the case where the LL cost $L$ in \eqref{eq:low-il} involves discrete optimization, the semantic attributes $\bm z$ in \fref{fig:framework} are usually defined in a discrete space.
However, the existence of discrete variables leads to extreme difficulty in solving IL, especially the computation of implicit gradient.
To overcome this challenge, we will present several gradient estimators for $\nabla_{\bm{\theta}} g \doteq \nabla_\theta \mathbb{E}_{f(\bm z | \bm \theta)} [g(\bm z, \mu, M)]$, where the reasoning engine $g$ functions over discrete latent variables $\bm z$.
In this section, we assume the network $f$ outputs the distribution function where $\bm z$ is sampled, i.e., $\bm z \sim f(\bm z| {\bm{\theta}})$.
Note that the difference is that $\bm z$ is the input to $g$, while $\bm{\mu}$ is the variable to optimize.
For simplicity, we omit $M$ and ${\bm{\mu}}$ in the reasoning engine $g(\bm{z})$.

\paragraph{Log-derivative Estimator} The log-derivative estimator, denoted as $\hat g'_l$, converts the computation of a non-expectation to an expectation using the log-derivative trick (LDT), which can then be approximated by random sampling.
Specifically, it can be derived with the following procedure:
\begin{subequations}
\begin{align}
\hat g'_l  & \doteq \nabla_\theta \mathbb{E}_{f(\bm z| \theta)}[g(\bm z)] \\
&= \nabla_\theta \int f(\bm z| \theta) g(\bm z) d\bm z \\ 
&= \int \nabla_\theta f(\bm z| \theta) g(\bm z)  d\bm z \quad \text{(Leibniz rule)} \label{eq:non-exp}\\ 
&= \int f(\bm z| \theta) \frac{\nabla_\theta  f(\bm z| \theta)}{f(\bm z| \theta)}  g(\bm z) d\bm z \quad \left(\text{Multiply } \frac{f(\bm z| \theta)}{f(\bm z| \theta)}\right) \\ 
&= \int f(\bm z| \theta) \nabla_\theta \log  f(\bm z| \theta)  g(\bm z) d\bm z  \quad \text{(The LDT)} \label{eq:exp}\\
&= \mathbb{E}_{f(\bm z| \theta)}[\nabla_\theta \log  f(\bm z| \theta)  g(\bm z)]\\
& \approx \frac{1}{S} \sum\limits_{i=1}^S \nabla_\theta \log f(\bm z_i| \theta)  g(\bm z_i), \quad \bm{z}_i \sim f(\bm{z}| {\bm{\theta}}),
\end{align}
\end{subequations}
where $\bm{z}_i$ is an instance among $S$ samplings. 
Note that $\nabla_\theta f(\bm z| \theta)$ is \textbf{not} a distribution function and thus \eqref{eq:non-exp} is not an expectation and cannot be approximated by sampling. In contrast, \eqref{eq:exp} is an expectation after the LDT is applied.
The \textit{log-derivative} estimator is a well-studied technique widely used in reinforcement learning \citep{sutton2018reinforcement}.
The advantage is that it is \textbf{unbiased} but the drawback is that it often suffers from \textbf{high variance}~\citep{grathwohl2017backpropagation}.

\paragraph{Reparameterization Estimator}
An alternative approach is applying the \textit{reparameterization trick}, which reparameterizes the non-differentiable variable $\bm z$ by a differentiable function $\bm{z}=T({\bm{\theta}}, \epsilon)$, where $\epsilon \sim \mathcal{N}(0,1)$ is a standard Gaussian random variable.
Denote the estimator as $\hat g'_r$, then
\begin{subequations}
    \begin{align}
    \hat g'_r
    & \doteq \mathbb{E}\left[\frac{\partial }{\partial {\bm{\theta}}} g (\bm{z})\right] = \mathbb{E}\left[\frac{\partial g}{\partial T}  \frac{\partial T}{\partial {\bm{\theta}}}\right], \quad \epsilon \sim \mathcal{N}(0, 1)\\
    & \approx \frac{1}{S}\sum_{i=1}^{S} \frac{\partial g}{\partial T(\epsilon_s)}  \frac{\partial T (\epsilon_i)}{\partial {\bm{\theta}}} , \quad \epsilon_i \sim \mathcal{N}(0, 1)
    \end{align}
\end{subequations}
where $\epsilon_i$ is an instance among $S$ samplings. 
Compared to the log-derivative estimator, the reparameterization estimator has a \textbf{lower variance} but sometimes introduces \textbf{bias} when model assumptions are not perfectly aligned with the true data, or if there are approximations in the model \citep{li2022implicit}.
Reparameterization trick has also been well studied and widely used in variational autoencoders \citep{doersch2016tutorial}.

\paragraph{Control Variate Estimator}\label{para:control_variate} We next introduce a simple but effective variance-reduced estimator based on control variate \citep{nelson1990control}, as defined in \lemmaref{lemma:control-variate}.
It has been widely used in Monte Carlo estimation methods \citep{geffner2018using,richter2020vargrad,leluc2022quadrature}.

\begin{lemma}\label{lemma:control-variate}
    A control variate is a zero-mean random variable used to reduce the variance of another random variable. Assume a random variable $X$ has an \textbf{unknown} expectation $\alpha$, and another random variable $Y$ has a known expectation $\beta$. Construct a new random variable $X^*$ so that
    \begin{equation} \label{eq:control-estimator}
            X^*=X+{\bm{\gamma}}(Y-\beta), 
    \end{equation}
    where $\bm{c} \doteq Y-\beta$ is a control variate, then the new random variable $X^*$ has $\alpha$ as its expectation but with a \textbf{smaller variance} than $X$, when the constant ${\bm{\gamma}}$ is properly chosen.
\end{lemma}
\begin{proof}
The variance of the new variable $X^*$ is
    \begin{equation}\label{ea:varX}
        \rm{Var}(X^*)=\rm{Var}(X) + {\bm{\gamma}}^2 \rm{Var}(Y) + 2{\bm{\gamma}} \rm{Cov}(X,Y),
    \end{equation}
where $\rm{Var}(X) = \mathbb{E}\|X\|^2 - \mathbb{E}^2\|X\|$ denotes the variance of a random variable, and $\rm{Cov}(X, Y) = \mathbb{E}[(X-\mathbb{E}X)^\top(Y-\mathbb{E} Y)]$ denotes the covariance of two random variables.
To minimize $\rm{Var}(X^*)$, we differentiate 
\eqref{ea:varX} w.r.t. ${\bm{\gamma}}$ and set the derivative to zero. This gives us the optimal value of ${\bm{\gamma}}$ as
    \begin{equation}\label{eq:gamma-value}
        {\bm{\gamma}}^*=-\frac{\rm{Cov}(X,Y)}{\rm{Var}(Y)}.
    \end{equation}
Substitute \eqref{eq:gamma-value} into \eqref{ea:varX}, we have the minimum variance 
\begin{equation}
    \rm{Var}^*(X^*)=(1-\rho^2_{X,Y})\rm{Var}(X), 
\end{equation}
where $\rho_{X,Y} = \frac{\rm{Cov}(X,Y)}{\sqrt{\rm{Var}(X)\rm{Var}(Y)}}$ is the correlation of the two random variables.
Since $\rho_{X,Y}\in[-1,1]$, we have $\rm{Var}(X^*)\leq \rm{Var}(X)$; when $\rho_{X,Y}$ is close to $1$, $\rm{Var}(X^*)$ will get close to 0.
This means that as long as we can design a random variable $Y$ that is correlated to $X$, we can reduce the estimation variance of $X$, which completes the proof.
\end{proof}

In practice, we usually have no direct access to ${\bm{\gamma}}^*$, and thus cannot guarantee the optimality of the control variate for arbitrarily chosen $Y$. However, without loss of generality, we can assume $X$ and $Y$ are positively correlated and solve the inequality ${\bm{\gamma}}^2\rm{Var}(Y)+2{\bm{\gamma}} \rm{Cov}(X,Y) < 0$, then ${\bm{\gamma}} \in (-\frac{2\rm{Cov}(X,Y)}{\rm{Var}(Y)},0)$. This means that as long as ${\bm{\gamma}}$ is within the range, $\rm{Var}(X^*)$ is less than $\rm{Var}(X)$. When $X$ and $Y$ are highly correlated and are of similar range, $\rm{Cov}(X, Y) \approx \rm{Var}(Y)$, thus ${\bm{\gamma}}^* = -1$ is the most common empirical choice.

Inspired by this control variate technique, we can construct a variance-reduced gradient estimator for $\nabla_{\bm{\theta}} g$ by designing a surrogate network (control variate) $s(\bm{z})$ approximating the reasoning process $g(\bm{z})$ \citep{grathwohl2017backpropagation}. 
Specifically, substitute ${\bm{\gamma}}=-1$ into \eqref{eq:control-estimator}, we have $X^* = X - Y + \beta$, this gives us the new \textbf{control variate estimator} $\hat g'_c$:
\begin{equation}\label{eq:control-variate}
    \hat g'_c =  \hat g'_l -  \hat g'_s + s'(\bm{z}),
\end{equation}
where $\hat g'_l$ is the sampling \textit{log-derivative} estimator with high variance, $\hat g'_s$ denotes the \textit{log-derivative} estimator of the surrogate network $s(\bm{z})$, and $s'(\bm{z})$ is another gradient estimator, which is the expectation of $\hat g'_s$.
For example, when $s'(\bm{z})$ is a reparameterization estimator for $\hat g'_s$, a properly trained system will have $\hat g'_l$ and $\hat g'_s$ highly-correlated, $\mathbb{E}\left[ \hat g'_s\right]=s'(\bm{z})$, then $\hat g'_c$ is an \textbf{unbiased} and \textbf{low-variance} estimator.

Using this control variate-based estimator, we can pass low-variance gradients through the semantic attribute $\bm{z}$ in the discrete space. This strategy does not assume the form of an LL optimization problem and thus can be applied to various applications. Note that the surrogate network is only needed for gradient estimation during training and is \textbf{unused} in the inference stage, meaning that the computational complexity is not increased. Next, we present an example using this technique to evaluate IL in the discrete space.

\subsection*{Example 5: Min-Max MTSP}

The multiple traveling salesman problem (MTSP) seeks tours for multiple agents such that all cities are visited exactly once while minimizing a cost function defined over the tours.
MTSP is critical in many robotic applications where a team of robots is required to collectively visit a set of target locations, e.g., unmanned aerial vehicles~\citep{sundar2013algorithms, ma2019coordinated}, automated agriculture~\citep{carpio2021mp}, warehouse logistics~\citep{wurman2008coordinating,10109784}.
The MTSP is renowned for its NP-hardness~\citep{bektas2006multiple}.
Intuitively, MTSP involves many decision variables including assigning the cities to the agents and determining the visiting order of the assigned cities. For example, an MTSP with \(100\) cities and \(10\) agents involves \(\frac{100! \times 99!}{10! \times 89!} \approx 10^{20000}\) possible solutions, while the number of atoms in the observable universe is estimated to be ``merely'' $10^{78}$ to $10^{82}$~\citep{gaztanaga2023mass}.

MTSP can be categorized by Min-Sum MTSP and Min-Max MTSP.
Intuitively, Min-Sum MTSP aims to minimize the sum of tour lengths over all agents, while Min-Max MTSP aims to reduce the longest tour length.
Though Min-Sum MTSP has been extensively studied \citep{bertazzi2015min}, Min-Max MTSP remains under-explored, and many techniques for Min-Sum MTSP may not directly apply to Min-Max MTSP.
From the perspective of learning-based approaches, Min-Max MTSP is particularly challenging because the Min-Max operation is non-differentiable, which makes it harder to use gradient-based optimization and leads to a more complicated search for the global optimum.
In this example, we will focus on the Min-Max MTSP, while our method may also apply to the Min-Sum MTSP. For the sake of brevity, we will use the abbreviation ``MTSP'' to refer to the Min-Max MTSP unless indicated otherwise.

\paragraph{Background}
Due to the high complexity, classic MTSP solvers such as Google's OR-Tools routing library meet difficulties for large-scale problems such as those with more than 500 cities \citep{ortools_routing}. To overcome this issue, there has been a notable shift towards employing neural network models to tackle MTSP~\citep{hu2020reinforcement, xin2021neurolkh,miki2018applying, vinyals2015pointer,kool2018attention,nazari2018reinforcement,park2021schedulenet,ouyang2021improving, khalil2017learning,d2020learning, wu2021learning}.
However, these methods still have fundamental limitations, particularly in their ability to generalize to new, unseen situations, and in consistently finding high-quality solutions for large-scale problems.
Supervised models struggle with limited supervision on small-scale problems and lack feasible supervision for large-scale instances, leading to poor generalization~\citep{xin2021neurolkh,miki2018applying, vinyals2015pointer}.
Reinforcement learning (RL)-based approaches usually exploit implementations of the policy gradient algorithm, such as the Reinforce algorithm~\citep{williams1992simple} and its variants. However, they often face the issue of slow convergence and highly sub-optimal solutions~\citep{hu2020reinforcement,park2021schedulenet}.
Other strategies like training greedy policy networks~\citep{ouyang2021improving,nazari2018reinforcement, kool2018attention, khalil2017learning} and iteratively improving solutions~\citep{d2020learning, wu2021learning} are also proposed but often get stuck at local optima.

\begin{figure}[t]
	\centering
    \includegraphics[width=1\linewidth]{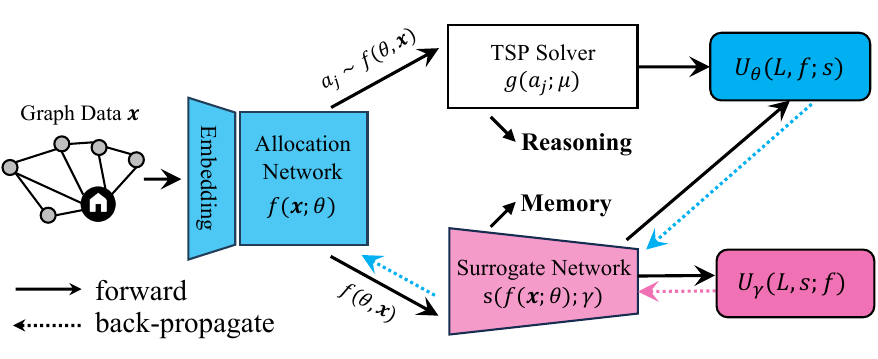}
    \caption{The framework of iMTSP. A surrogate network is introduced as the memory in the IL framework, constructing a low-variance gradient for the allocation network through the non-differentiable and discrete TSP solvers.}
	\label{fig:imtsp}
\end{figure}

\paragraph{Approach}

To overcome the above limitation, we reformulate MTSP as a BLO following the IL framework based on the control variate gradient estimator, as shown in \fref{fig:imtsp}.
This comprises a UL optimization, aiming to train an allocation network and a surrogate network assisting gradient estimation, and an LL optimization that solves decomposed single-agent TSPs.
Specifically, the allocation network $f(\bm{x}; {\bm{\theta}})$ takes the MTSP graph data $\bm{x}$ and produces a probability matrix, where its $(i, j)$-th entry represents the probability that city $i$ is assigned to agent $j$.
This decomposes the MTSP into multiple single-agent TSPs by random sampling the city allocation $a_j \sim f(\bm{x}; {\bm{\theta}})$ for agent $j$.
Then, a non-differentiable TSP solver $g$ is employed to quickly find an optimal tour ${\bm{\mu}}$ for single-agent TSPs based on the allocation $a_j$, where $j=1,\cdots, M$.
During training, the surrogate network $s({\bm{\gamma}})$, which serves as the memory module in the IL framework, is trained to have a low-variance gradient estimation through the non-differentiable reasoning engine $g$.
This results in a self-supervised formulation for MTSP, which we refer to as imperative MTSP (iMTSP):
\begin{subequations}\label{eq:imtsp}
\begin{align}
    & \min_{{\bm{\theta}}} U_{\bm{\theta}} \left(L({\bm{\mu}}^*), f({\bm{\theta}}); {\bm{\gamma}}\right);\quad \min_{\bm{\gamma}} U_{\bm{\gamma}}\left(L({\bm{\mu}}^*), s({\bm{\gamma}}); {\bm{\theta}}\right)\\
    &~\textrm{s.t.}\quad \mathbf{{\bm{\mu}}}^* = \arg\min_{\bm{\mu}} L({\bm{\mu}}),\\
    & \quad \quad~~ L \doteq \max_j D(g(a_j;{\bm{\mu}})), \quad j=1,2,\cdots,M, \\
    & \quad \quad~~ a_j \sim f(\bm{x}; {\bm{\theta}}),
\end{align}
\end{subequations}
where $D$ computes the traveling cost given a tour, and $L$ returns the maximum route length among all agents.
The UL optimization consists of two separate components: the UL cost $U_{\bm{\theta}}$ is to optimize the allocation network $f({\bm{\theta}})$, while $U_{\bm{\gamma}}$ is to optimize the surrogate network $s({\bm{\gamma}})$.
Inspired by the control variant gradient estimator \eqref{eq:control-variate} and to have a lower gradient variance, we design the UL costs $U_{\bm{\theta}}$ and $U_{\bm{\gamma}}$ as 
\begin{subequations}
\begin{align}\label{eq:upper-theta}
    U_{\bm{\theta}} & \doteq L({\bm{\mu}}^*)\log f(\bm{x};{\bm{\theta}}) -L^\prime\log f(\bm{x};{\bm{\theta}})+s(f(\bm{x};{\bm{\theta}});{\bm{\gamma}}), \\
    U_{\bm{\gamma}} & \doteq \rm{Var}( \frac{\partial U_{\bm{\theta}}}{\partial {\bm{\theta}}}),
\end{align}
\end{subequations}
where $L^\prime$ is the output of the surrogate network (control variate) but is taken as a constant by cutting its gradients.
The first instinct of training the surrogate network is to mimic the input-output relation of the TSP solver. However, highly correlated input-output mapping does not necessarily guarantee highly correlated gradients, which breaks the assumption of using control variate.
To overcome this issue, $U_{\bm{\gamma}}$ is set as $\rm{Var}(\frac{\partial U_{\bm{\theta}}}{\partial {\bm{\theta}}})$ to minimize its gradient variance.

\begin{algorithm}[t]
    \caption{The iMTSP algorithm}\label{alg:iMTSP}
    \begin{algorithmic}
    \STATE Given MTSP Graph data $\bm{x}$, TSP solver $g$, and randomly initialized allocation and surrogate networks $f({\bm{\theta}})$, $s({\bm{\gamma}})$. 
    \WHILE {not converged}
        \STATE $a_j \sim f(\bm{x};{\bm{\theta}})$;\quad
        \COMMENT {Sampling allocation}
        \STATE $L \gets \max_j D(g(a_j;{\bm{\mu}}))$; \quad
        \COMMENT{$g$ is non-differentiable}
        \STATE $L^\prime \gets s(f(\bm{x};{\bm{\theta}});{\bm{\gamma}})$; \quad
        \COMMENT{Cut gradients for $L'$}
        \STATE $  \triangledown {\bm{\theta}} \gets [L-L^\prime]\frac{\partial }{\partial {\bm{\theta}}}\log f+\frac{\partial }{\partial {\bm{\theta}}}s(f;{\bm{\gamma}})$; \quad
        \COMMENT{Compute $\frac{\partial U}{\partial {\bm{\theta}}}$}
        \STATE $ \triangledown {\bm{\gamma}} \gets \frac{\partial}{\partial {\bm{\gamma}}} \frac{\partial U}{\partial {\bm{\theta}}}^2$; \quad
        \COMMENT{Compute $\frac{\partial}{\partial {\bm{\gamma}}} \rm{Var}( \frac{\partial U}{\partial {\bm{\theta}}})$}
        \STATE ${\bm{\theta}} \gets {\bm{\theta}}-\alpha_1 \triangledown {\bm{\theta}}$; \quad
        \COMMENT{Update allocation network}
        \STATE ${\bm{\gamma}} \gets {\bm{\gamma}}-\alpha_2\triangledown {\bm{\gamma}}$; \quad
        \COMMENT{Update surrogate network}
    \ENDWHILE
\end{algorithmic}
\end{algorithm}

\paragraph{Optimization}

We next present the optimization process for iMTSP.
Intuitively, the gradient for the allocation network is the partial derivative of the UL cost $U_{\bm{\theta}}$ w.r.t. ${\bm{\theta}}$:
\begin{equation}\label{equ:grad}
    \frac{\partial U_{\bm{\theta}}}{\partial {\bm{\theta}}}=[L({\bm{\mu}}^*)-L^\prime]\frac{\partial }{\partial {\bm{\theta}}}\log f({\bm{\theta}})+\frac{\partial }{\partial {\bm{\theta}}}s(f({\bm{\theta}});{\bm{\gamma}}),
\end{equation}
which provides a low-variance gradient estimation through the non-differentiable TSP solver and the discrete decision space to the allocation network. Again, $L$ and $L^\prime$ are treated as constants instead of functions of network parameters ${\bm{\theta}}$.

On the other hand, the surrogate network $s(f(\cdot),\phi)$ is separately optimized, and updated with a single sample variance estimator. Since $\rm{Var}( \frac{\partial U}{\partial {\bm{\theta}}})= \mathbb{E}[(\frac{\partial U}{\partial {\bm{\theta}}})^2] - \mathbb{E}[(\frac{\partial U}{\partial {\bm{\theta}}})]^2$, the gradient of the surrogate network is:
\begin{equation}\label{eq:grad-gamma}
\frac{\partial U_{\bm{\gamma}}}{\partial {\bm{\gamma}}} =
    \frac{\partial}{\partial {\bm{\gamma}}}\mathbb{E}[ (\frac{\partial U_{\bm{\theta}}}{\partial {\bm{\theta}}})^2]-\frac{\partial}{\partial {\bm{\gamma}}}\mathbb{E}[ \frac{\partial U_{\bm{\theta}}}{\partial {\bm{\theta}}}]^2 = \frac{\partial}{\partial {\bm{\gamma}}}\mathbb{E}[ (\frac{\partial U_{\bm{\theta}}}{\partial {\bm{\theta}}})^2],
\end{equation}
where $\frac{\partial}{\partial {\bm{\gamma}}}\mathbb{E}[ \frac{\partial U_{\bm{\theta}}}{\partial {\bm{\theta}}}]^2$ is always zero since $\frac{\partial U_{\bm{\theta}}}{\partial {\bm{\theta}}}$ in \eqref{equ:grad} is always an unbiased gradient estimator.
This means that its expectation $\mathbb{E}[ \frac{\partial U}{\partial {\bm{\theta}}}]$ will \textbf{not} change with ${\bm{\gamma}}$.
As a result, in each iteration, we only need to alternatively compute \eqref{equ:grad} and \eqref{eq:grad-gamma} and then update the allocation and surrogate networks.
In practice, $\mathbb{E}[(\frac{\partial U}{\partial {\bm{\theta}}})^2]$ is implemented as the mean of squared gradient $\frac{\partial U_{\bm{\theta}}}{\partial {\bm{\theta}}}$.
We list a summarized optimization algorithm in \algref{alg:iMTSP} for easy understanding.

\section{Experiments} \label{sec:experiments}

We next demonstrate the effectiveness of IL in robot autonomy by evaluating the five robotic examples, i.e., logical reasoning in autonomous driving, path planning for mobile robots, model predictive control for UAV, SLAM for mobile robots, and min-max MTSP for multi-agent collaboration, and comparing their performance with the SOTA algorithms in their respective fields.

\subsection{Path Planning} \label{sec:exp-planning}

In this section, we will conduct the experiments for iA* and iPlanner for global and local planning, respectively.

\subsection*{Example A: Global Path Planning} \label{sec:exp-iastar}

To demonstrate the effectiveness of IL in global planning, we conduct both qualitative and quantitative evaluations for iA*. 

\subsubsection{Baselines}
To evaluate the generalization ability of the proposed iA* framework, we select both classic and data-driven methods as the baselines.
Specifically, we will compare iA* with the popular classic method, weighted A$^*$ \citep{pohl1970heuristic} and the latest imitation learning-based methods including neural A$^*$ \citep{yonetani2021nastar}, focal search with path probability map (FS + PPM) and weighted A$^*$ with correction factor (WA$^*$ + CF) \citep{kirilenko2023transpath}.

\begin{figure*}[t]
    \centering
    \vspace{3pt}
    \includegraphics[width=1\linewidth]{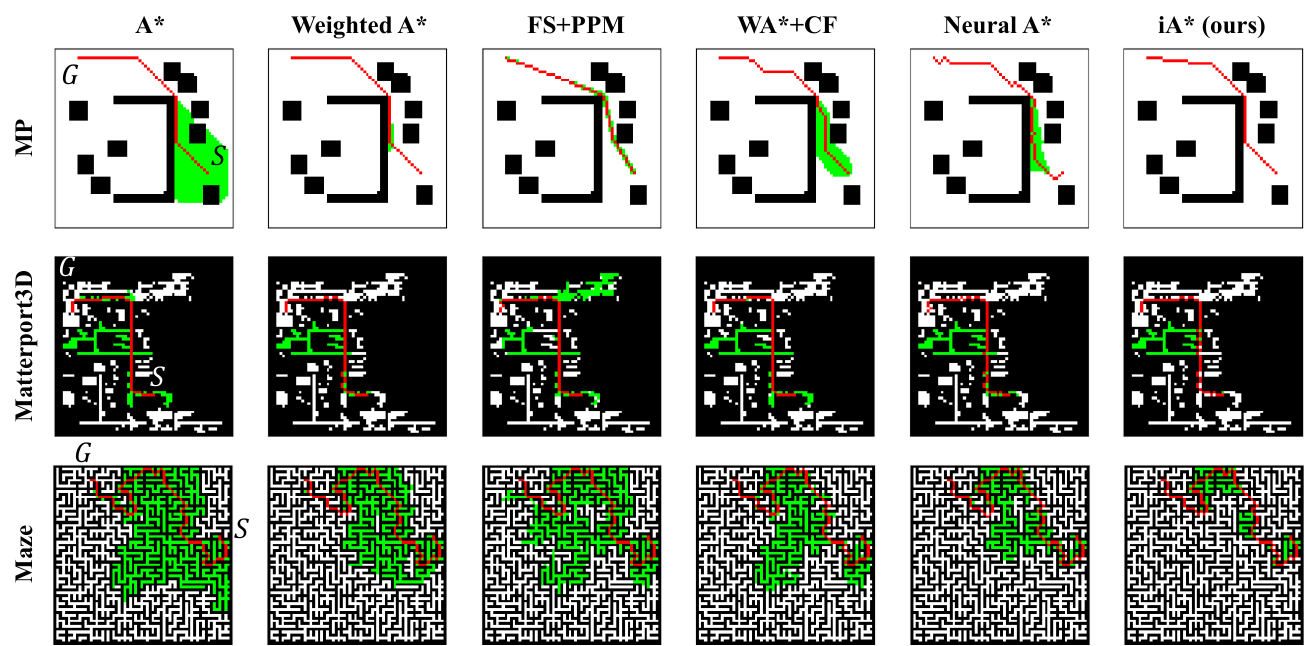}
    \caption{The qualitative results of path planning algorithms on three widely used datasets, including MP, Maze, and Matterport3D. The symbols $S$ and $G$ indicate the randomly selected start and goal positions. The optimal paths found by different path planning algorithms and their associated search space are indicated by red trajectories and green areas, respectively.}
    \label{fig:exp:iAstar-exp}
\end{figure*}

\begin{table}[t]
    \centering
    \caption{Path planning evaluation on the MP dataset. The symbol `$-$' denotes the method cannot provide meaningful results with that map size due to the large search space.}
    \label{tab:iastar:mp}
    \renewcommand{\tabcolsep}{2pt}
    \resizebox{\linewidth}{!}{
    \begin{tabular}{p{43pt}>{\centering}p{35pt}>{\centering}p{35pt}>{\centering}p{36pt}>{\centering}p{35pt}>{\centering}p{35pt}>{\centering\arraybackslash}p{35pt}}
        \toprule
        Metric &\textit{Exp}&  \textit{Rt}&\textit{Exp}&  \textit{Rt}&\textit{Exp}&  \textit{Rt}\\
        \midrule
        Map Size &\multicolumn{2}{c}{$64\times64$} &\multicolumn{2}{c}{$128\times128$} &\multicolumn{2}{c}{$256\times256$} \\
        \midrule
        WA$^*$ &63.5\% & 47.0\%& 70.1\%& 55.4\% &52.0\%&54.4\%\\
        Neural A$^*$ & 56.0\% & 24.5\% & \uline{67.2\%} & \uline{42.1\%} & \uline{44.9\%} & \uline{50.0\%} \\
        WA$^*$+CF & 68.8\% & 48.0\% & $-$ & $-$  & $-$ & $-$ \\
        FS+PPM&\textbf{88.3\%} & \textbf{83.6\%}& $-$ & $-$  & $-$ & $-$ \\
        \textbf{iA* (ours)} & \uline{69.4\%} & \uline{52.8\%} & \textbf{75.9\%} & \textbf{62.3\%} & \textbf{56.1\%} & \textbf{61.5\%} \\
        \bottomrule
    \end{tabular}
    }
\end{table}

\begin{table}[t]
    \centering
    \caption{Path planning evaluation on the Matterport3D dataset.}
    \label{tab:iastar:matterport3D}
    \renewcommand{\tabcolsep}{2pt}
    \resizebox{\linewidth}{!}{
    \begin{tabular}{p{43pt}>{\centering}p{35pt}>{\centering}p{35pt}>{\centering}p{36pt}>{\centering}p{35pt}>{\centering}p{35pt}>{\centering\arraybackslash}p{35pt}}
        \toprule
        Metric &\textit{Exp}&  \textit{Rt}&\textit{Exp}&  \textit{Rt}&\textit{Exp}&  \textit{Rt}\\
        \midrule
        Map Size &\multicolumn{2}{c}{$64\times64$} &\multicolumn{2}{c}{$128\times128$} &\multicolumn{2}{c}{$256\times256$} \\
        \midrule
        WA$^*$& 48.1\%& 36.8\%& 76.4\%& 79.6\%& 84.7\% &81.5\%\\
        Neural A$^*$ & 42.3\% & 28.4\% & \uline{65.6\%} & \uline{50.1\%} & \uline{79.4\%} & \uline{63.9\%} \\
        WA$^*$+CF & \uline{57.2\%} & \uline{42.7\%} & $-$ & $-$ & $-$ & $-$ \\
        FS+PPM & 9.3\%& $-$43.8\%& $-$ & $-$ & $-$ & $-$ \\
        \textbf{iA* (ours)} &\textbf{57.8\%} & \textbf{44.7\%} & \textbf{79.2\%} & \textbf{81.8\%} & \textbf{89.5\%} & \textbf{87.8\%} \\
        \bottomrule
    \end{tabular}
    }
\end{table}

\begin{table}[t]
    \centering
    \caption{Path planning evaluation on the Maze dataset.}
    \label{tab:iastar:maze}
    \renewcommand{\tabcolsep}{2pt}
    \resizebox{\linewidth}{!}{
    \begin{tabular}{p{43pt}>{\centering}p{35pt}>{\centering}p{35pt}>{\centering}p{36pt}>{\centering}p{35pt}>{\centering}p{35pt}>{\centering\arraybackslash}p{35pt}}
        \toprule
        Metric &\textit{Exp}&  \textit{Rt}&\textit{Exp}&  \textit{Rt}&\textit{Exp}&  \textit{Rt}\\
        \midrule
        Map Size &\multicolumn{2}{c}{$64\times64$} &\multicolumn{2}{c}{$128\times128$} &\multicolumn{2}{c}{$256\times256$} \\
        \midrule
        WA$^*$& 21.7\%& 27.4\%& 48.8\%& 11.7\%& 41.8\% &29.5\%\\
        Neural A$^*$ & 30.9\% & 40.8\% & \uline{60.5\%} & \uline{32.8\%} & \uline{51.0\%} & \uline{38.9\%} \\
        WA$^*$+CF&\textbf{47.3\%}& \textbf{45.3\%}& $-$ & $-$ & $-$ & $-$ \\
        FS+PPM& $-$8.9\% & $-$105.0\%& $-$ & $-$  & $-$ & $-$ \\
        \textbf{iA* (ours)} & \uline{37.8\%} & \uline{44.9\%} & \textbf{64.8\%} & \textbf{33.0\%} & \textbf{51.5\%} & \textbf{63.9\%} \\
        \bottomrule
    \end{tabular}
    }
\end{table}

\subsubsection{Datasets} 
To evaluate iA* in diverse environments, we select three widely used datasets in path planning, Motion Planning (MP) \citep{bhardwaj2017learning}, Matterport3D \citep{chang2017matterport3d}, and Maze \citep{maze-dataset}.
Specifically, the MP dataset contains a set of binary obstacle maps from eight environments; the Matterport3D dataset provides diverse indoor truncated signed distance function (TSDF) maps; and the Maze dataset contains sophisticated mazes-type maps. To better evaluate the generalization ability, we randomly sample the start and goal positions following \citep{yonetani2021nastar} to provide more diverse maps. It is essential to note that all methods only use the MP dataset for training, and use Matterport3D and Maze for testing, which is to demonstrate their generalization abilities.

\subsubsection{Metrics}
To evaluate the performance of different algorithms, we adopt the widely used metrics including the reduction ratio of node explorations (\textit{Exp}) \citep{yonetani2021nastar} and the reduction ratio of time (\textit{Rt}). \textit{Exp} is the ratio of the reduced search area relative to the classical A$^*$, which can be defined as $\textit{Exp} = 100 \times \frac{S^* - S}{S^*}$, where $S$ represents the search area of the method and $S^*$ is the search area of classical A$^*$. \textit{Rt} is the ratio of saved time relative to the classical A$^*$, which can be defined as $\textit{Rt} = 100 \times \frac{t^*-t}{t^*}$, where $t$ is the runtime of the method and $t^*$ is the runtime of classical A$^*$. Intuitively, \textit{Exp} aims to measure the search efficiency while \textit{Rt} aims to measure the runtime efficiency.

\subsubsection{Quantitative Performance}

Since the classical $A^*$ has a poor search efficiency for large maps, we resize the maps to different sizes, including $64\times64$, $128\times128$, and $256\times256$, to further demonstrate the efficiency of different algorithms.

\tref{tab:iastar:mp} summarizes the performance on the MP dataset. It can be seen that although iA* ranked \nth{2} for map size of $64\times64$, it achieves outstanding performance for larger maps with sizes $128\times128$ and $256\times256$. Specifically, it improves the runtime efficiency \textit{Rt} by 23\% for map size $256\times256$ compared to Neural A$^*$, the SOTA method.
This indicates that iA* has an advantage in larger maps, demonstrating the effectiveness of the upper-level predicted cost function for reducing the search space for LL path planning. 
We observe that FF+PPM and WA$^*$+CF are limited to a fixed shape of input instances due to their integration of the attention mechanism. Their network architectures include fully connected (FC) layers, which require predefined input and output shapes. The number of parameters within these FC layers varies with different input and output shapes. Consequently, adjustments and re-training procedures are necessary to accommodate varying input instance shapes.

\tref{tab:iastar:matterport3D} summarizes the path planning performance on the Matterport3D dataset, where iA* outperforms all methods in all metrics and map sizes. Specifically, iA* achieves a 12.7\% and a 37.4\% higher performance in search efficiency \textit{Exp} and runtime efficiency \textit{Rt} than the SOTA method neural A$^*$, respectively.
It is worth noting that FS+PPM, the \nth{1} ranked method on the $64\times64$ map category in the MP dataset, achieves the worst performance on the Matterport3D dataset, indicating that FS+PPM is fine-tuned overfitting to the MP dataset and cannot generalize to other types of maps.

\tref{tab:iastar:maze} summarizes the path planning performance on the Maze dataset, which is believed to be the most difficult map type due to its complicated routes.
It can be seen that the FS+PPM algorithm performs poorly, even increasing 105\% operation time and 8.9\% search area than the classical A$^*$ search. Additionally, iA* shows the most significant runtime efficiency improvements on large maps, i.e., a 64.2\% higher performance in \textit{Rt} than Neural A$^*$. This further demonstrates the generalization ability of the IL framework. 

\subsubsection{Qualitative Performance}
We next visualize their qualitative performance of all methods in \fref{fig:exp:iAstar-exp}, where each row shows the representative examples from the MP, Matterport3D, and Maze datasets, respectively. 
Specifically, the green area and the red line represent the search space and optimal paths found by each method, respectively.

For the MP dataset, as shown in the \nth{1} row of \fref{fig:exp:iAstar-exp},  it is observed that all the methods have a smaller search space than the classical A$^*$ and our iA* has the smallest search space, indicating that iA* achieves the best overall efficiency.
It is also observed that the path found by Neural A$^*$ is not smooth and all the other data-driven methods show slight differences in terms of path shape and length. This further indicates the effectiveness of our iA* framework.

The \nth{2} row of \fref{fig:exp:iAstar-exp} displays the path planning results from the Matterport3D dataset, which mainly focuses on indoor scenarios. It can be seen that the search space of FS+PPM and WA$^*$+CF are larger than the other methods. It is worth noting that FS+PPM has a larger search space than classical A$^*$.
Besides, all the optimal paths found by different methods are close to the optimal path of classical A$^*$.

In the \nth{3} row of \fref{fig:exp:iAstar-exp}, we show the performance of all baseline methods for the maze-type maps from the Maze dataset. 
From these examples, it can be observed that all data-driven methods successfully improve the search efficiency than classical A$^*$. 
Additionally, iA$*$ has the smallest search space, which means that it can even greatly improve the search efficiency of maze-level complicated maps.

\subsection*{Example B: Local Path Planning} \label{sec:exp-iplanner}

\subsubsection{Baselines}
To evaluate the robustness and efficiency of IL in local path planning, we will compare iPlanner against both classic and data-driven methods.
Specifically, we will compare with one of the SOTA classic planners, motion primitives planner (MPP) \citep{zhang2020falco,cao2022autonomous}, which received the \nth{1} place for most explored sectors in the DARPA Subterranean (SubT) challenge \citep{darpa_subt}.
Additionally, we will also compare with data-driven methods, including the supervised learning (SL)-based method \citep{loquercio2021learning} and the reinforcement learning (RL)-based method \citep{hoeller2021learning}, which will be simply denoted as SL and RL methods, respectively.

It is worth noting that MPP incorporates a 360$^\circ$ LiDAR as an onboard sensor, while iPlanner and the data-driven methods only use a depth camera, which has a narrower field of view (87$^\circ$) and noisier range measurements (15 \hertz).
Notably, the RL method was trained for a downward-tilted (30$^\circ$) camera, whereas the SL method used a front-facing camera setup.
For fairness, iPlanner will be evaluated under both camera setups, i.e., iPlanner and iPlanner (Tilt), although it was only trained for the front-facing camera.

\begin{table}[t]
\centering
\caption{The performance of SPL in local planning (Unit: \%).}
\label{tab:spl_table}
\resizebox{\linewidth}{!}{
\begin{tabular}{cccccc}
\toprule
\textbf{Method} & \textbf{Forest} & \textbf{Garage} & \textbf{Indoor} & \textbf{Matterport} & \textbf{Overall} \\
\midrule
MPP (LiDAR)& 95.09 & \textbf{89.42} & 85.82 & 74.18 & 86.13 \\
SL& 65.58 & 46.70 & 50.03 & 28.87 & 47.80 \\
RL (Tilt)& 95.08 & 69.43 & 61.10 & 59.24 & 71.21 \\
\midrule
\textbf{iPlanner} (Tilt) & 95.66 & 73.49 & 68.87 & 76.67 & 78.67 \\
\textbf{iPlanner} & \textbf{96.37} & 88.85 & \textbf{90.36} & \textbf{82.51} & \textbf{89.52} \\
\bottomrule
\end{tabular}}
\end{table}

\begin{figure}[t]
	\centering
	\subfloat[Planning runtime.\label{fig:computing-time}]{\includegraphics[height=0.341\linewidth]{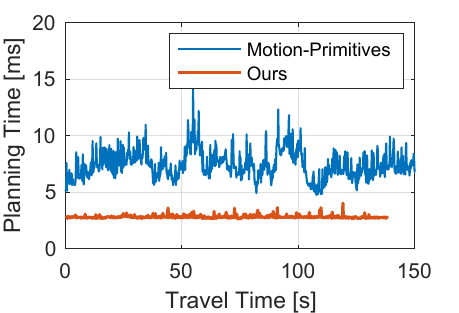}}
	\hfill
	\subfloat[Executed trajectory.\label{fig:trajectory}]{\includegraphics[height=0.325\linewidth]{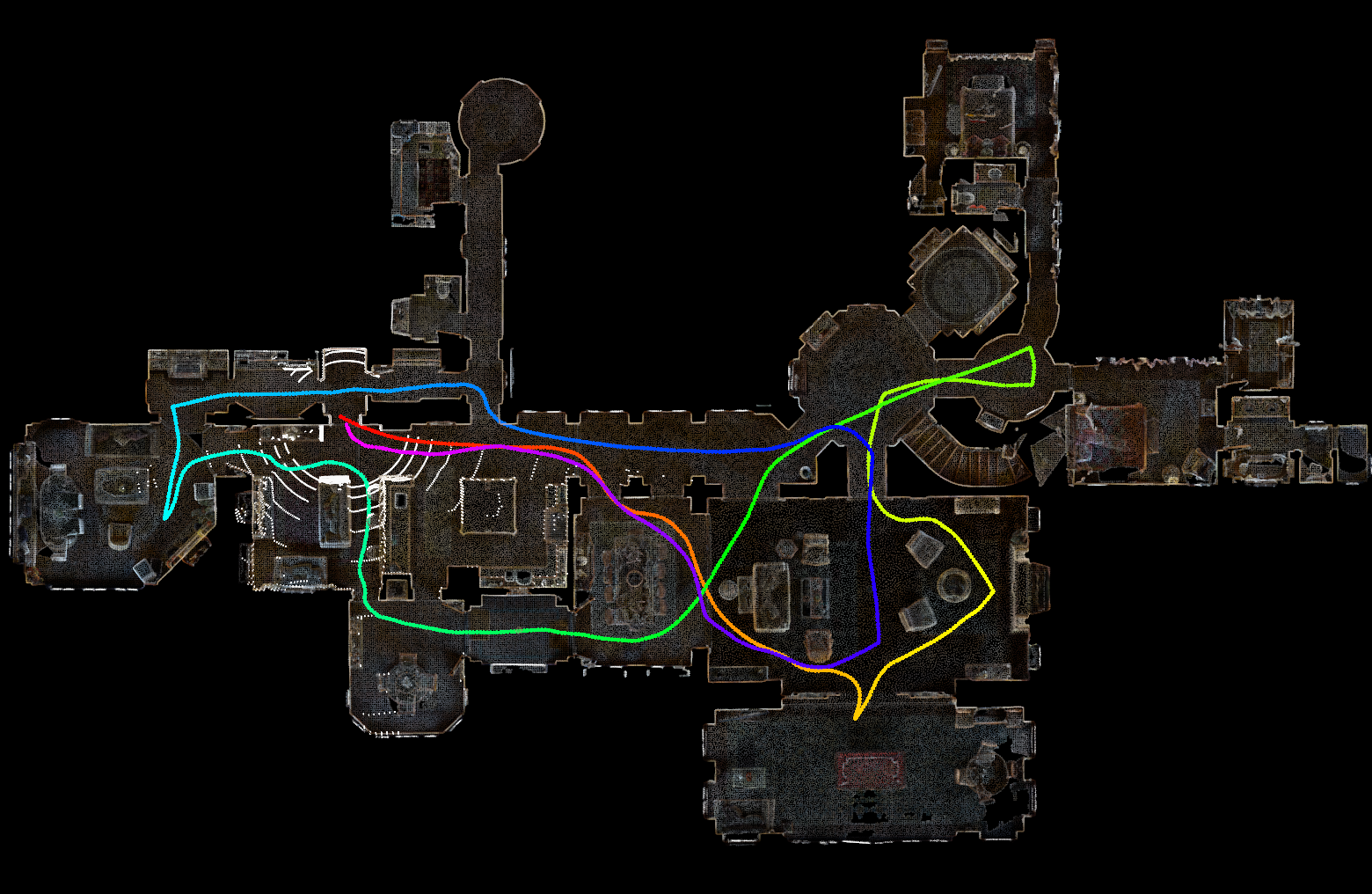}}
	\caption{\revise{An example of real-time planning using iPlanner. (a) shows the comparison of planning time. (b) shows the trajectory executed by iPlanner and the outline of the environment.}}
	\label{fig:planner}
\end{figure}

\subsubsection{Platforms}

To thoroughly evaluate the performance of local planning algorithms, we select four distinct simulated environments set up by \citep{cao2022autonomous}, including indoor, garage, forest, and Matterport3D \citep{chang2017matterport3d}.
In each of the environments, 30 pairs of start and goal positions are randomly selected in traversable areas for a comprehensive evaluation. It is worth noting that none of these environments was observed by iPlanner and the baseline data-driven methods during training.
Specifically, a laptop with a 2.6 G\hertz~Intel i9 CPU and an NVIDIA RTX 3080 GPU are used for running all methods.

\begin{figure}[t]
    \centering
    \includegraphics[width=1\linewidth]{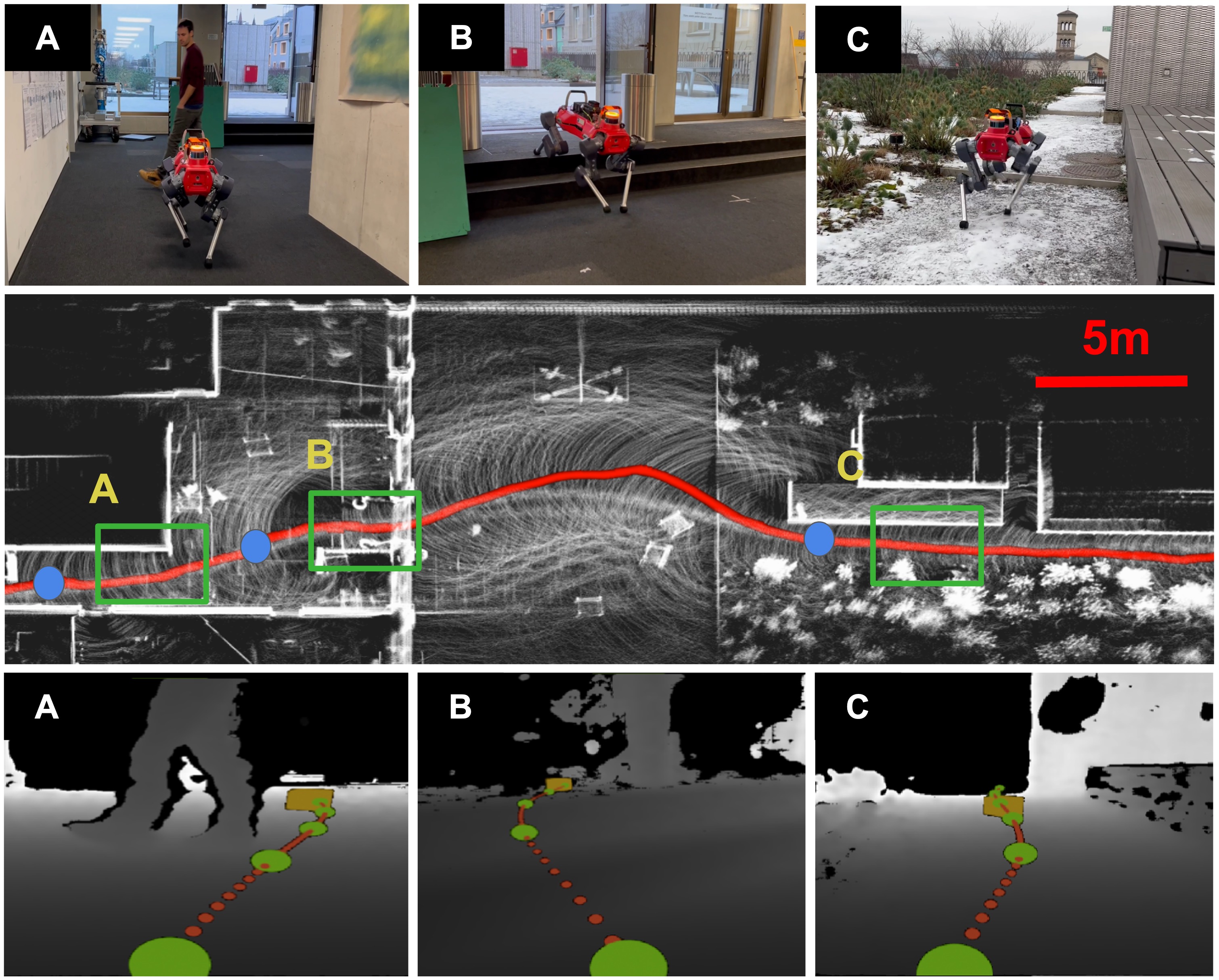}
    \caption{Real-world experiment for local path planning using iPlanner with a legged robot. The red curve indicates the robot's trajectory, beginning inside a building and then navigating to the outdoors. The robot follows a series of waypoints (blue) and plans in different scenarios marked by green boxes including (A) passing through doorways and avoid dynamic obstacles, (B) ascending stairs, and (C)  circumvent outdoor static obstacles.}
    \label{fig:iplanner-exp}
\end{figure}

\subsubsection{Metrics} The success rate weighted by path length (SPL) will be used for evaluation \citep{anderson2018evaluation}. It has been widely adopted by various path planning systems \citep{yokoyama2021success, ehsani2024spoc}.

\subsubsection{Quantitative Performance}

The overall performance of iPlanner and all baseline methods are detailed in \tref{tab:spl_table}. Notably, iPlanner demonstrates an average performance improvement of 87\% over the SL method and 26\% over the RL method in terms of SPL.
Additionally, iPlanner even outperformed MPP which utilizes LiDAR for perception in most of the environments. This further demonstrates the robustness and reliability of our IL framework.

\subsubsection{Efficiency Analysis}

\revise{We evaluate iPlanner's efficiency in Matterport3D. The robot follows 20 manual waypoints for global guidance but autonomously plans feasible paths and avoids local obstacles. As shown in Figure \ref{fig:computing-time}, iPlanner achieves an average $3\times$ speedup over MPP. The executed trajectory is shown in \fref{fig:trajectory}.}

\begin{figure*}[t]
    \centering
    \includegraphics[width=1\linewidth]{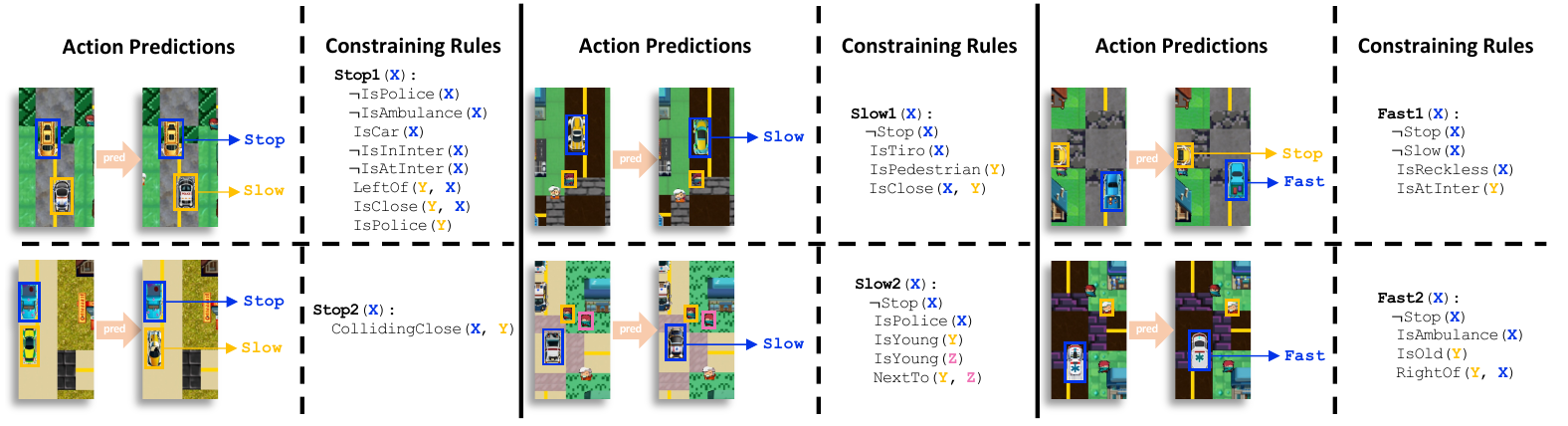}
    \caption{To predict each action, iLogic conducts rule induction with perceived groundings and the constraining rules exhibited on the right side and finally gets the accurate actions exhibited on the left side.}
    \label{fig:ilogic:vis}
\end{figure*}

\subsubsection{Real-world Tests}
We adopt an ANYmal-legged robot \citep{hutter2016anymal} which is equipped with an NVIDIA Jetson Orin and Intel Realsense D435 front-facing depth camera.
The real-world experiments involve both dynamic and static obstacles in indoor, outdoor, and mixed settings.
The robot is given a sequence of waypoints by a human operator, as illustrated in~\fref{fig:iplanner-exp}.
Specifically, the robot goes through areas including passing through doorways, circumventing both static and dynamic obstacles, and ascending and descending stairs.  These trials are designed to test the adaptability and reliability of iPlanner under varied and unpredictable environmental conditions. During the test, iPlanner can generate feasible paths online and guide the robot through all areas safely and smoothly, which demonstrates the generalization ability of our IL framework. 

\subsection{Logical Reasoning} \label{sec:exp-iLogic}

We conduct the task of visual action prediction for multiple agents in traffic to assess the efficacy of our IL framework.

\subsubsection{Datasets}

The VAP task in LogiCity \citep{li2024logicity} integrates concept grounding and logical reasoning involving agent concepts and spatial relationships.
It is particularly more challenging than other logical reasoning benchmarks with structured low-dimensional data, such as BlocksWorld \citep{dong2019nlm} or knowledge graphs \citep{cropper2021popper}.
The main reasons are that LogiCity requires learning abstract concepts from varying agent distributions in an urban setting; additionally, it also involves high-dimensional image data with diverse visual appearances, however, as a comparison, BlocksWorld only contains binary vectors.

\begin{table}[t]
    \centering
    \setlength{\tabcolsep}{1.25mm}
    \fontsize{8}{9} \selectfont
    \caption{Evaluation results of the \textit{Easy} mode in LogiCity.}
    \label{tab:ilogic:easy}
    \begin{tabular}{c|ccc|cc}
    \toprule
    Metric & \multicolumn{3}{c|}{Recall Rate}  & wAcc & $\textrm{Var}(\text{wAcc})$ \\
    Action & Slow  & Normal & Stop  & \multicolumn{2}{c}{Overall} \\
    \midrule
    ResNet+NLM & 0.671 & 0.644 & 0.995 & 0.724 & 0.00128\\
    \textbf{iLogic (ours)} & 0.678 & 0.679 & 0.995 & \textbf{0.740} & \textbf{0.00088} \\
    \bottomrule
    \end{tabular}
\end{table}

To be specific, this benchmark comprises 2D RGB renderings for multiple types of agents whose actions are constrained by first-order logic (FOL) rules. 
The training set consists of 100 different urban environments and 12 agents, the validation set includes 20 environments and 14 agents, and the test set contains 20 environments and 14 agents. Each environment encompasses 100 stimulation steps, with agent types remaining consistent within each world of the same subset. Notably, the groups of 14 agents in the validation and test sets are distinct, which aims to demonstrate the models' generalization abilities. According to the difficulty, the task is divided into two modes, i.e., the \textit{Easy} and \textit{Hard}. In each timestep, agents perform one of four actions: \texttt{Slow}, \texttt{Normal}, \texttt{Fast}, and \texttt{Stop}. The \textit{Easy} mode excludes the \texttt{Fast} action, while the \textit{Hard} mode includes all four actions, leading to the actions constrained by more complicated rules.

Actions in the VAP task are imbalanced; for instance, \texttt{Fast} actions in the \textit{Hard} model occur only about $\nicefrac{1}{4}$ to $\nicefrac{1}{9}$ as often as other actions, complicating the rule induction.

\subsubsection{Metrics}
To evaluate the model performance thoroughly, we utilize the weighted accuracy (wAcc), the variance of wAcc, and the prediction recall rate of each action as the metrics. Particularly, wAcc places greater emphasis on infrequently occurring data, reflecting the generalization ability; the variance of wAcc implies the optimization stability of the algorithms; and the recall rate assesses the accuracy with which a model learns the rules constraining each action.

\subsubsection{Baselines \& Implementation Details}
Since very little literature has studied FOL reasoning on RGB images, we don't have many baselines. As a result, we adopt the method proposed in LogiCity \citep{li2024logicity} that involves the SOTA visual encoder and reasoning engines. Particularly, we conduct experiments with two methods: (1) ResNet+NLM: ResNet as a visual encoder and NLM as the reasoning engine; (2) iLogic (ours): We train the ``ResNet+NLM'' within the IL framework, where the NLM is the LL problem.
Specifically, the visual encoder includes a concept predictor and a relationship predictor. The concept predictor involves a ResNet FPN and an ROI-Align layer, followed by a 4-layer MLP to predict agent concepts (unary predicates). The relationship predictor includes a 3-layer MLP to encode paired agent information into agent relationships (binary predicates). The reasoning engine involves an NLM with breath as 3 and depth as 4. Two AdamW optimizers with a learning rate of 0.001 are adopted for the visual encoder and the reasoning engine, respectively.

\begin{table}[t]
    \centering
    \setlength{\tabcolsep}{1.25mm}
    \fontsize{8}{9} \selectfont
    \caption{Evaluation results of the \textit{Hard} mode in LogiCity.}
    \label{tab:ilogic:hard}
    \begin{tabular}{c|cccc|cc}
    \toprule
    Metric & \multicolumn{4}{c|}{Recall Rate}  & wAcc & $\textrm{Var}(\text{wAcc})$ \\
    Action & Slow  & Normal & Fast  & Stop  & \multicolumn{2}{c}{Overall} \\
    \midrule
    ResNet+NLM & 0.552 & 0.575 & 0.265 & 0.999 & 0.400 & 0.00252\\
    \textbf{iLogic (ours)} & 0.459  & 0.598  &  0.338 &  0.999 & \textbf{0.442}  & \textbf{0.00083}\\
    \bottomrule
    \end{tabular}
\end{table}

\subsubsection{Evaluation Results}

For fairness, we mitigate the effect of randomness on the training process by running 10 times for each experiment and reporting their average performance.
As is displayed in \tref{tab:ilogic:easy} and \tref{tab:ilogic:hard}, the prediction recall rate of most actions using iLogic is higher than or equal to those of the baseline in both \textit{Easy} and \textit{Hard} mode. This proves that our iLogic can learn rules that constrain most actions more effectively.

Given the imbalanced number of actions, it is challenging to induct all the constrained rules simultaneously.
However, iLogic exhibits a marked improvement in wAcc, with 2.2\% and 10.5\% higher performance in \textit{Easy} and \textit{Hard} mode, respectively, demonstrating its stronger generalization ability. Besides, the variance of wAcc is 31.3\% and 67.1\% lower in \textit{Easy} and \textit{Hard} mode, respectively, which indicates higher optimization stability using the IL framework.

We present several qualitative examples in \fref{fig:ilogic:vis}, showing how actions are predicted from the learned rules. Since the actions are constrained by explicit rules in the LogiCity, a model has to learn these rules correctly to predict the actions. 
It can be seen that iLogic can induct the rules with perceived groundings on the right side and obtain the accurate actions exhibited on the left side of each example.

\subsubsection{Efficiency Analysis}

Due to the complicated BLO, the training time for iLogic is about 40\% longer than the supervised learning ``ResNet + NLM'', and this extra cost is well justified by the significant performance gain in \tref{tab:ilogic:easy} and \tref{tab:ilogic:hard}.
However, there is no additional computational cost in inference, with the system achieving a speed of approximately 19.78 frames/second on an RTX 3090 GPU.

\begin{table}[t]
    \centering
    \caption{The performance of UAV attitude control under different initial conditions. The ``IMU'' means the RMSE (unit: $\times 10^{-3}\rm{rad}$) of attitude estimation for the corresponding method.}
    \label{tab:impc:init}
  \resizebox{\linewidth}{!}{
    \begin{tabular}{cccccc}
        \toprule
        Initial & Methods & \textit{ST} ($\second$) & \textit{RMSE} (\degree) & \textit{SSE} (\degree) & IMU \\ \midrule
        \multirow{4}{*}{\centering 10\degree} & IMU+MPC & $ 0.287 $ & $ 0.745 $ & $ 0.250  $ & $ 7.730  $ \\ 
        & IMU$^+$+MPC & $ 0.281 $ & $ 0.692 $ & $ 0.193  $ & $ 6.470 $ \\ 
        & IMU+MPC$^+$ & $ 0.283 $ & $ 0.726 $ & $ 0.216 $ & $ 7.370 $ \\ 
        & \textbf{iMPC (ours)} & $ \mathbf{0.275} $ & $ \mathbf{0.691} $ & $ \mathbf{0.185} $ & $ \mathbf{6.310} $ \\
        \midrule
        \multirow{4}{*}{\centering 15\degree} & IMU+MPC & $ 0.330 $ & $ 0.726$ & $ 0.262$ & $ 8.390 $ \\ 
        & IMU$^+$+MPC & $ 0.299 $ & $ 0.691 $ & $ 0.213 $ & $ 7.020 $ \\ 
        & IMU+MPC$^+$ & $ 0.321 $ & $ 0.721 $ & $ 0.230 $ & $ 7.980 $ \\ 
        & \textbf{iMPC (ours)} & $ \mathbf{0.296} $ & $ \mathbf{0.688} $ & $ \mathbf{0.201} $ & $ \mathbf{6.800} $ \\ 
        \midrule
        \multirow{4}{*}{\centering 20\degree} & IMU+MPC & $ 0.336 $ & $ 0.728 $ & $ 0.263 $ & $ 8.370 $ \\ 
        & IMU$^+$+MPC & $ 0.318 $ & $ 0.685 $ & $ 0.217 $ & $ 7.270 $ \\ 
        & IMU+MPC$^+$ & $ 0.324 $ & $ 0.715 $ & $ 0.224 $ & $ 8.090 $ \\ 
        & \textbf{iMPC (ours)} & $ \mathbf{0.317} $ & $ \mathbf{0.684} $ & $ \mathbf{0.208} $ & $ \mathbf{7.010} $ \\ 
        \bottomrule
    \end{tabular}}
\end{table}

\subsection{Optimal Control} \label{exp:impc}

To validate the effectiveness of the iMPC framework, we test it on the attitude control for a quadrotor UAV. Specifically, we will show its ability to stabilize (hover) the UAV with different initial conditions under the effects of wind gusts.

\subsubsection{Baselines}
To demonstrate the effectiveness of the IL framework for jointly optimizing the data-driven IMU model and MPC, we select four baselines including (1) IMU+MPC: classic IMU integrator with a regular Diff-MPC; (2) IMU$^+$+MPC: data-driven IMU model with a regular Diff-MPC; (3) IMU+MPC$^+$: classic IMU integrator with a Diff-MPC with learnable moment of inertia (MOI); and (4) iMPC (ours): IMU$^+$ with MPC$^+$ trained with IL framework. Specifically, ``IMU$^+$'' refers to the data-driven IMU denoising and integration model from \citep{qiu2023airimu}. 
\revise{Additionally, for the wind disturbance test, we select an RL approach with a commonly used approach, proximal policy optimization (PPO) \citep{schulman2017proximal} as the baseline.}

\subsubsection{Implementation Details}

Due to efficiency considerations, we use a lightweight network for the data-driven IMU model, which includes a 2-layer multi-layer perceptron (MLP) as the encoder to get the feature embeddings from both the accelerometer and the gyroscope.
Two decoders are implemented each comprising two linear layers to address the specific corrections needed for the accelerometer and the gyroscope. These decoders are dedicated to accurately predicting the necessary adjustments for each sensor type, thereby enhancing the overall precision of the IMU readings in dynamic conditions.
For both the encoder and the decoder, we use \texttt{GELU} as the activation function for its effectiveness in facilitating smoother non-linear transformations. 
For RL, we take pre-integrated IMU measurements as the observation, output UAV action (thrust and torques) in OpenAI gym \citep{brockman2016openai}, and designed the reward to minimize the difference between the current and target UAV pose.

\subsubsection{Simulation Environment}
To accurately measure the performance of different controllers, we build a simulation environment for a quadrotor UAV with a standard 6 degree-of-freedom (DoF) dynamics model running at 1 K\hertz~\citep{gabriel2007quadrotor,abdellah2004dynamic} and an onboard IMU sensor running at 200 \hertz.
To simulate real-world effects, such as environmental disturbance, actuator uncertainty, sensor noise, and other unknown behaviors, we inject zero-mean Gaussian noise with a standard deviation of 1e-4 to the control input and zero-mean Gaussian noise with a standard deviation of 8.73e-2 to the UAV attitude at 1K\hertz.
Additionally, we employ the sensor noise models documented by the~\cite{epsonG365IMU} IMU, including initial bias, bias instability, and random walk for both the accelerometer and gyroscope. 
This produces the same noise level as a typical UAV tracking system using a visual-inertial odometry system~\citep{xu2023airvo} for attitude estimation.

\begin{table}[t]
    \centering
    \caption{The learned UAV MOI error using iMPC under different initial conditions with an initial MOI of 50\% offset.}
    \label{tab:moi}
    \resizebox{\linewidth}{!}{
    \begin{tabular}{c|cc|cc|cc}
    \toprule
    Initial Offset & Initial & Error & Initial & Error & Initial & Error \\ 
    \midrule
     50\% & $ 10 \degree $ & $ 2.67 \%$ & $ 15 \degree $ & $ 3.41 \%$ & $ 20 \degree $ & $ 2.22 \%$    \\
    \bottomrule
    \end{tabular}}
\end{table}

\begin{figure}[t]
    \centering
    \includegraphics[width=0.98\linewidth]{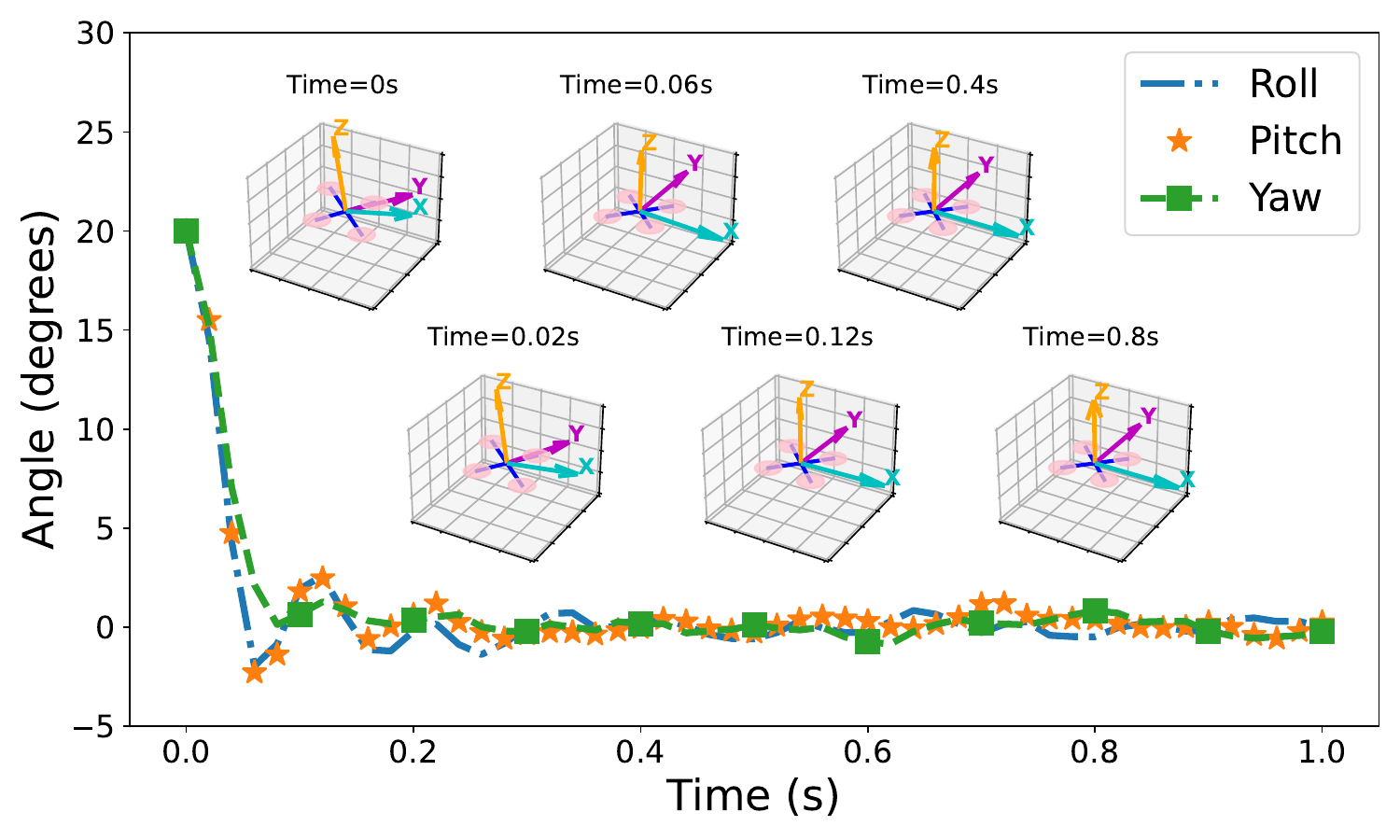}
    \caption{The UAV attitude quickly returns to a stable (zero) state for an initial condition of 20\degree~using iMPC as the controller.}
    \label{fig:mpc-response}
\end{figure}

\begin{figure*}[t]
  \centering
  \subfloat[UAV pitch angle response when encountering an \textbf{impulse wind} disturbance at 0.2 \second~for different speeds. PPO loses control at a wind speed of 15 \meter/\second, while iMPC loses control at 45 \meter/\second. \label{fig:wind-impulse}]{
    \includegraphics[width=0.48\linewidth]{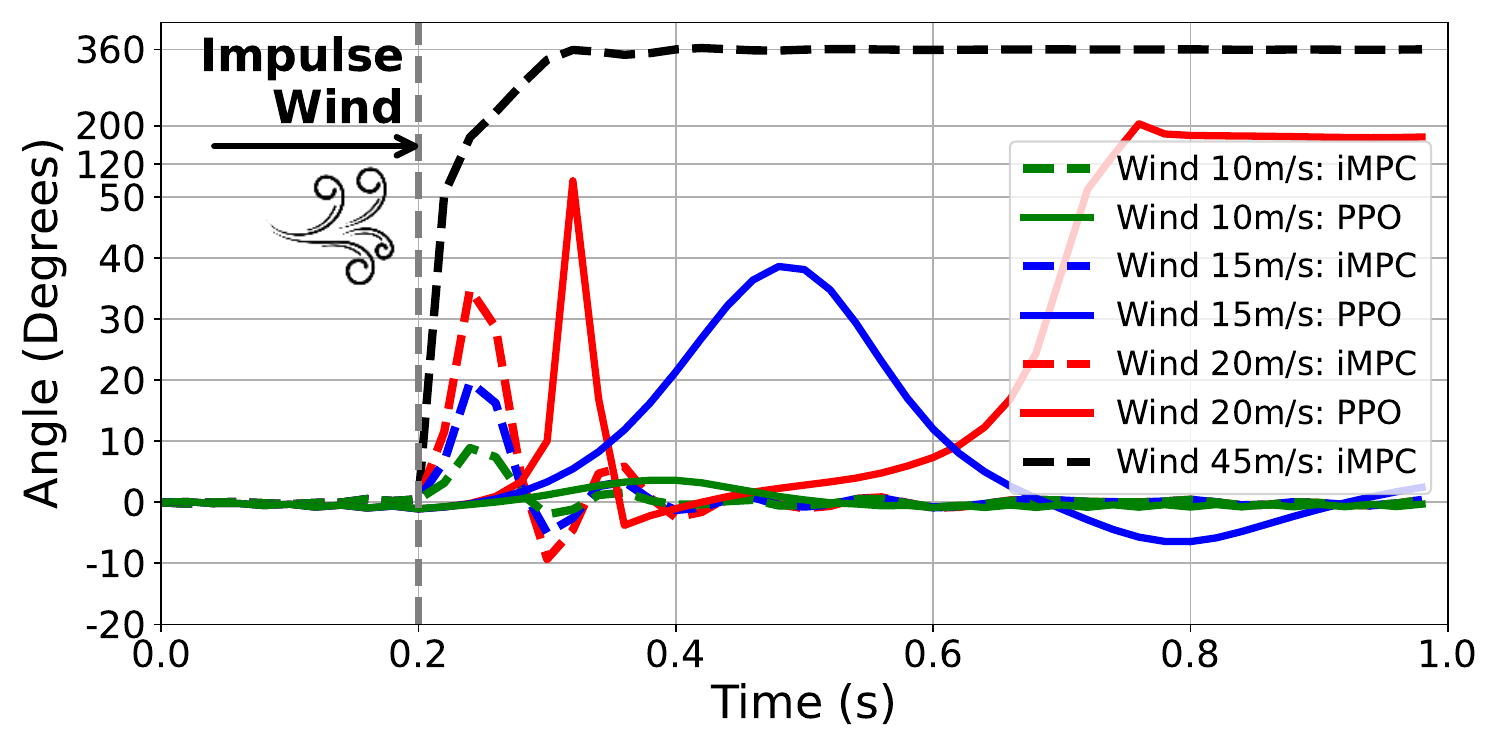}
  } \hfill
  \subfloat[UAV pitch angle response when the encountering a \textbf{step wind} disturbance at 0.2 \second~and lasting for 0.3 \second~for different speeds. PPO loses control at a wind speed of 15 \meter/\second, while iMPC loses control at 35 \meter/\second. \label{fig:wind-step}]{
    \includegraphics[width=0.48\linewidth]{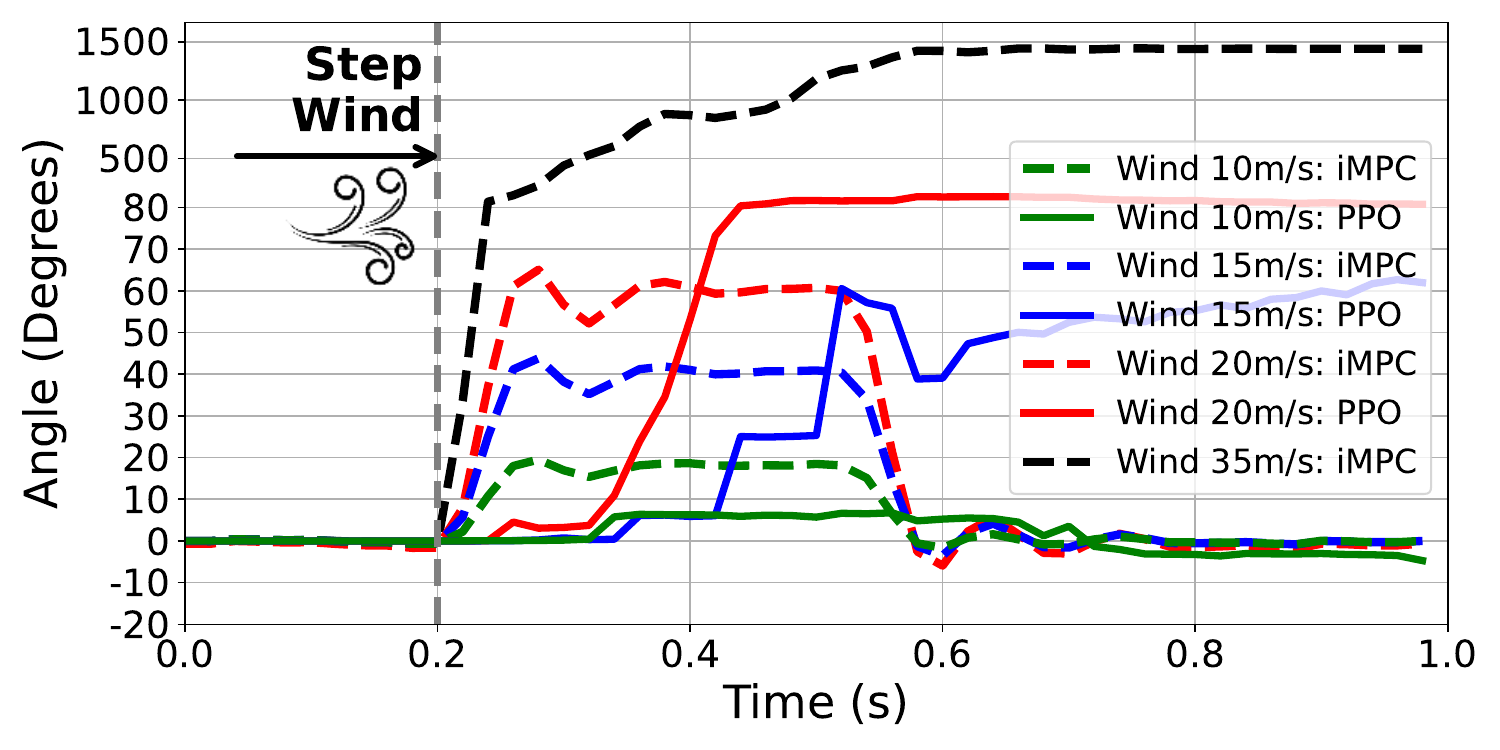}
  }
  \caption{\centering Control performance of iMPC and PPO under different levels of wind disturbance.
  }
  \label{fig:wind_gust}
\end{figure*}

\subsubsection{Metrics}

To evaluate the performance of a controller, we use three widely-used metrics, including settling time (\textit{ST}), root-mean-square error (\textit{RMSE}), and steady-state error (\textit{SSE}).
Specifically, \textit{ST} is the time for the UAV to enter and remain within $\pm 1.5 \degree$ of its final steady attitude, measuring how quickly the system settles; \textit{RMSE} is the root-mean-square of the difference between estimated and the desired attitude; and \textit{SSE} is the absolute difference between the steady and desired attitude, evaluating the control accuracy.

\begin{table}[t]
    \centering
    \caption{UAV attitude control performance under \textbf{impulse wind}, where `\ding{55}' means losing control at that wind speed.}
    \label{tab:wind-impulse}
  \resizebox{\linewidth}{!}{
    \begin{tabular}{cccccc}
        \toprule
        Wind & Methods & \textit{ST} ($\second$) & \textit{RMSE} (\degree) & \textit{SSE} (\degree) & IMU \\
        \midrule
        \multirow{4}{*}{\centering 10 \meter/\second} & IMU+MPC & $ 0.406 $ & $ 0.360 $ & $ 0.169 $ & $ 9.190 $ \\ 
        & IMU$^+$+MPC & $ 0.392 $ & $ 0.355 $ & $ 0.142 $ & $ 9.010 $ \\ 
        & IMU+MPC$^+$ & $ 0.406 $ & $ 0.360 $ & $ 0.168 $ & $ 9.180 $ \\ 
        & \revise{RL (PPO)} & \revise{$ 0.574 $} & \revise{$ 1.730 $} & \revise{$ 1.426 $} & \revise{\ding{55}} \\
        & \textbf{iMPC (ours)} & $ \mathbf{0.368} $ & $ \mathbf{0.353} $ & $ \mathbf{0.132} $ & $ \mathbf{8.660} $ \\ 
        \midrule
        \multirow{4}{*}{\centering 15 \meter/\second} & IMU+MPC & $ 0.454 $ & $ 0.328 $ & $ 0.120 $ & $ 8.160 $ \\ 
        & IMU$^+$+MPC & $ 0.450 $ & $ 0.316 $ & $ 0.065 $ & $ 7.730 $ \\ 
        & IMU+MPC$^+$ & $ 0.454 $ & $ 0.325 $ & $ 0.117 $ & $ 8.140 $ \\ 
        & \revise{RL (PPO)} & \revise{\ding{55}} & \revise{\ding{55}} & \revise{\ding{55}} & \revise{\ding{55}} \\
        & \textbf{iMPC} & $ \mathbf{0.430} $ & $ \mathbf{0.312} $ & $ \mathbf{0.060} $ & $ \mathbf{7.570} $ \\ 
        \midrule
        \multirow{4}{*}{\centering 20 \meter/\second} & IMU+MPC & $ 0.486 $ & $ 0.361 $ & $ 0.186 $ & $ 6.550 $ \\ 
        & IMU$^+$+MPC & $ 0.476 $ & $ 0.356 $ & $ 0.139 $ & $ 6.330 $ \\ 
        & IMU+MPC$^+$ & $ 0.484 $ & $ 0.361 $ & $ 0.185 $ & $ 6.520 $ \\ 
        & \revise{RL (PPO)} & \revise{\ding{55}} & \revise{\ding{55}} & \revise{\ding{55}} & \revise{\ding{55}} \\
        & \textbf{iMPC } & $ \mathbf{0.470} $ & $ \mathbf{0.354} $ & $ \mathbf{0.135} $ & $ \mathbf{6.160} $ \\ 
        \bottomrule
    \end{tabular}
    }
\end{table}

\subsubsection{Robustness to Initial Conditions}

We first evaluate the performance of the UAV attitude control under different initial conditions.
The performance of different controllers is shown in \tref{tab:impc:init} with an initial attitude of 10\degree~to 20\degree~of all three Euler angles (roll, pitch, and yaw).
Each experiment was repeated 10 times to ensure the accuracy and reliability of the results. It is clear that our iMPC can stabilize the UAV for all initial conditions within the test range. It is also worth noting that all methods show a stable performance with a small standard deviation thus we omit those numbers in the table. Compared to the baseline methods, iMPC ends up with less \textit{SSE}, \textit{RMSE}, and \textit{ST}.
\fref{fig:mpc-response} shows an example that the UAV attitude quickly returns to a stable (zero) state for an initial condition of 20\degree~using our iMPC.

To investigate how the lower-level MPC helps the upper-level IMU learning, we show the IMU attitude estimation performance in the last column of \tref{tab:impc:init}.
We can see that due to the physics knowledge from dynamics and as well as the lower-level MPC, the denoising and prediction performance of the IMU module get significantly improved compared to the separately trained case. 
Therefore, the imperative learning framework enforces the improvement of both the IMU network and the final control performance.

Since the UAV MOI is learned in iMPC, it is important to show the final estimated MOI is close to the real value after the IL learning process. 
To illustrate this, we set an initial MOI with 50\% offset from its true value and show the final estimated MOI error using iMPC in \tref{tab:moi}.
It can be seen that iMPC can learn the MOI with a final error of less than 3.5\%.
Additionally, if we jointly consider the performance in \tref{tab:impc:init} and \tref{tab:moi} and compare iMPC with IMU$^+$+MPC, and IMU+MPC$^+$ with IMU+MPC, where ``MPC$^+$'' indicates the learned MOI is in the control loop, we can conclude that a better learned MOI can lead to a better attitude control performance, with smaller SSE and settling time.

\begin{table}[t]
    \centering
    \caption{UAV attitude control performance under \textbf{step wind}, where `\ding{55}' means losing control at that wind speed.}
    \label{tab:wind-step}
  \resizebox{\linewidth}{!}{
    \begin{tabular}{cccccc}
        \toprule
        Wind & Methods & \textit{ST} ($\second$) & \textit{RMSE} (\degree) & \textit{SSE} (\degree) & IMU \\
        \midrule
        \multirow{4}{*}{\centering 10 \meter/\second} & IMU+MPC & $ 0.596 $ & $ 0.396 $ & $ 0.189 $ & $ 6.870 $ \\ 
        & IMU$^+$+MPC & $ 0.594 $ & $ 0.373 $ & $ 0.183 $ & $ 6.110 $ \\ 
        & IMU+MPC$^+$ & $ 0.588 $ & $ 0.392 $ & $ 0.179 $ & $ 6.830 $ \\ 
        & \revise{RL (PPO)} & \revise{$ 0.754 $} & \revise{$ 1.882 $} & \revise{$ 1.478 $} & \revise{\ding{55}} \\
        & \textbf{iMPC (ours)} & $ \mathbf{0.574} $ & $ \mathbf{0.352} $ & $ \mathbf{0.168} $ & $ \mathbf{5.990} $ \\ 
        \midrule
        \multirow{4}{*}{\centering 15 \meter/\second} & IMU+MPC & $ 0.684 $ & $ 0.337 $ & $ 0.172 $ & $ 7.190 $ \\ 
        & IMU$^+$+MPC & $ 0.670 $ & $ 0.309 $ & $ 0.167 $ & $ 6.740 $ \\ 
        & IMU+MPC$^+$ & $ 0.684 $ & $ 0.337 $ & $ 0.171 $ & $ 7.180 $ \\ 
        & \revise{RL (PPO)} & \revise{\ding{55}} & \revise{\ding{55}} & \revise{\ding{55}} & \revise{\ding{55}} \\
        & \textbf{iMPC (ours)} & $ \mathbf{0.620} $ & $ \mathbf{0.297} $ & $ \mathbf{0.149} $ & $ \mathbf{6.730} $ \\ 
        \midrule
        \multirow{4}{*}{\centering 20 \meter/\second} & IMU+MPC & $ 0.704 $ & $ 0.376 $ & $ 0.240 $ & $ 6.300 $ \\ 
        & IMU$^+$+MPC & $ 0.690 $ & $ 0.317 $ & $ 0.173 $ & $ 6.120 $ \\ 
        & IMU+MPC$^+$ & $ 0.704 $ & $ 0.372 $ & $ 0.253 $ & $ 6.300 $ \\ 
        & \revise{RL (PPO)} & \revise{\ding{55}} & \revise{\ding{55}} & \revise{\ding{55}} & \revise{\ding{55}} \\
        & \textbf{iMPC (ours)} & $ \mathbf{0.676} $ & $ \mathbf{0.311} $ & $ \mathbf{0.150} $ & $ \mathbf{6.110} $ \\ \bottomrule
    \end{tabular}}
\end{table}

\subsubsection{Robustness to Wind Disturbance}

We next evaluate the control performance under the wind disturbance to validate the robustness of the proposed approach. Specifically, we inject the wind disturbance by applying an external force $F = \frac{1}{2} \cdot C_d \cdot \rho \cdot A \cdot V^2$ and torque $\tau = r \times F$ to the nominal system based on the drag equation, where $C_d$ is the drag coefficient, $\rho$ is the air density, $A$ is the frontal area of the object facing the wind, $V$ is the wind speed, $r$ is the length of lever arm. 
Hence, for a fixed UAV configuration, different wind speeds $V$ lead to different disturbance levels.
We conduct two sets of experiments: in the first set, we add impulse wind gust to the UAV in the direction of opposite heading during the hover with varying wind speed $ 10/15/20/45\  \meter/\second$; in the second set, a step wind is added to the system in the same direction during hover and lasts for 0.3 \second~for different wind speed $ 10/15/20/35\  \meter/\second$. The external wind starts at 0.2 \second~after entering a steady hover.
During RL (PPO) training, we apply external wind disturbances of up to $10 \meter/\second$. In contrast, our hybrid approach with MPC provides robustness without requiring exposure to a wide range of wind conditions during training.

The pitch angle responses of the UAV using iMPC under the two different wind disturbances are shown in \fref{fig:wind-impulse} and \fref{fig:wind-step}, respectively. It can be seen that the wind affects the UAV attitude with different peak values under different wind speeds. The vehicle can resist wind and return to hover quickly even under $20\meter/\second^2$ impulse or step wind. Note that the maximum wind speed the system can sustain under step wind is smaller than that under impulse wind due to the long duration of the disturbance. When the wind speed is too high, the UAV can no longer compensate for the disturbance and the UAV flipped and finally crashed.

\begin{table}[t]
    \centering
    \caption{The average RMSE drifts on KITTI dataset. $r_{\text{rel}}$ is rotational RMSE drift (\degree/100 \meter), $t_{\text{rel}}$ is translational RMSE drift (\%) evaluated on various segments with length 100–800 \meter.}
    \label{tab:kitti_vo}
    \begin{threeparttable}
    \resizebox{0.9\linewidth}{!}{
    \begin{tabular}{c|c|c|cc}
        \toprule
        Methods & Inertial & Supervised & $r_{\text{rel}}$ & $t_{\text{rel}}$ \\ \midrule
        DeepVO & \ding{55} & \ding{51} & 5.966 & 5.450 \\ 
        UnDeepVO & \ding{55} & \ding{55} & 2.394 & 5.311 \\ 
        TartanVO & \ding{55} & \ding{51} & 3.230 & 6.502 \\ 
        DROID-SLAM & \ding{55} & \ding{51} & \textbf{0.633} & 5.595 \\ 
        \textbf{iSLAM (VO Only)} & \ding{55} & \ding{55} & 1.101 & \textbf{3.438} \\ 
        \midrule
        \cite{wei2021unsupervised} & \ding{51} & \ding{55} & 0.722 & 5.110 \\ 
        \cite{yang2022efficient} & \ding{51} & \ding{51} & 0.863 & 2.403 \\
        DeepVIO & \ding{51} & \ding{55} & 1.577 & 3.724 \\ 
        \textbf{iSLAM} (ours) & \ding{51} & \ding{55} & \textbf{0.262} & \textbf{2.326} \\
        \bottomrule
    \end{tabular}
    }
    \end{threeparttable}
\end{table}

Detailed comparisons with other approaches are presented in \tref{tab:wind-impulse} and \tref{tab:wind-step} for the two scenarios with different disturbances, respectively.
It is evident that iMPC leads to better attitude control performance with smaller \textit{SSE} and a faster \textit{ST}. Particularly compared to IMU+MPC and IMU$^+$+MPC, iMPC reduces noises with a more accurate attitude state prediction. Additionally, iMPC further reduces the IMU prediction error in terms of RMSE. Therefore, even in the presence of significant external wind disturbances, iMPC can still control the vehicle attitude robustly. This demonstrates that the IL framework can simultaneously enhance both the MPC parameter learning and IMU denoising and prediction. 
Compared to the RL baseline PPO, iMPC demonstrates significantly better performance, especially under wind disturbances that were not encountered during training. While PPO performs similarly under wind conditions seen during training, as shown in \fref{fig:wind-impulse} and \fref{fig:wind-step}, it completely loses control of the UAV when faced with out-of-distribution disturbances. These results further highlight the robustness of the IL framework, benefiting from a hybrid approach that integrates prior physical knowledge.

\subsubsection{Efficiency Analysis}
Compared to a conventional approach of IMU+MPC, iMPC achieves similar efficiency because the MPC dominates the total computation time and the overhead from the learning-based IMU inference is negligible. In fact, iMPC runs at 50 Hz control frequency with a 200 Hz IMU measurement, which is at the same level as standard MPC. On the other hand, when compared to RL, although the inference time of iMPC is longer as expected due to the additional requirement of solving an online optimization, such physics-based optimization also brings prior knowledge into the training process, thereby overcoming the typical low sampling efficiency issue of RL. In fact, the experiment shows the PPO needs an extra 221.28\% time for training due to the low sampling efficiency.

\subsection{SLAM} \label{exp:islam}

We next provide a comprehensive evaluation of the iSLAM framework in terms of estimation accuracy and runtime. We will also provide assessments on the unique capabilities brought by IL, including the mutual enhancement between the front-end and the back-end of a SLAM system and its self-supervised adaptation to new environments.

\subsubsection{Datasets}

To ensure thorough evaluation, we adopt two widely-used datasets, including KITTI \citep{geiger2013vision} and EuRoC \citep{burri2016euroc}.
They have diverse environments and motion patterns: KITTI incorporates high-speed long-range movements in driving scenarios, while EuRoC features aggressive motions in indoor environments.

\subsubsection{Metrics}
Following prior works \citep{qiu2021airdos}, we use the widely-used absolute trajectory error (ATE), relative motion error (RME), and root mean square error (RMSE) of rotational and translational drifts as evaluation metrics.

\begin{table*}[ht]
    \centering
    \caption{\centering Absolute Trajectory Errors (ATE) on the EuRoC dataset.}
    \label{tab:euroc}
    \resizebox{0.9\textwidth}{!}{
    \begin{tabular}{c|cccccccccccc}
        \toprule
        Methods & MH01 & MH02 & MH03 & MH04 & MH05 & V101 & V102 & V103 & V201 & V202 & V203 & \textbf{Avg} \\ \midrule
        DeepV2D & 0.739 & 1.144 & 0.752 & 1.492 & 1.567 & 0.981 & 0.801 & 1.570 & \textbf{0.290} & 2.202 & 2.743 & 1.354 \\ 
        DeepFactors & 1.587 & 1.479 & 3.139 & 5.331 & 4.002 & 1.520 & 0.679 & 0.900 & 0.876 & 1.905 & \textbf{1.021} & 2.085 \\ 
        TartanVO & 0.783 & \textbf{0.415} & 0.778 & 1.502 & 1.164 & 0.527 & 0.669 & 0.955 & 0.523 & 0.899 & 1.257 & 0.869 \\ 
        \textbf{iSLAM (VO Only)} & \underline{0.320} & 0.462 & \underline{0.380} & \underline{0.962} & \underline{0.500} & \underline{0.366} & \underline{0.414} & \underline{0.313} & 0.478 & \underline{0.424} & 1.176 & \underline{0.527} \\ 
        \textbf{iSLAM (ours)} & \textbf{0.302} & \underline{0.460} & \textbf{0.363} & \textbf{0.936} & \textbf{0.478} & \textbf{0.355} & \textbf{0.391} & \textbf{0.301} & \underline{0.452} & \textbf{0.416} & \underline{1.133} & \textbf{0.508} \\ 
        \bottomrule
    \end{tabular}
    }
\end{table*}

\begin{figure}[t]
  \centering
  \subfloat[\centering Decrease of the ATE.\label{fig:decrease-a}]{
    \includegraphics[width=0.47\linewidth]{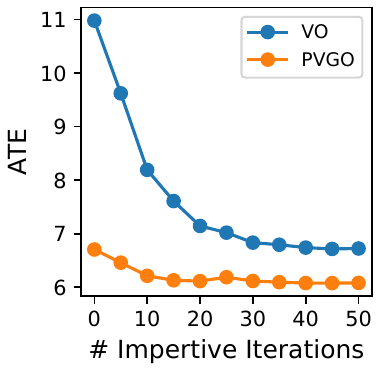}
  }
  \hfill
  \subfloat[\centering Decrease of error percentage.\label{fig:decrease-b}]{
    \includegraphics[width=0.485\linewidth]{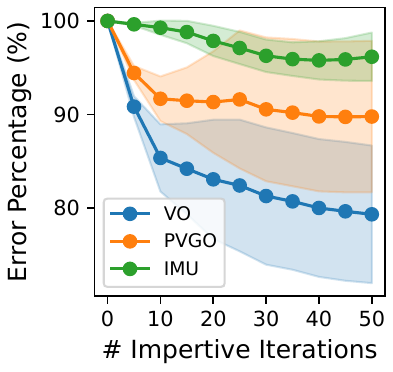}
  }
  \caption{(a) The ATE of our VO and PVGO w.r.t. number of imperative iterations. (b) The decrease in error percentage. The error before imperative learning is treated as 100\%. The ATE is used for VO and PVGO to calculate the percentage, while the relative displacement error is used for IMU. The solid lines are mean values while the transparent regions are the variances.
  }
  \label{fig:decrease}
\end{figure}

\begin{figure*}[t]
    \centering
        \subfloat{
        \includegraphics[width=0.187\linewidth]{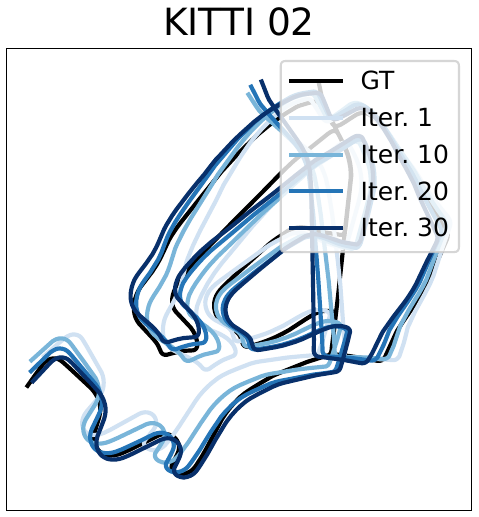}
    }
    \subfloat{
        \includegraphics[width=0.187\linewidth]{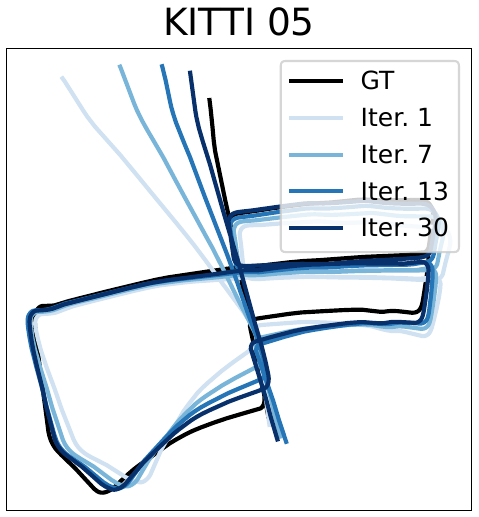}
    }
    \subfloat{
        \includegraphics[width=0.187\linewidth]{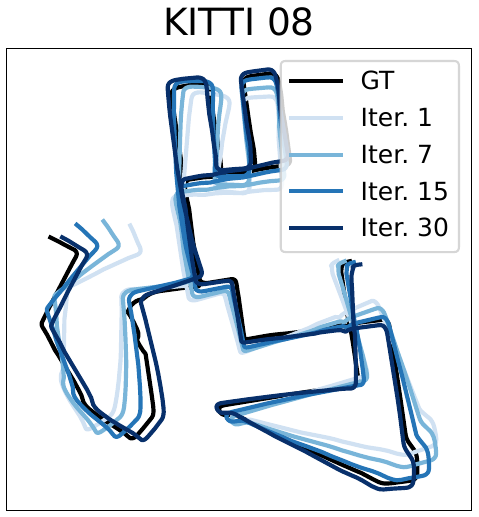}
    }
    \subfloat{
        \includegraphics[width=0.187\linewidth]{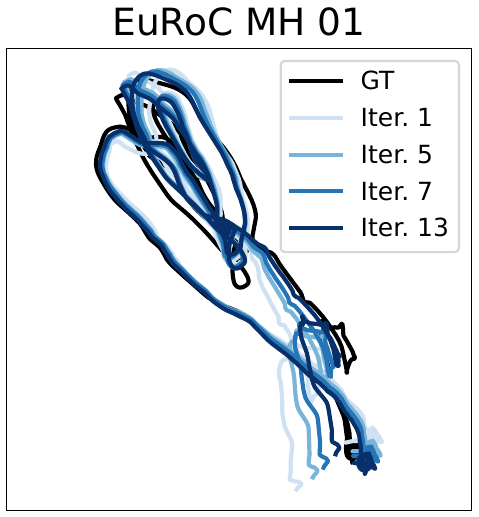}
    }
    \subfloat{
        \includegraphics[width=0.187\linewidth]{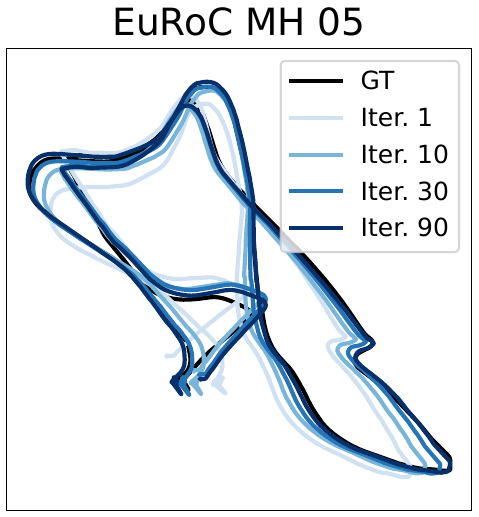}
    }
    \caption{The predicted trajectories from the front-end are improved concerning the number of imperative iterations in iSLAM.}
    \label{fig:improve}
\end{figure*}

\subsubsection{Accuracy Evaluation}

We first assess the localization accuracy of iSLAM on the KITTI dataset, which has been widely used in previous works on various sensor setups. To facilitate a fair comparison, in \tref{tab:kitti_vo}, we evaluate our standalone VO component against existing VO networks, including DeepVO \citep{supe_wang2017deepvo}, UnDeepVO \citep{unsup_li2018undeepvo}, TartanVO \citep{wang2021tartanvo}, and DROID-SLAM \citep{teed2021droid}, and compare the full iSLAM to other learning-based visual-inertial methods, which are \cite{wei2021unsupervised}, \cite{yang2022efficient}, and DeepVIO \citep{han2019deepvio}. Sequences 00 and 03 are omitted in our experiment since they lack completed IMU data.
Notably, some methods such as DeepVO and \citep{yang2022efficient} were supervisedly trained on KITTI, while our iSLAM was self-supervised. Even though, iSLAM outperforms all the competitors. Additionally, it is noteworthy that our base model, TartanVO, doesn't exhibit the highest performance due to its lightweight design. Nevertheless, through imperative learning, we achieve much lower errors with a similar model architecture. 

We also test iSLAM on the EuRoC benchmark, where the presence of aggressive motion, substantial IMU drift, and significant illumination changes pose notable challenges to SLAM algorithms \citep{qin2018vins}. However, both our standalone VO and the full iSLAM generalize well to EuRoC. As shown in \tref{tab:euroc}, iSLAM achieves an average ATE 62\% lower than DeepV2D \citep{supe_teed2018deepv2d}, 76\% lower than DeepFactors \citep{czarnowski2020deepfactors}, and 42\% lower than TartanVO \citep{wang2021tartanvo}.

\subsubsection{Efficiency Analysis}

Efficiency is a crucial factor for SLAM systems in real-world robot applications. 
Here we conduct the efficiency assessments for iSLAM on an RTX4090 GPU.
Our stereo VO can reach a real-time speed of 29-31 frames per second (FPS). The scale corrector only uses about 11\% of inference time, thus having minimal impact on the overall efficiency. The IMU module achieves an average speed of 260 FPS, whereas the back-end achieves 64 FPS, when evaluated independently. The entire system operates at around 20 FPS. Note that the imperative learning framework offers broad applicability across different front-end and back-end designs, allowing for great flexibility in balancing accuracy and efficiency.
We next measure the runtime of our ``one-step'' back-propagation strategy against the conventional unrolling approaches \citep{tang2018ba, teed2021droid} during gradient calculations. A significant runtime gap is observed: our ``one-step'' strategy is on average 1.5$\times$ faster than the unrolling approach.

\subsubsection{Effectiveness Validation}\label{sec:effective}

We next validate the effectiveness of imperative learning in fostering mutual improvement between the front-end and back-end components in iSLAM.
The reduction of ATE and error percentage w.r.t. imperative iterations is depicted in \fref{fig:decrease}. One imperative iteration refers to one forward-backward circle between the front-end and back-end over the entire trajectory. As observed in \fref{fig:decrease-a}, the ATE of both VO and PVGO decreases throughout the learning process. Moreover, the performance gap between them is narrowing, indicating that the VO model has effectively learned geometry knowledge from the back-end through BLO. \fref{fig:decrease-b} further demonstrates that imperative learning leads to an average error reduction of 22\% on our VO network and 4\% on the IMU module after 50 iterations. Meanwhile, the performance gain in the front-end also improves the PVGO result by approximately 10\% on average. This result confirms the efficacy of mutual correction between the front-end and back-end components in enhancing overall accuracy. 

The estimated trajectories are visualized in \fref{fig:improve}. As observed, the trajectories estimated by the updated VO model are much closer to the ground truth, which further indicates the effectiveness of the IL framework for SLAM.

\begin{table*}[t]
  \centering
    \caption{Performance of different methods on large-scale MTSPs. iMTSP significantly outperforms the Google OR-Tools \citep{ortools_routing} and provides an average of $3.2 \pm 0.01\%$ shorter maximum route length than the RL baseline \citep{hu2020reinforcement}.}
  \resizebox{0.9\linewidth}{!}{
  \begin{tabular}{ccc|ccccccc|c}
    \toprule
    \multicolumn{3}{c|}{Training Setting} & \multicolumn{7}{c|}{\# Test Cities}& \multirow{2}{*}{Avg gap}\\
    \cmidrule{1-10}
    Model& \# Agents & \# Cities & 400 & 500 & 600 & 700 & 800 & 900 & 1000\\
    \midrule
         Google OR-Tools & 10 & $-$ & 10.482 & 10.199 & 12.032 & 13.198 & 14.119 & 18.710 & 18.658 &277.1\%\\
         RL baseline  & 10 & 50 & 3.058 & 3.301 & 3.533 & 3.723 & 3.935 & \textbf{4.120} & \textbf{4.276} & 0.2\%\\
         \textbf{iMTSP (ours)} & 10 & 50&\textbf{3.046}&\textbf{3.273}&\textbf{3.530}&\textbf{3.712}&\textbf{3.924}&4.122&4.283&0\\
         \midrule
         Google OR-Tools &10& $-$ & 10.482 & 10.199 & 12.032 & 13.198 & 14.119 & 18.710 & 18.658&290.1\%\\
         RL baseline &10& 100& 3.035 & 3.271 & 3.542 & 3.737 & 3.954 & 4.181 & 4.330&4.3\%\\
         \textbf{iMTSP (ours)} &10& 100 &\textbf{3.019}&\textbf{3.215}&\textbf{3.424}&\textbf{3.570}&\textbf{3.763}&\textbf{3.930}&\textbf{4.042}&0\\
         \midrule
         Google OR-Tools &15& $-$ & 10.293 & 10.042 & 11.640 & 13.145 & 14.098 & 18.471 & 18.626&349.9\%\\
         RL baseline &15& 100& 2.718 &2.915&3.103&3.242&3.441&3.609&3.730&6.3\%\\
         \textbf{iMTSP (ours)} &15& 100&\textbf{2.614}&\textbf{2.771}&\textbf{2.944}&\textbf{3.059}&\textbf{3.221}&\textbf{3.340}&\textbf{3.456}&0\\
         \bottomrule
  \end{tabular}
  }
  \label{tab:MTSP_results}
\end{table*}

\begin{figure*}
\begin{subfigure}[t]{.33\textwidth}
  \centering
  \includegraphics[width=\linewidth]{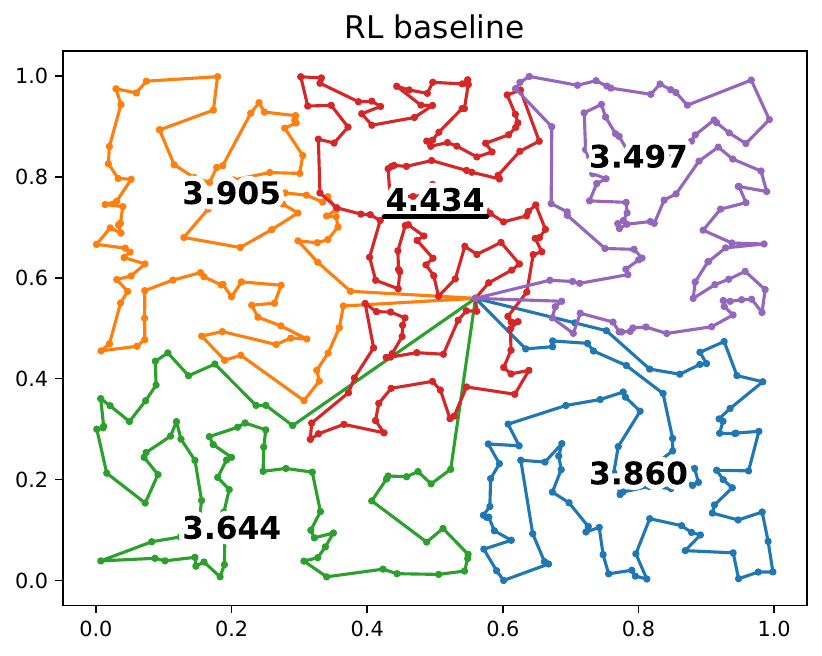}
  \caption{\centering RL baseline.}
  \label{fig:mtsp:rl-1}
\end{subfigure}
\begin{subfigure}[t]{.33\textwidth}
  \centering
  \includegraphics[width=\linewidth]{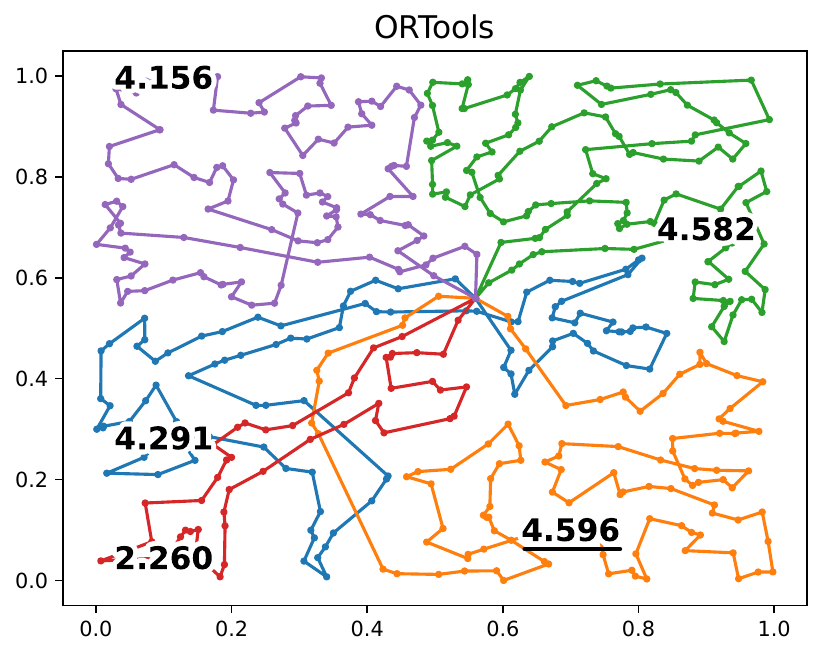}
  \caption{\centering OR-Tools with $1800 \second$ runtime.}
  \label{fig:mtsp:or-1}
\end{subfigure}
\begin{subfigure}[t]{.33\textwidth}
  \centering
  \includegraphics[width=\linewidth]{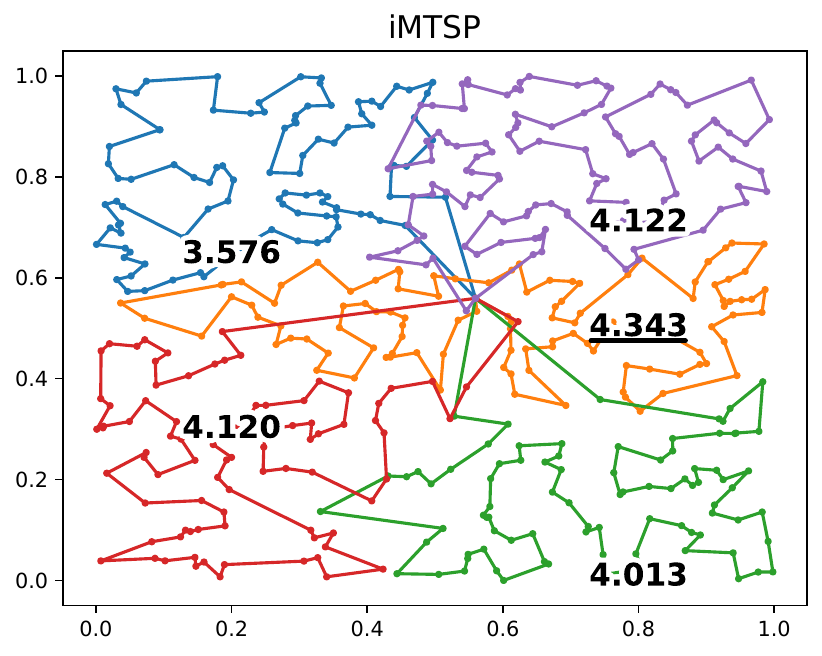}
  \caption{\centering iMTSP.}
  \label{fig:iMTSP-1}
\end{subfigure}
\caption{The performance visualization of the baselines and iMTSP for a problem instance with a central depot, $5$ agents, and $500$ cities. The numbers denote the length of the route, and different colors represent different routes. The longest tour lengths are underlined and iMTSP has the best solution and has fewer sub-optimal patterns like circular partial routes or long straight routes.}
\label{fig:mtsp-instances}
\end{figure*}

\subsection{Min-Max MTSP}

We next demonstrate our iMTSP in generalization ability, computational efficiency, and convergence speed.

\subsubsection{Baselines}
To thoroughly evaluate the performance of different methods for MTSP, we compare our iMTSP with both classic methods and learning-based methods.
This includes a well-known meta-heuristic method implemented by the Google OR-Tools routing library~\citep{ortools_routing} and an RL-based method \citep{hu2020reinforcement}, which will be referred to as Google OR-Tools and RL baseline for abbreviation, respectively.
Specifically, Google OR-Tools can provide near-optimal solutions\footnote[3]{Google OR-Tools is a meta-heuristic and cannot guarantee the optimality of solutions. However, it is efficient for small-scale problems and wildly used for academic and commercial purposes. Thus, in our experiments, it can be seen as an ``ideal'' solver providing ``optimal'' solutions.} to MTSPs with a few hundred cities, while the RL baseline can solve larger-scale MTSPs such as with one thousand cities. 
We cannot compare with other methods such as~\citep{liang2023splitnet, gao2023amarl,park2021schedulenet} because they either don't provide source code or cannot be applied to our settings.

\subsubsection{Datasets}

Most of the existing approaches can handle MTSPs with about $150$ cities, but their performance will significantly compromise with $400$ or more cities. To demonstrate iMTSP's ability to generalization and handle large-scale problems, we conduct experiments on training sets with $50$ to $100$ cities but test the models on problems with $400$ to $1000$ cities. We believe this challenging setting can reflect the generalization ability of the proposed models.

Due to the uniqueness of MTSP, we cannot deploy MTSP algorithms on such a large-scale robot team, thus we build a simulation for comparison.
Specifically, all the data instances are generated by uniformly sampling points in a unit rectangular so that both $x$ and $y$ coordinates are in the range of $0$ to $1$. The test set consists of i.i.d. samples but with a larger number ($400$ to $1000$) of cities as aforementioned. Note that the proposed approach is applicable to general graphs with arbitrary costs assigned to edges, although our dataset is constrained in the 2-D space.

\subsubsection{Implementation Details}

The structure of our allocation network is similar to that in \citep{hu2020reinforcement}, where the cities and agents are embedded into $64$ dimensional vectors.
The surrogate network is a three-layer MLP whose hidden dimension is $256$ and whose activation function is ``\texttt{tanh}''.
Additionally, Google OR-Tools provides the TSP solver in our iMTSP. It uses the ``Global Cheapest Arc'' strategy to generate its initial solutions and uses ``Guided local search'' for local improvements. We use the best-recommended settings in Google OR Tools for all experiments. All tests were conducted locally using the same desktop-level computer.

\subsubsection{Quantitative Performance}

We next present quantitative evidence of iMTSP's superior solution quality, particularly in maximum route length, as detailed in \tref{tab:MTSP_results}.
It can be seen that iMTSP achieves up to $80\%$ shorter maximum route length than Google OR-Tools which cannot converge to a local minimum within the given $300$ seconds time budget.

On the other hand, iMTSP provides a better solution in most cases than the RL baseline. On average, iMTSP has $3.2 \pm 0.01\%$ shorter maximum route length over the RL baseline. When the models are trained with $100$ cities, the difference between the route lengths monotonically increases from $0.4\%$ to $8.0\%$ and from $3.4\%$ to $8.9\%$, respectively with $10$ and $15$ agents. The results demonstrate that with the lower-variance gradient provided by our control variate-based optimization algorithm, iMTSP usually converges to better solutions when being trained on large-scale instances.

\begin{table}[t]
    \centering
    \caption{Computing time of iMTSP dealing with $1000$ cities. Due to a similar network structure, the RL baseline~\citep{hu2020reinforcement} has roughly the same runtime efficiency compared with ours, while Google OR-Tools~\citep{ortools_routing} cannot find feasible solutions within a given time budget.}
    \begin{tabular}{c|ccc}
    \toprule
         Number of Agents & 5 & 10 & 15 \\
         \midrule
         Computing Time & 4.85 \second & 1.98 \second & 1.35 \second\\
    \bottomrule
    \end{tabular}
    \label{tab:imtsp:time}
\end{table}

\subsubsection{Efficiency Analysis}

Since iMTSP contains both a data-driven allocation network and classic TSP solvers, it is important to identify the bottleneck of the architecture for future improvements.
As in Table \ref{tab:imtsp:time}, with $5$, $10$, and $15$ agents, the computing time of our model are $4.85$ \second, $1.98$ \second, and $1.35$ \second, respectively, to solve one instance with $1000$ cities. Note that the computing time decreases as the number of agents increases. This result indicates that the major computational bottleneck is the TSP solvers rather than the allocation network because more agents indicate more parameters in the allocation network but less average number of cities per agent. One possible direction to further reduce the computing time of iMTSP is to create multiple threads to run the TSP solvers in parallel since the TSPs in the lower-level optimization of iMSTP are independent.

\subsubsection{Qualitative Performance}

We next conduct qualitative analysis for the MTSP solvers in \fref{fig:mtsp-instances}, where different colors indicate different routes. 
It is observed that sub-optimal patterns exist for both baselines but very few of them exist in that of iMTSP. For example, we can observe circular partial routes for Google OR-Tools such as the green and purple routes, where the route length can be reduced by simply de-circle at the overlapping point. 
This observation provides further intuitions about the advantages of iMTSP.

\subsubsection{Gradient Variance}
As demonstrated in \sref{sec:discrete}, our control variate-based iMTSP framework is expected to have a smaller gradient variance than the RL baseline. We verify such a hypothesis by explicitly recording the mini-batch gradient variance during the training process. The experiment is conducted twice with $10$ agents respectively visiting $50$ and $100$ cities. The training data with a batch size of $512$ is divided into several mini-batches containing $32$ instances. We compute and store the gradients for every mini-batch, and then calculate their variance for the entire batch. The parameters of the allocation network are still updated with the average gradient of the whole batch.
The variance of the gradient with respect to the training process is shown in \fref{fig:impsp:var_hist}, where the horizontal axis represents the number of training iterations and the vertical axis denotes the mean logarithm gradient variance. It is worth noting that the gradient of our iMTSP converges about $20\times$ faster than the RL baseline \citep{hu2020reinforcement}, which verifies the effectiveness of our control variate-based BLO process.

\begin{figure}[t]
    \centering
    \includegraphics[width=\linewidth]{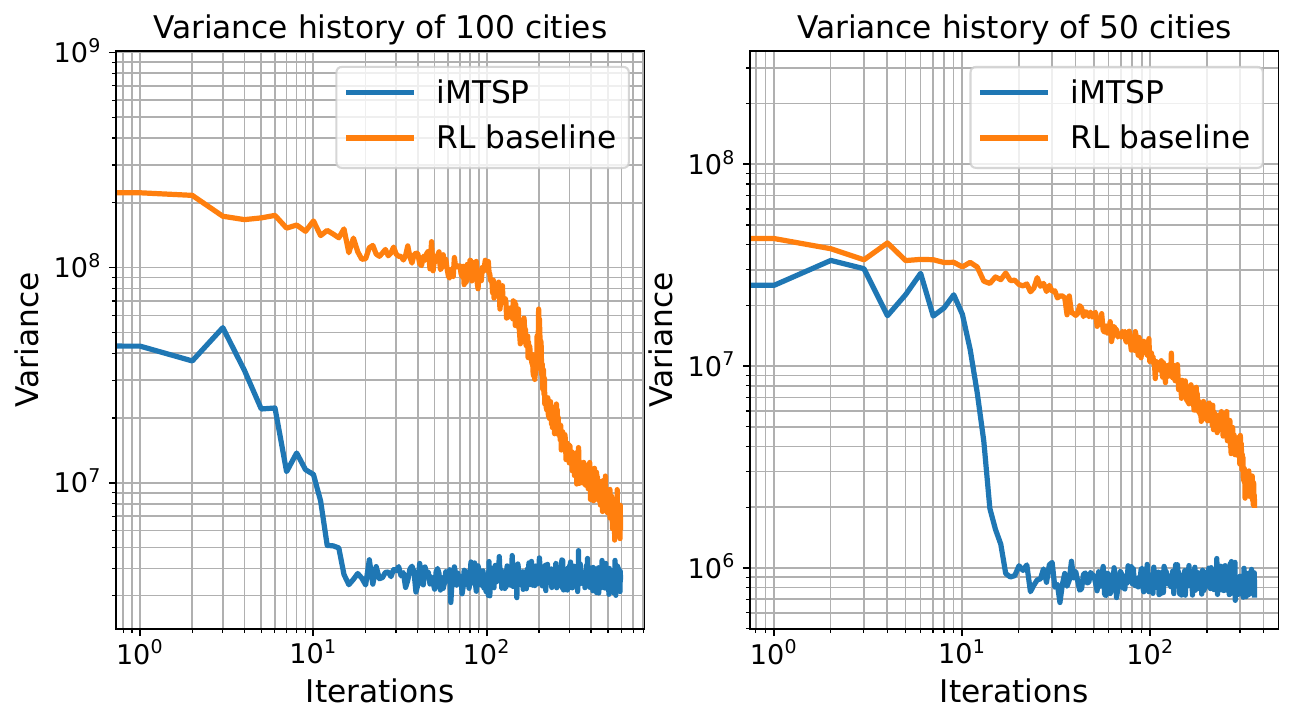}
    \caption{The gradient variance history of our method and the RL baseline during the training process with $50$ and $100$ cities. The vertical axis is the logarithmic variance of the allocation network gradient, and the horizontal axis denotes the number of training iterations. Our control-variate-base gradient estimator converges about $20\times$ faster than the RL baseline.} 
    \label{fig:impsp:var_hist}
\end{figure}

\section{Conclusions \& Discussions}

We introduced imperative learning, a self-supervised neuro-symbolic learning framework aimed at enhancing robot autonomy. Imperative learning provides a systematic approach to developing neuro-symbolic systems by integrating the expressiveness of neural methods with the generalization capabilities of symbolic methods.
We presented five distinct applications of robot autonomy: path planning, rule induction, optimal control, visual odometry, and multi-agent routing. Each application exemplifies different optimization techniques for addressing lower-level problems, namely closed-form solutions, first-order optimization, second-order optimization, constrained optimization, and discrete optimization, respectively. These examples demonstrate the versatility of imperative learning and we expect they can stimulate further research in the field of robot autonomy.

Similar to other robot learning frameworks, imperative learning has its drawbacks and faces numerous unresolved challenges. Robotics often involves highly nonlinear real-world problems, where theoretical assumptions for bilevel optimizations do not always hold, leading to a lack of theoretical guarantees for convergence and stability, despite its practical effectiveness in many tasks. Additionally, unlike data-driven models, the training process for bilevel optimization requires careful implementation. Moreover, applying imperative learning to new problems is not straightforward. Researchers need to thoroughly understand the problems to determine the appropriate allocation of tasks to the respective modules (neural, symbolic, or memory), based on the strengths of each module. For instance, a neural model might excel in tasks involving dynamic obstacle detection and classification, while a symbolic module could be more suitable for deriving and optimizing multi-step navigation strategies based on rules and logical constraints.

We believe that the theoretical challenges of imperative learning will also inspire new directions and topics for fundamental research in bilevel optimization. For example, second-order bilevel optimization remains underexplored in machine learning due to time and memory constraints. However, its relevance is increasingly critical in robotics, where second-order optimization is essential for achieving the required accuracy in complex tasks. Furthermore, handling lower-level constraints within general bilevel optimization settings presents significant challenges and remains an underdeveloped area. Recent advancements, such as those based on proximal Lagrangian value functions \citep{yao2024overcoming, yao2023constrained}, offer potential solutions for constrained robot learning problems. We intend to investigate robust approaches and apply these techniques to enhance our experimental outcomes for imperative learning. Additionally, we plan to develop heuristic yet practical solutions for bilevel optimization involving discrete variables, and we aim to refine the theoretical framework for these discrete scenarios with fewer assumptions, leveraging recent progress in control variate estimation.

To enhance the usability of imperative learning for robot autonomy, we will extend PyPose \citep{wang2023pypose,zhan2023pypose}, an open-source Python library for robot learning, to incorporate bilevel optimization frameworks. Additionally, we will provide concrete examples of imperative learning using PyPose across various domains, thereby accelerating the advancement of robot autonomy.

\begin{acks}
This work was supported by the DARPA award HR00112490426. Any opinions, findings, conclusions, or recommendations expressed in this paper are those of the authors and do not necessarily reflect the views of DARPA.
\end{acks}

{
\balance
\bibliographystyle{SageH}
\bibliography{references/citation, references/imperative, references/publication, references/pypose, references/slamroadmap, references/iSLAM,references/planning,references/control, references/multi,references/bilevel}
}

\end{document}